\documentclass[twoside,11pt]{article}

\usepackage{jmlr2e}

\usepackage{amsmath}
\usepackage{amsfonts} 
\usepackage{setspace}
\usepackage{subfigure}
\usepackage{bm}
\usepackage{verbatim}
\usepackage{float}
\usepackage{mathtools}
\usepackage{enumitem}
\usepackage{multicol}
\usepackage[figuresright]{rotating}
\usepackage{multicol}

\usepackage{datetime}
\usepackage{mathrsfs}
\usepackage{hypernat}
\usepackage[colorlinks,citecolor=blue,urlcolor=blue,linkcolor=blue]{hyperref}
\usepackage[capitalize]{cleveref}

\usepackage[usenames]{color}

\usepackage{tikz}
\usetikzlibrary{shapes,decorations,shadows}
\usetikzlibrary{decorations.pathmorphing}
\usetikzlibrary{decorations.shapes}
\usetikzlibrary{fadings}
\usetikzlibrary{patterns}
\usetikzlibrary{calc}
\usetikzlibrary{decorations.text}
\usetikzlibrary{decorations.footprints}
\usetikzlibrary{decorations.fractals}
\usetikzlibrary{shapes.gates.logic.IEC}
\usetikzlibrary{shapes.gates.logic.US}
\usetikzlibrary{fit,chains}
\usetikzlibrary{positioning}
\usepgflibrary{shapes}
\usetikzlibrary{scopes}
\usetikzlibrary{arrows}
\usetikzlibrary{backgrounds}
\usetikzlibrary{matrix}

\tikzset{latent/.style={circle,fill=white,draw=black,inner sep=1pt,
minimum size=20pt, font=\fontsize{10}{10}\selectfont},
obs/.style={latent,fill=gray!25},
const/.style={rectangle, inner sep=0pt},
factor/.style={rectangle, fill=black,minimum size=5pt, inner sep=0pt},
>={latex},
color={blue},
draw={violet},
}

\pgfdeclarelayer{b}
\pgfdeclarelayer{f}
\pgfsetlayers{b,main,f}

\def\[#1\]{\begin{align}#1\end{align}}
\newcommand{\defas}{\vcentcolon=}  
\newcommand{\given}{\mid}

\newcommand{\AND}{\wedge}

\newcommand{\dee}{\mathrm{d}}

\newcommand{\Nats}{\mathbb{N}}
\newcommand{\Ints}{\mathbb{Z}}
\newcommand{\NNInts}{\Ints_+}

\newcommand{\equaldist}{\overset{d}{=}}

\renewcommand{\det}[1]{\vert #1 \vert}

\newcommand{\dist}{\ \sim\ }

\newcommand{\Reals}{\mathbb{R}}

\newcommand{\iid}{i.i.d.}

\newcommand{\Tcal}{\mathcal{T}}
\newcommand{\Scal}{\mathcal{S}}
\newcommand{\Ncal}{\mathcal{N}}

\newcommand{\Rcal}{\mathcal{R}}
\newcommand{\Vcal}{\mathcal{V}}
\newcommand{\Bcal}{\mathcal{B}}

\newcommand{\EE}{\mathbb{E}}
\renewcommand{\Pr}{\mathbb{P}}

\newcommand{\tr}{\text{tr}}

\newcommand{\bfn}{\lambda_r}
\newcommand{\sfn}{\lambda_s}

\newcommand{\betadist}{\text{beta}}
\newcommand{\gammadist}{\text{gamma}}
\newcommand{\betabinom}{\text{beta-binomial}}

\newcommand{\soffset}{\theta_s}
\newcommand{\boffset}{\theta_r}

\DeclarePairedDelimiter{\ceil}{\lceil}{\rceil}

\newcommand{\numstop}{n_s}
\newcommand{\numbranch}{n_r}

\newcommand{\cardsharp}{\#}

\newcommand{\numfeat}[1]{K_{#1}}

\newenvironment{noindlist}
 {\begin{list}{\labelitemi}{\leftmargin=0.5cm \itemindent=0em}}
 {\end{list}}

\newtheorem{thm}{Theorem}

\jmlrheading{X}{20XX}{X-XX}{X/XX}{XX/XX}{Creighton Heaukulani, David A. Knowles, and Zoubin Ghahramani}

\ShortHeadings{Beta diffusion trees}{Heaukulani, Knowles \& Ghahramani}
\firstpageno{1}

\begin{document}

\title{Beta diffusion trees and hierarchical feature allocations}

\author{\name Creighton Heaukulani{$^\dagger$} \email ckh28@cam.ac.uk \\
       \name David A.\ Knowles{$^*$} \email davidknowles@cs.stanford.edu \\
       \name Zoubin Ghahramani$^\dagger$ \email zoubin@eng.cam.ac.uk
       \AND
       \addr Department of Engineering{$^\dagger$} \\
       University of Cambridge\\
       Cambridge, CB2 1PZ, UK
       \AND
       \addr Department of Computer Science$^*$ \\
       Stanford University\\
       Stanford, CA, 94305-9025, USA}

\editor{Somebody}

\maketitle

\begin{abstract} 
We define the \emph{beta diffusion tree}, a random tree structure with a set of leaves that defines a collection of overlapping subsets of objects, known as a \emph{feature allocation}.
A generative process for the tree structure is defined in terms of particles (representing the objects) diffusing in some continuous space, analogously to the \emph{Dirichlet diffusion tree} \citep{NealDDT2003}, which defines a tree structure over partitions (i.e., non-overlapping subsets) of the objects.
Unlike in the Dirichlet diffusion tree, multiple copies of a particle may exist and diffuse along multiple branches in the beta diffusion tree, and an object may therefore belong to multiple subsets of particles.
We demonstrate how to build a hierarchically-clustered factor analysis model with the beta diffusion tree and how to perform inference over the random tree structures with a Markov chain Monte Carlo algorithm.
We conclude with several numerical experiments on missing data problems with data sets of gene expression microarrays, international development statistics, and intranational socioeconomic measurements.
\end{abstract} 


%

\section{Introduction}

Latent feature models assume that each object (from some collection of objects) can be assigned to zero or more overlapping sets, called \emph{features}.
These objects could be the data measurements themselves, or some other unobserved variables in a model.
The overlapping sets assumption in feature allocation models is appropriate for a variety of statistical tasks.
For example, in visual scene analyses, images could be assigned to the following features: ``image contains a chair'', ``image contains a table'', ``image is of a kitchen'', etc.
Clearly these features are not mutually exclusive.
The \emph{Indian buffet process} \citep[IBP;][]{GG2006, GGS2007} defines a generative process for a collection of features (called a \emph{feature allocation}), the number of which is unbounded.
With the IBP, objects are (or are not) assigned to a feature with a feature-specific probability that is independent of the other features.  In the scene example, however, the features are structured into a hierarchy: tables and chairs are likely to appear in scenes together, and if the scene is of a kitchen, then possessing both tables and chairs is highly probable.
In order to model hierarchically related feature allocations, we define the \emph{beta diffusion tree}, a random tree structure with a set of leaves that defines a feature allocation for a collection of objects.  As with the IBP, the number of leaves (features) is random and unbounded, but will be almost surely finite for a finite set of objects.

\subsection{Diffusion trees, partitions, and feature allocations}
\label{sec:diffusion_trees}

Models for hierarchically structured \emph{partitions}, i.e., non-overlapping subsets, of a collection of objects can be accomplished with the \emph{Dirichlet diffusion tree} \citep{NealDDT2003}, in which a collection of $N$ particles (each labeled with an object) diffuse in a continuous space $\mathcal X$ over a fictitious unit time interval $t\in [0,1]$.
In this work, we take $\mathcal X = \Reals^D$ for some $D\ge 1$, and we denote the locations of the particles in $\mathcal X$ at time $t$ by $\bm x_1(t), \dotsc, \bm x_N(t)$.
The particles may follow any diffusion process, though we typically choose one that allows us to analytically integrate out the particular paths taken by the particles.
For example, if the diffusion paths are distributed as Brownian motions with variance $\sigma_X^2$, then a particle at position $\bm x (t)$ in $\Reals^D$ at time $t$ will reach position $\bm x(t+\dee t)$ at time $t+\dee t$, with
\[
\bm x ( t+\dee t) \given \bm x (t)
	 \dist \Ncal ( \bm x (t) , \sigma_X^2 \bm I_D \dee t )
	 ,
	 \label{eq:brownian_motion}
\]
where $\bm I_D$ denotes the $D\times D$ identity matrix and $\Ncal( \bm \mu, \bm \Sigma )$ denotes a Gaussian distribution with mean $\bm \mu$ and covariance matrix $\bm \Sigma$.
Note that the ``time'' index $t$ need not refer to actual time, but is merely an index used to generate a hierarchically structured partition of the objects.
It is important to note that the tree structure does not depend on the choice of diffusion process.

In the Dirichlet diffusion tree, each particle starts at the origin, follows the path of previous particles, and \emph{diverges} from the path at randomly distributed times.
Once the particle has diverged, it diffuses (independently of all previous particles) until time $t=1$.
In particular, let $\alpha > 0$, which we call the \emph{divergence parameter}.  The first particle $\bm x_1(t)$ starts at the origin $\bm x_1(0) = \bm 0$ and diffuses until time $t=1$.
Each subsequent particle $\bm x_2(t), \dotsc, \bm x_N(t)$ starts at the origin and follows the path of the previous particles.
If at time $t>0$ the particle has not diverged from the path, it diverges in $[t, t+\dee t]$ with probability
\[
\frac{ \alpha }{ m }\, \dee t
	,
	\label{eq:DDTrate}
\]
where $m$ is the number of particles previously along this path.
In other words, the time until diverging is exponentially distributed with rate $\alpha/m$.
If the particle does not diverge before reaching a previous divergence point on the path (where at least one previous particle diverged), then at the divergence point it follows one of the two paths with probability proportional to the number of particles that have followed that path previously.
We stop this process at $t=1$ and record the positions of the particles in $\mathcal X$.
\begin{figure}
\centering
\subfigure[A simulated Dirichlet diffusion tree]{
                \includegraphics[scale=.3]{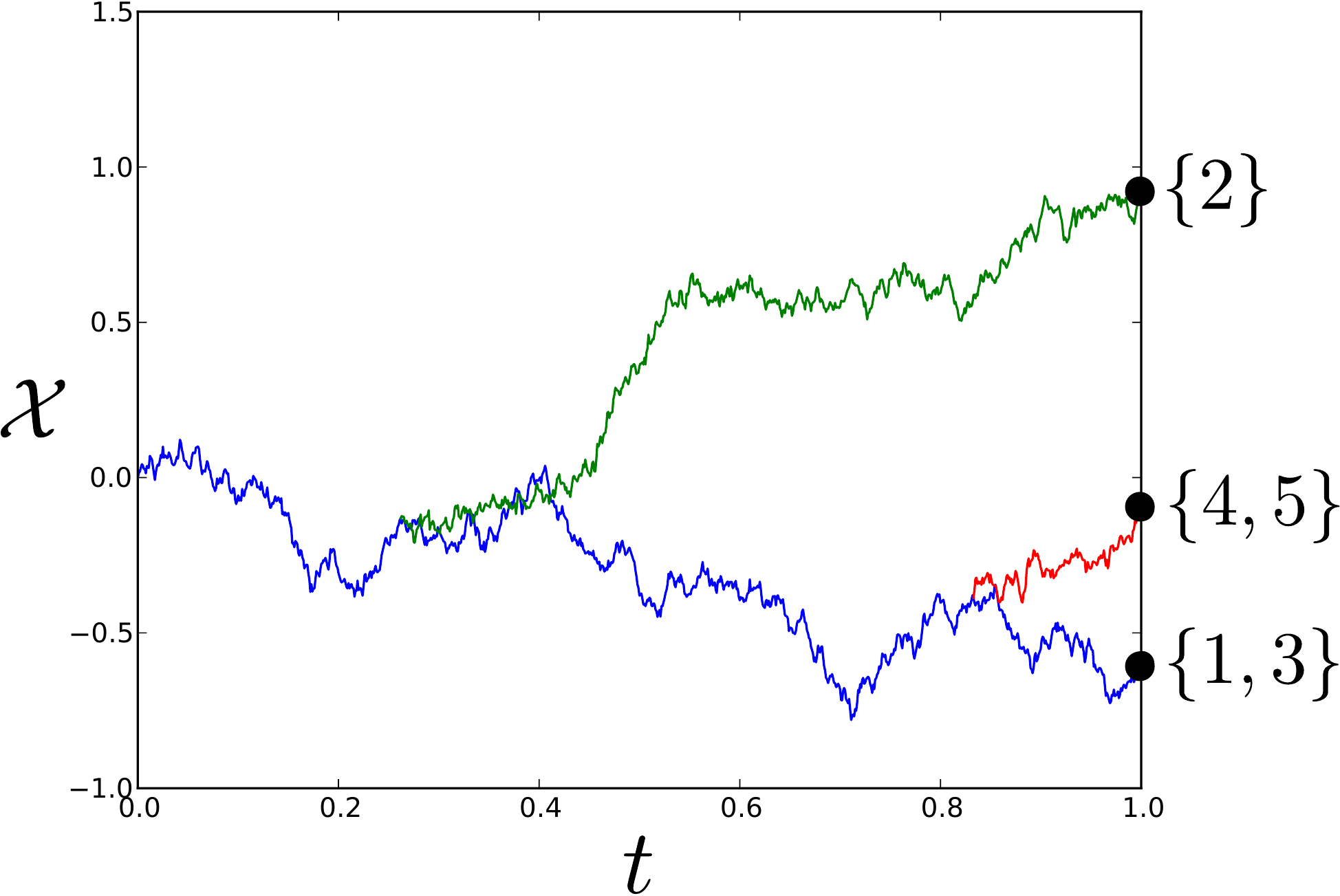}
                \label{fig:DDT_demo_paths}
                }
        \quad \quad 
\subfigure[Corresponding tree structure]{
		\raisebox{15mm}{
               \includegraphics[scale=.3]{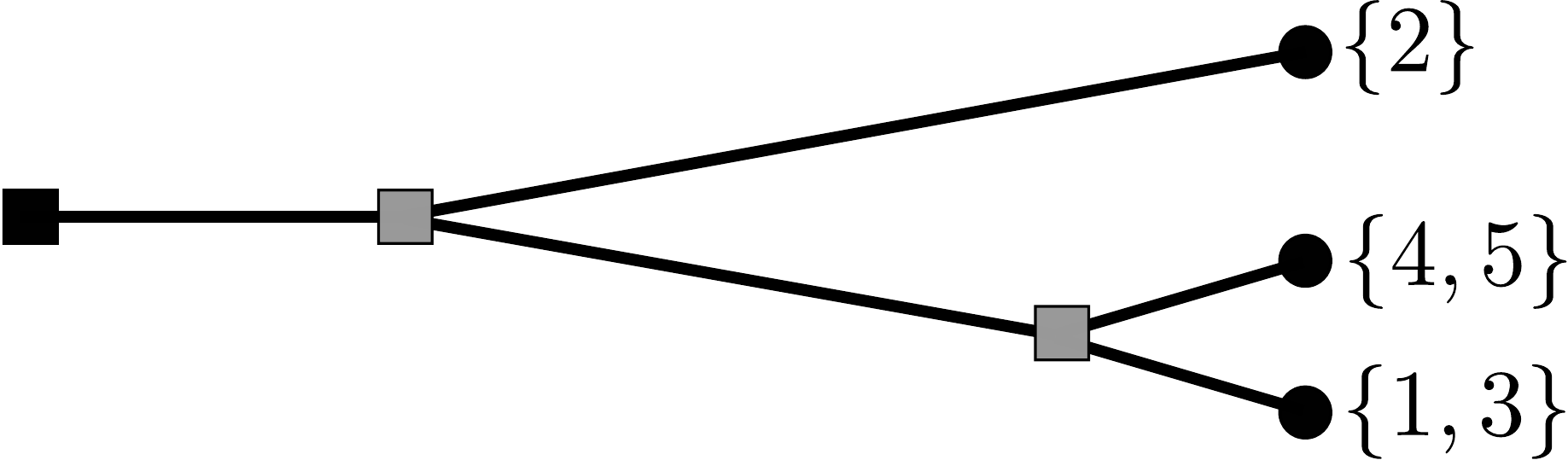}
               }
               \label{fig:DDT_demo_tree}
        }
       	\caption{(a) A simulated Dirichlet diffusion tree with $N=5$ objects and divergence parameter $\alpha=1$.  (b) The corresponding tree structure is shown where the divergence locations are the branch nodes in the tree.}
	\label{fig:DDT_demo}
\end{figure}
A simulated Dirichlet diffusion tree with $N=5$ objects, divergence parameter $\alpha = 1$, and Brownian motion paths is shown in \cref{fig:DDT_demo}.
The corresponding tree structure is also shown, where the origin is the root node, the divergence points define the internal nodes, the locations of the particles (in $\mathcal X$) at time $t=1$ define the leaves, and the segments between nodes are the branches in the tree structure.
The leaves of the tree structure define a partition of the $N$ objects, because each object is represented by only one particle that diffuses to a leaf node in the tree.

Note that if a particle does not diverge before $t=1$, then it joins an existing partition.
However, in many applications, e.g., density modeling or phylogenetic clustering, an agglomerative clustering of the objects is required (i.e., each subset in the partition is a singleton).  In such a case, a more general, positive-valued divergence function $\alpha(t)$ is typically used that ensures every particle will diverge before $t=1$.
The waiting time until divergence is then the time until the first jump in an inhomogeneous Poisson process with rate function $\alpha(t)/m$ (see for example \citet{NealDDT2003} and \citet{knowlesTPAMIPYDT}).

While non-parametric models for random partitions have been popular in probability and statistics for the last few decades \citep{pitman2006combinatorial,kingman1982genealogy,kingman1982coalescent,teh2006hierarchical,TehEA2006}, it is desirable in many applications to allow objects to belong to more than one cluster, i.e., to allow the clusters to overlap.
We call a set of (potentially) overlapping subsets of $[N] \defas (1,\dotsc,N)$ a \emph{feature allocation} (as termed by \citet{broderick2013feature} following the work of \citet{GG2006}), where we take $[N]$ to represent the sequence of $N$ objects.  
For example, a feature allocation of five objects, represented by their labels $\{1,\dotsc,5\}$, can be denoted by $\{ \{1,2\}, \{1,3\}, \{4,5\} \}$.
Here there are three features and we say that the first object is allocated to the features $\{1,2\}$ and $\{1,3\}$, etc.
Note, however, that this structure could never be captured by a Dirichlet diffusion tree like the one depicted in \cref{fig:DDT_demo}, because the particle representing the first object cannot travel along both branches at the first divergence point and end up at two different leaves in the tree.
In order to model hierarchical feature allocations, the beta diffusion tree overcomes this restriction.  In particular, the beta diffusion tree proceeds analogously to the Dirichlet diffusion tree, except that multiple copies of a particle (corresponding to multiple copies of an object) may be created (or destroyed) at random times.
Objects may therefore be represented by particles following multiple branches to multiple leaves in the tree.
The leaves then represent the subsets of a feature allocation that are hierarchically related through the tree structure.

\subsection{Related work}
\label{sec:related_work}

Most contemporary work on non-parametric feature allocation models (where the number of subsets in the feature allocation is unbounded) trace their origins to the IBP.
The IBP was originally described in terms of a random binary matrix $\bm Z$ with an unbounded number of \emph{independent} columns, where the entry $z_{n,k} = 1$ denotes that object $n$ is allocated to feature $k$ and $z_{n,k} = 0$ denotes that object $n$ is not allocated to feature $k$.
Similarly, we may view the beta diffusion tree as defining a random, infinite binary matrix, except that the columns are no longer independent; instead, they are hierarchically related through the tree structure.
A class of \emph{correlated IBP} models was introduced by \citet{DG2009}, which cluster the columns of an IBP-distributed matrix in order to induce (sparse) dependencies between the features.
For example, let $\bm Z^{(1)}$ be an IBP-distributed matrix, and conditioned on $\bm Z^{(1)}$, let $\bm Z^{(2)}$ be another IBP-distributed matrix whose rows correspond to the columns of $\bm Z^{(1)}$.
This scheme is extended to an arbitrary number of iterations by the \emph{cascading Indian buffet process} \citep{AWG2010}, in which the rows in an IBP-distributed matrix $\bm Z^{(m)}$ at iteration $m$ correspond to the columns in the IBP-distributed matrix $\bm Z^{(m-1)}$ at iteration $m-1$.
While the beta diffusion tree generalizes the``flat clustering'' of the correlated IBP to a hierarchical clustering, it does not obtain the general network structure obtained with the cascading IBP.
The beta diffusion tree is therefore an appropriate model for hierarchically clustering features, while the cascading IBP is an appropriate model for (random) neural network architectures.

Models based on the beta diffusion tree are not to be confused with the \emph{phylogenetic Indian buffet process} \citep{MGJ2008}, which hierarchically clusters the \emph{rows} (objects) in an IBP-distributed matrix.
This clustering is done in a non-probabilistic manner (the tree structures are not random), and the (random) allocations of objects to the features depend on the constructed tree structure over the objects.
A similar approach is taken by the \emph{distance-dependent IBP} \citep{GFB2011}, which assumes that there is an observed distance metric between objects.  If two objects are close, then they tend to share the same features.
Both of these models, unlike models based on the beta diffusion tree and the correlated IBP, assume that the features themselves are {\it a priori} independent.
Additionally, both models assume the objects are not exchangeable.
This is a distinctive difference to the rest of the models so far discussed, which are all exchangeable with respect to the ordering of the objects, a property of the beta diffusion tree that is established in \cref{sec:exchangeability}.

Previous hierarchical clustering models have been limited to partitions, the most relevant example being the Dirichlet diffusion tree, as discussed earlier in \cref{sec:diffusion_trees}.
A generalization of the Dirichlet diffusion tree to arbitrary branching degrees is achieved by the \emph{Pitman--Yor diffusion tree} \citep{KnowlesGhahramani2011}.  
While these stochastic processes do not have an immediate relation to feature allocation models, they can both be derived as the infinite limits of finite models defined by nested partitioning schemes \citep{bertoin2006random, teh2011modelling, KnowlesThesis2012, knowlesTPAMIPYDT}, motivating the derivation of the beta diffusion tree as the infinite limit of a finite model defined by nested feature allocation schemes in \cref{sec:continuum_limit}.
These ``diffusion tree'' processes so far discussed are all types of \emph{diffusion} (or \emph{fragmentation}) processes, which are continuous-time stochastic processes with a close relationship to \emph{coalescent} (or \emph{coagulation}) processes (see \citet{bertoin2006random} and \citet[Ch.~5]{pitman2006combinatorial}).
The prototypical example from the latter class of models is \emph{Kingman's coalescent} \citep{kingman1982coalescent}, and the reader may refer to \citet{teh2007treecoalescent} for an example of its use as a model for hierarchical clustering.
An extension beyond binary branching is obtained with the \emph{$\Lambda$-coalescent} \citep{pitman1999coalescents}, and the reader may refer to \citet{hu2013binary} for details on practical inference and application.
These stochastic processes all model continuous tree structures, which are most useful when modeling continuous variables associated with the hierarchy.  We will see examples using the beta diffusion tree in \cref{sec:LG}.
Probabilistic models for non-parametric, \emph{discrete} tree structures have also been studied \citep{blei2010nested, rodriguez2008nested, PWBJ2012, adams2010tree, steinhardt2012flexible}, which would be more appropriate for modeling discrete variables associated with the tree structure.
%

\subsection{Outline of the article}

In \cref{sec:BDTmain}, we define the beta diffusion tree and study its properties in depth, including a characterization as a generative (sequential) process and as the infinite limit of a finite model.
This latter characterization reveals several properties of the model that are useful for applications.
In \cref{sec:LG}, we demonstrate how to build a hierarchically clustered factor analysis model with the beta diffusion tree, and in \cref{sec:inference} we present a Markov chain Monte Carlo procedure for integrating over the tree structures.
These are applied in \cref{sec:experiments} to numerical experiments on missing data problems with data sets of gene expression microarrays, international development statistics, and intranational socioeconomic measurements.
In \cref{sec:conclusion}, we conclude with a discussion of future research directions.
%

\section{The beta diffusion tree}
\label{sec:BDTmain}

We describe a generative (sequential) process for a collection of diffusing particles that proceeds analogously to the generative process for the Dirichlet diffusion tree described in \cref{sec:diffusion_trees}.
In particular, each particle is labeled with one of $N$ objects and diffuses in $\mathcal X$ over a (fictitious) unit time interval $t \in [0,1]$.
We recall that the random tree structure being described does not depend on the choice of diffusion process.
If a particle is labeled with object $n$, then we call it an \emph{$n$-particle}, multiple of which may exist at time $t>0$.
Sequentially for every object $n=1,\dotsc,N$, we begin with one $n$-particle at the origin, which follows the paths of previous particles.  At random times $t^*$ throughout the process, any $n$-particle travelling along any path may perform one of two actions:
\begin{enumerate}
\item {\bf stop:} The particle stops diffusing at time $t^*$.

\item {\bf replicate:} A copy of the $n$-particle is created at time $t^*$.  One copy continues along the original path and the other diverges from the path and diffuses independently of all other particles.
\end{enumerate}

For simplicity, we describe the generative process here in terms of infinitesimal probabilities, though an equivalent (and more practical) representation with exponential random variables is provided in \cref{sec:prior_sampler}.
Precisely, let $\sfn, \bfn, \soffset$, and $\boffset$ be positive, finite constants that parameterize the generative process, which proceeds as follows:

\newpage

\begin{noindlist}
\item $\bm{n=1}$: A 1-particle starts at the origin and diffuses for $t >0$.  At time $0<t<1$:

\begin{noindlist}
\item[--] The particle \emph{stops} in the next infinitesimal time interval $[t, t+\dee t]$ with probability
\[
\sfn\, \dee t .
	\label{eq:stop_prob}
\]

\item[--] Conditioned on not stopping in $[t,t+\dee t]$, the particle \emph{replicates} in $[t,t+\dee t]$ with probability
\[
\bfn\, \dee t ,
	\label{eq:branch_prob}
\]
creating a copy of the $1$-particle.
Both particles diffuse (independently of each other) for $t> 0$, each stopping or replicating with the probabilities given by \cref{eq:stop_prob,eq:branch_prob}, respectively.
Arbitrarily label one of the paths as the ``original path'' and the other as the ``divergent path''.

\end{noindlist}

\item $\bm{n \ge 2}$: For every $n \ge 2$, a single $n$-particle starts at the origin and follows the path initially taken by the previous particles.
At time $0<t<1$, for a particle travelling on a path along which $m$ particles have previously travelled:

\begin{noindlist}

\item[--] The particle stops in $[t, t+\dee t]$ with probability
\[
\frac \soffset {\soffset+m} \sfn\, \dee t .
	\label{eq:stop_prob_general}
\]

\item[--] Conditioned on not stopping in $[t,t+\dee t]$, the particle replicates in $[t, t+\dee t]$ with probability
\[
\frac \boffset {\boffset+m} \bfn\, \dee t 
	,
	\label{eq:branch_prob_general}
\]
creating a copy of the $n$-particle.  One copy follows the original path, and the other copy diverges from the path and diffuses independently of all other particles, stopping or replicating with the probabilities given by \cref{eq:stop_prob,eq:branch_prob}, respectively.  The newly created path is labeled as the ``divergent path''.

\item[--] If the particle reaches an existing stop point (a point on the path where at least one previous particle has stopped), it also stops at this point with probability
\[
\frac{\numstop}{\soffset+m}
	,
	\label{eq:stopprob}
\]
where $\numstop$ is the number of particles that have previously stopped at this location.

\item[--] If the particle reaches an existing replicate point (a point on the path where a particle has previously replicated), the particle also replicates at this point with probability
\[
\frac{\numbranch}{\boffset + m}
	,
	\label{eq:branchprob}
\]
where $\numbranch$ is the number of particles that have previously replicated at this point (and taken the divergent path).
In this case, one copy of the particle follows the original path and the other follows the divergent path.
If the particle does not replicate, then it continues along the original path.

\end{noindlist}

\end{noindlist}

\begin{figure}[ht!]
\centering
\subfigure[A beta diffusion tree with $N=3$ objects]{
                \includegraphics[scale=.35]{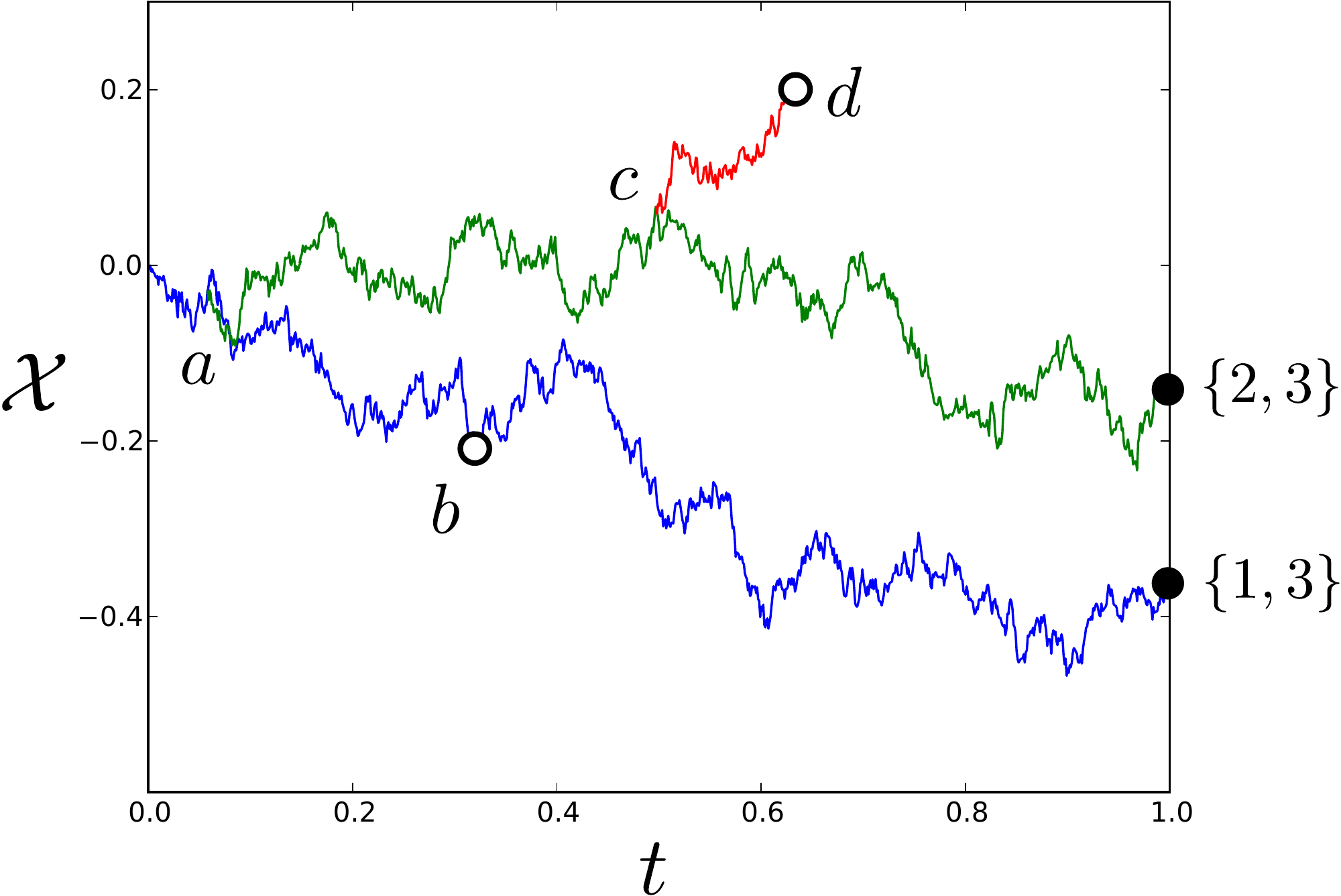}
                \label{fig:demo_tree_paths}
                }
        \quad \quad
\subfigure[Corresponding tree structure]{
            	\includegraphics[scale=.3]{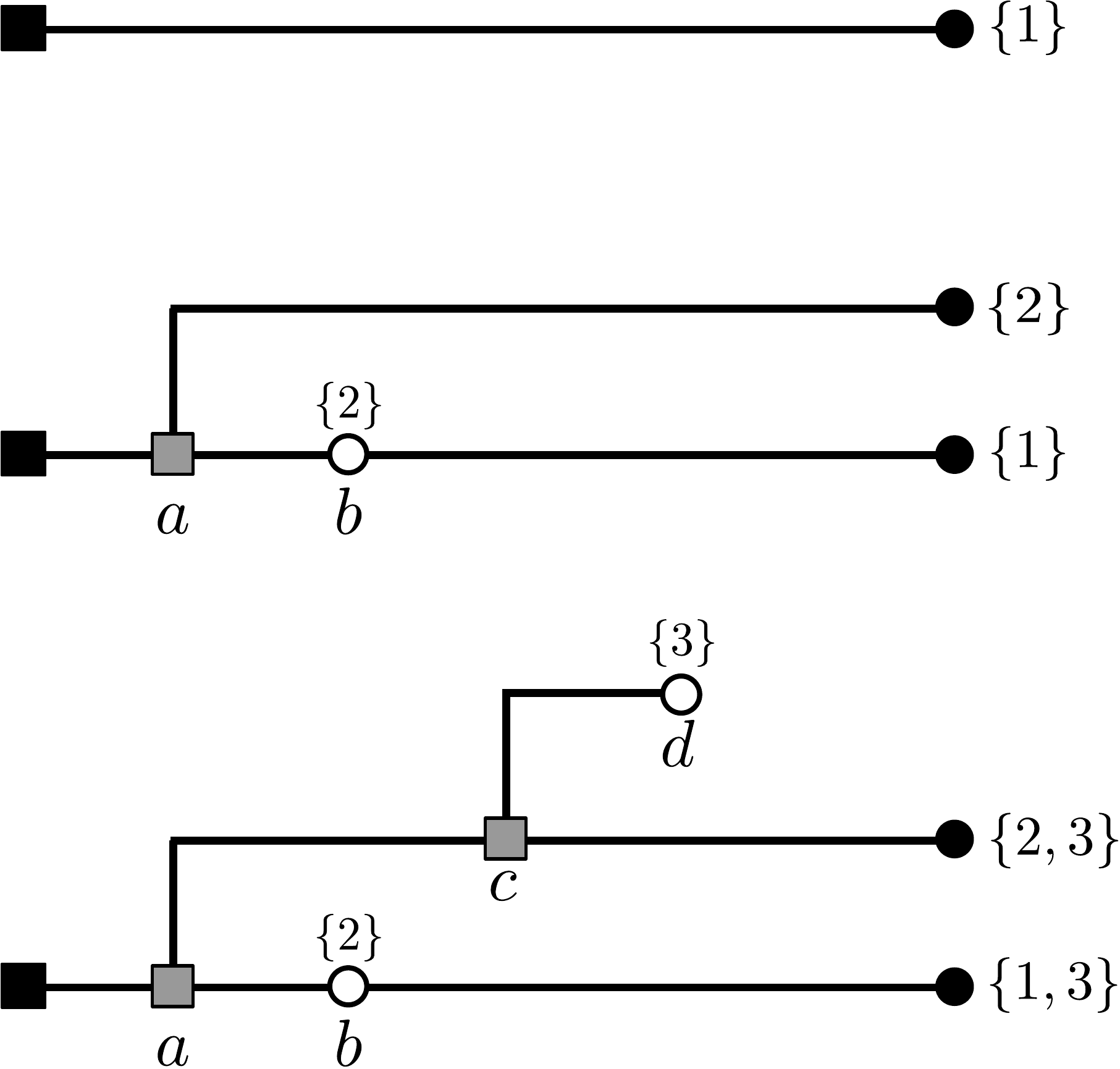}
		\label{fig:demo_tree_steps}
        	}
       	\caption{(a) A simulated beta diffusion tree with $N=3$ objects, where $b$ and $d$ are stop points and $a$ and $c$ are replicate points, resulting in the feature allocation $\{\{1,3\}, \{2,3\}\}$.  (b) The corresponding tree structures as each object is included in the generative process, where the origin is the root node, the replicate and stop points (grey blocks and open circles, respectively) are internal nodes, the features are leaf nodes, and segments between nodes are branches.  Here, the stop nodes have been additionally annotated with the (labels of) the particles that have stopped at the node.  The ``divergent'' branches are depicted as offshoots of the ``original'' branch.}
	\label{fig:demo_tree}
\end{figure}
The process terminates at $t=1$, at which point all particles stop diffusing.
Again, an equivalent representation in terms of exponential random variables is provided in \cref{sec:prior_sampler}.
In particular, note that the times until stopping or replicating on a path along which $m$ particles have previously travelled (described by \cref{eq:stop_prob_general,eq:branch_prob_general}, respectively) are exponentially distributed with rates $\sfn \soffset / (\soffset +m )$ and $\bfn \boffset / (\boffset+m)$, respectively.
In \cref{fig:demo_tree}, we show an example of a simulated beta diffusion tree with $N=3$ objects in $\mathcal X = \Reals$, where the paths are Brownian motion.
We also show the corresponding tree structure as each object is sequentially included in the generative process.  As with the Dirichlet diffusion tree, the root of the tree is the origin and the locations of the (clusters of) particles in $\mathcal X$ at time $t=1$ are the leaves.  With the beta diffusion tree, however, both replicate and stop points comprise the internal nodes of the tree, displayed as grey blocks and open circles, respectively.
We will refer to the origin as the \emph{root node}, stop points as \emph{stop nodes}, replicate points as \emph{replicate nodes}, and the points at $t=1$ as \emph{leaf nodes}.  We call segments between nodes \emph{branches}.
At a replicate point, we must keep track of which branch is the ``divergent'' branch and which is the ``original'' branch.  In \cref{fig:demo_tree_steps}, we therefore depict divergent branches as offshoots of the original branch.
We also annotate each stop node with the labels of the particle(s) that have stopped at that point.
Note that, given this information and the labels of the particles at the leaves, we may uniquely determine the path that every particle has taken throughout the tree.
Because multiple copies of a particle (all corresponding to the same object) can follow multiple branches to multiple leaves in the tree, the leaves define a feature allocation of the $N$ objects.
For example, the beta diffusion tree in \cref{fig:demo_tree} determines a feature allocation with two features $\{1,3\}$ and $\{2,3\}$.
The number of (non-empty) features is therefore the number of leaves in the tree structure, which is unbounded, however, we will see in the following section that this number is (almost surely) finite for any finite number of objects.

We now fix some notation.
We call $\sfn$ and $\bfn$ the \emph{stop} and \emph{replicate rate parameters}, respectively, which control the probabilities of instantaneously stopping and replicating.
We call $\soffset$ and $\boffset$ the \emph{stop} and \emph{replicate concentration parameters}, respectively, which control the probabilities of stopping at existing stop points and following existing divergent paths, respectively.
Denote by $\Tcal_{[N]}$ the collection of nodes, associated node times, and branches in a beta diffusion tree generated by the ordered sequence $[N] \defas (1, \dotsc, N)$, which we call the \emph{tree structure} of the beta diffusion tree.
Let $\Vcal(\Tcal_{[N]})$ denote the set of non-root nodes, and let $\bm x_{\Tcal_{[N]}} \defas \{ \bm x_v \}_{v\in \Vcal(\Tcal_{[N]})}$ denote the locations of the nodes in $\mathcal X$. 
While the index set of the stochastic process $(\Tcal_{[N]}, \bm x_{\Tcal_{[N]}})$ is the (ordered) set $[N]$, we will show in the next section that its distribution does not depend on this ordering.

\subsection{The generative process is exchangeable}
\label{sec:exchangeability}

The distribution of $(\Tcal_{[N]}, \bm x_{\Tcal_{[N]}})$ is characterized by the generative process, which depends on the ordering of the $N$ objects.
However, we typically do not want our models for feature allocations to depend on this ordering, and indeed we will now show that the density $p(\Tcal_{[N]}, \bm x_{\Tcal_{[N]}})$ is invariant to permutations of $[N]$.
Consider sequentially computing the factors that each object in the generative process contributes to the density function:
For each particle traversing the tree, the times until stopping or replicating on a branch are exponentially-distributed with the rates in \cref{eq:stop_prob_general,eq:branch_prob_general}, respectively.
Then the probability of neither replicating nor stopping between times $t$ and $t'$ (with $t<t'$) on a branch along which $m$ previous particles have traversed is given by
\[
\label{eq:prob_no_event}
\begin{split}
\Psi_m (t,t')
&\defas
\Pr \{ \text{not replicating and not stopping in } [t,t'] \}
	\\
	&
	\: = \exp \Bigl \{
			- \frac {\boffset} {\boffset + m} \bfn\, (t' - t) 
			- \frac {\soffset} {\soffset+m} \sfn\, (t' - t)
		\Bigr \}
	.
\end{split}
\]
Next note that the conditional probabilities that the particle replicates or stops at time $t'$, given that the particle did not replicate or stop in $[t,t']$, are given by $\boffset \bfn /(\boffset+m)$ and $\soffset \sfn / (\soffset+m)$, respectively.
Then the probability contributed by a branch in the tree structure spanning the interval $[t,t']$ is characterized by either
\[
\label{eq:conditionalreplicate}
\Pr \{ \text{replicate at } t' \wedge \text{did not replicate or stop in } [t,t'] \}
	= \frac {\boffset} {\boffset + m} \bfn \Psi_m(t,t'),
\]
or
\[
\Pr \{ \text{stop at } t' \wedge \text{did not replicate or stop in } [t,t'] \}
	= \frac {\soffset} {\soffset + m} \sfn \Psi_m(t,t')
	.
\]

For example, consider the tree with $N=3$ objects depicted in \cref{fig:demo_tree}, consisting of nodes $a, b, c, d, \{1,3\}$, and $\{2,3\}$, with corresponding times $t_a$, $t_b$, etc.
From the tree structure, we can determine that there is one 1-particle, two 2-particles, and three 3-particles.
The 1-particle does not replicate or branch before $t=1$ and therefore contributes $\Psi_0 (0,1)$ to the density.
The 2-particles contribute:

\begin{enumerate}[noitemsep]
\item $\frac{\boffset}{\boffset+1} \bfn \Psi_1 (0,t_a)$ for replicating at $t=t_a$, conditioned on not replicating or stopping in $t\in(0,t_a)$,
\item $\frac{\soffset}{\soffset+1} \sfn \Psi_1 (t_a,t_b)$ for stopping at $t=t_b$, conditioned on not replicating or stopping in $t\in(t_a,t_b)$,
\item $\Psi_0 (t_a,1)$ for not stopping or replicating in $t\in (t_a,1)$.
\end{enumerate}
The 3-particles contribute:
\begin{enumerate}[noitemsep]
\item $\Psi_2 (0,t_a)$ for not replicating or stopping in $t\in(0,t_a)$,
\item $\frac{1}{\boffset + 2}$ for taking the divergent path at $t=t_a$,
\item $\Psi_2 (t_a,t_b)$ for not replicating or stopping in $t\in(t_a,t_b)$,
\item $1-\frac{1}{\soffset + 2}$ for not stopping at $t=t_b$,
\item $\Psi_1 (t_b,1)$ for not replicating or stopping in $t\in(t_b,1)$,
\item $\frac{\boffset}{\boffset + 1} \bfn \Psi_1 (t_a,t_c)$ for replicating at $t=t_c$, conditioned on not replicating or stopping in $t\in(t_a,t_c)$,
\item $\Psi_1 (t_c,1)$ for not replicating or stopping in $t\in(t_c,1)$,
\item $\sfn \Psi_0 (t_c,t_d)$ for stopping at $t=t_d$, conditioned on not replicating or stopping in $t\in(t_c,t_d)$.
\end{enumerate}
Finally, we must account for the terms in the density contributed by the random diffusion paths.
Again, the diffusion process does not affect the tree structure.
For example, if the diffusion paths are Brownian motion with variance $\sigma_X^2$, then the components of the density resulting from the node locations $x_a$, $x_b$, etc., are
\[ 
&\Ncal( x_a; 0, \sigma_X^2 t_a )
	\:
	\Ncal( x_b ; x_a, \sigma_X^2 (t_b - t_a) )
	\:
	\Ncal( x_c ; x_a, \sigma_X^2 (t_c - t_a) )
	\:
	\Ncal( x_d ; x_c, \sigma_X^2 (t_d - t_c) )
	\label{eq:prob_tree}
	\\
	&\qquad \qquad \times
	\Ncal( x_{\{1,3\}} ; x_b, \sigma_X^2 ( 1-t_b) )
	\:
	\Ncal( x_{\{2,3\}} ; x_c, \sigma_X^2 (1-t_c) )
	.
	\nonumber
\]
There is one Gaussian term for each branch.

Because the behavior of (particles representing) objects in the generative process depends on the behavior of (particles representing) previous objects, the terms above depend on the ordering of the objects.
However, this dependence is superficial; if we were to multiply these terms together, we would find that the resulting expression for $p(\Tcal_{[3]}, \bm x_{\Tcal_{[3]}})$ does not depend on the ordering of the objects.
We generalize to an arbitrary number of objects in the following result, the proof of which parallels those by \citet{KnowlesThesis2012} and \citet{knowlesTPAMIPYDT}:

\begin{thm}[Exchangeability] \label{thm:exchangeability}
The beta diffusion tree is exchangeable with respect to the ordering of the objects.
\end{thm}

\begin{proof}
We compute the density $p(\Tcal_{[N]}, \bm x_{\Tcal_{[N]}})$ and show that it is invariant to permutations of $[N]$.
The density may be decomposed into the following parts:
\begin{enumerate}[noitemsep]
\item[(i)] the probability of $\Tcal_{[N]}$, the tree structure and associated node times, 
\item[(ii)] the probability of the node locations $\bm x_{\Tcal_{[N]}}$, given $\Tcal_{[N]}$.
\end{enumerate}
We will enumerate the factors contributing to each of these components as they are defined by the generative process, and show that neither depends on the ordering of the objects.
Let $\Bcal(\Tcal_{[N]})$ denote the set of branches in $\Tcal_{[N]}$.
For every branch $[uv] \in \Bcal(\Tcal_{[N]})$, going from node $u$ to node $v$ with corresponding locations $\bm x_u$ and $\bm x_v$ and times $t_u$ and $t_v$, let $m(v)$ denote the total number of particles that have traversed the branch.

\vspace{1em}
\noindent {\bf (i) Probability of the tree structure and node times.}
For branch $[uv]$, first consider the case when $v$ is a replicate node.
Let $\numbranch(v)$ denote the total number of particles that followed the divergent path at $v$.
Sequentially for each particle down the branch, let $c_r^v$ denote the number of particles that previously followed the divergent path at $v$, which ranges from 1 to $\numbranch(v)-1$.
Then the $\numbranch(v)$ particles that followed the divergent path contribute a term
\[
\prod_{c_r^v = 1}^{\numbranch(v) - 1} c_r^v 
	= (\numbranch(v) - 1)!
\]
to the numerator of the probability expression for the configuration at this replicate node.
Let $i_v$ be the index of the particle {\it out of the particles down this branch} that created the replicate point at $v$.
The $j$-th particle, for $j = i_v+1, \dotsc, m(v)$, contributes the term $\boffset + j -1$ to the denominator of the probability expression, regardless of whether or not it chose to replicate at $v$ and follow the divergent path.
If $\numbranch(v) = m(v)$, i.e., all particles replicated at $v$, then we don't include any more terms.
Otherwise, there are $m(v) - \numbranch(v)$ particles that chose not to replicate and follow the divergent path.
Sequentially for each one of these particles, let $d_r^v$ denote the number of particles that previously did \emph{not} replicate and follow the divergent path, which we note ranges from $i_v-1$ to $m(v)-\numbranch(v)-1$.
The particles that did not replicate at $v$ therefore contribute the following term to the numerator of the probability expression for the configuration at this node
\[
\prod_{d_r^v = i_v-1}^{m(v) - \numbranch(v) -1} (\boffset + d_r^v)
	= \frac{ \Gamma(\boffset + m(v) - \numbranch(v)) }{\Gamma(\boffset+i_v-1)}
	.
\]

Next consider the probability of the node time $t_v$:
All previous particles traversing this branch only contribute terms for not stopping or replicating, which will be considered later.
From \cref{eq:conditionalreplicate}, the conditional probability that the $i_v$-th particle replicates at $t_v$, given that the particle does not replicate or stop before $t_v$, is
\[
\frac{\boffset}{\boffset + i_v -1} \bfn
	.
	\label{eq:exch_1}
\]
Multiplying, the factor of the density function contributed by the replicate node $v$ is
\[
&\frac {\boffset \bfn} {\boffset + i_v -1} 
	\frac{ (\numbranch(v) - 1)! }{ \prod_{j=i_v+1}^{m(v)} (\boffset + j - 1) }
	\frac{ \Gamma(\boffset + m(v) - \numbranch(v)) }{\Gamma(\boffset+i_v-1)}
	\label{eq:exch_2}
	\\
	&\qquad \quad
	= \boffset \bfn\, 
		B( \boffset+m(v)-\numbranch(v), \numbranch(v) )
	,
	\nonumber
\]
where $B(a,b) = \Gamma(a)\Gamma(b)/\Gamma(a+b)$ denotes the beta function.
Note that this expression does not depend on the label $i_v$, and therefore changing the order of the particles down this branch does not affect the probability of the configuration at node $v$.
Now consider if $v$ is a stop node.
Let $\numstop(v)$ denote the number of particles that stopped at node $v$.
Then it follows analogously to the derivation above that the factor of the density function contributed by node $v$ is
\[
\soffset \sfn\, 
	B(\soffset + m(v) - \numstop(v), \numstop(v) )
	,
	\label{eq:exch_3}
\]
which likewise does not depend on the ordering of the particles down the branch.
Each replicate and stop node contributes such a term, for a total contribution of
\[
\label{eq:prob_nodes}
\begin{split}
&\prod_{u \in \Rcal (\Tcal_{[N]})} \Bigl [
	\boffset \bfn\, 
	B( \boffset + m(u) - \numbranch(u)), \numbranch(u) )
	\Bigr ]
	\\
	&\hspace{1in}
	\times
\prod_{v \in \Scal(\Tcal_{[N]})} \Bigr [
	\soffset \sfn\, 
	B( \soffset + m(v) - \numbranch(v), \numbranch(v) )
	\Bigr ]
	,
\end{split}
\]
where $\Rcal (\Tcal_{[N]})$ and $\Scal(\Tcal_{[N]})$ denote the sets of replicate and stop nodes, respectively.
The $m(v)-1$ particles that followed the first particle down the branch did not stop or replicate before $t_v$, which by \cref{eq:prob_no_event} contributes
\[
\begin{split}
\prod_{i=1}^{m(v)-1} \exp \Bigl \{
	- \frac {\boffset}{\boffset + i -1} \bfn (t_v - t_u)
	- \frac {\soffset}{\soffset+ i-1} \sfn (t_v - t_u)
	\Bigr \}
	\\
	= \exp \Bigl \{ 
		- (t_v - t_u) (\bfn \boffset H_{m(v)}^{\boffset} + \sfn \soffset H_{m(v)}^{\soffset})
	 \Bigr \}
		, 
\end{split}
\]
where $H_n^{\alpha} \defas \sum_{i=0}^{n-1} (\alpha + i)^{-1}$.
Each branch contributes such a term for a final factor of
\[
\prod_{[uv]\in\mathcal B (\Tcal_N)}  
	\exp \Bigl \{ 
		- (t_v - t_u) (\bfn \boffset H_{m(v)}^{\boffset} + \sfn \soffset H_{m(v)}^{\soffset})
	 \Bigr \}
		,
		\label{eq:prob_branches}
\]
which again does not depend on the ordering of the particles down the branches.
In practice, it may be convenient to note that $H_n^{\alpha} = \psi(\alpha+n) - \psi(\alpha)$, where $\psi(\alpha) \defas \Gamma'(\alpha)/\Gamma(\alpha)$ denotes the digamma function.

\vspace{1em}
\noindent {\bf (ii) Probability of the node locations.}
Again, the tree structure does not depend on the choice of diffusion process.
By choosing any diffusion process that does not depend on the ordering of the particles, the result holds.
Consider the example when the diffusion paths are Brownian motion, which will be of use in our experiments later.
We compute the contributions to the density function from the locations of the nodes $\bm x_{\Tcal_{[N]}}$, conditioned on the tree structure and node times $\Tcal_{[N]}$.
Generalizing \cref{eq:prob_tree}, we can see that each branch contributes a Gaussian factor, resulting in an overall factor of
\[
\prod_{[uv] \in \Bcal (\Tcal_{[N]})} 
	\Ncal( \bm x_v ; \bm x_u, \sigma_X^2 ( t_v - t_u) \bm I_D )
	.
	\label{eq:prob_locations}
\]
Again, this term does not depend on the ordering with which the objects were considered.
\end{proof}

The stochastic process $(\Tcal_{[N]}, \bm x_{\Tcal_{[N]}})$ is therefore \emph{exchangeable}, and because the ordering of the sequence $[N]$ is irrelevant, we will henceforth simply write $\Tcal_N$.
Let $\Tcal_N^{(n)}$ be the \emph{subtree structure} of $\Tcal_N$ associated with object $n$, i.e., the collection of nodes and branches in $\Tcal_N$ along which an $n$-particle has traversed.
Note that $\Tcal_N = \bigcup_{n \le N} \Tcal_N^{(n)}$ and that the generative process characterizes the distribution of $\Tcal_N^{(1)}$, followed by the conditional distribution of $\Tcal_N^{(n)}$, given $\Tcal_N^{(n-1)}, \Tcal_N^{(n-2)}, \dotsc, \Tcal_N^{(1)}$, for every $n\ge 1$.
It is therefore clear that the beta diffusion tree is projective, i.e., that the model for $N-1$ objects is the same as a model for $N$ objects with the last object removed.
That is, the density of a beta diffusion tree with $N-1$ objects $( \Tcal_{N-1} , \bm x_{\Tcal_{N-1}} )$ is the same as the density of $$\Bigl ( \bigcup_{n=1}^{N-1} \Tcal_N^{(n)} , \bm x_{ \bigcup_{n = 1}^{N-1} \Tcal_N^{(n)} } \Bigr ),$$ where $( \Tcal_{N} , \bm x_{\Tcal_{N}} )$ is a beta diffusion tree with $N$ objects,
for any $N \in \Nats$, where $\Nats \defas \{ 1, 2, \dotsc \}$ denotes the natural numbers.
We may therefore define a stochastic process by a beta diffusion tree with set of objects $\Nats$, the associated tree structure $\Tcal_{\Nats}$ of which is a tree structure over feature allocations of $\Nats$.
%

\subsection{The infinite limit of a nested feature allocation scheme}
\label{sec:continuum_limit}

We now provide an alternative characterization of the beta diffusion tree as the infinite limit of a finite model.
Let there be $L$ levels in a \emph{nested feature allocation scheme} of $\Nats$ defined as follows:
At the first level, we allocate the natural numbers $\Nats$ to two different sets $f_1$ and $f_2$, which we call \emph{nodes}, independently with probabilities $p_0^{(1)}$ and $p_0^{(2)}$, respectively, where $p_0^{(1)}$ and $p_0^{(2)}$ are independent random variables with
\[
\label{eq:level_probs}
\begin{split}
p_0^{(1)} & \dist \betadist \Bigl ( \, \soffset \Bigl ( 1 - \frac{\sfn}{L} \Bigr ) ,\, \soffset \frac{\sfn}{L}\, \Bigr ) 
	,
	\\
p_0^{(2)} & \dist \betadist \Bigl (\, \boffset \frac{\bfn}{L} ,\, \boffset \Bigl ( 1 - \frac{\bfn}{L} \Bigr )\, \Bigr )
	.
\end{split}
\]
That is, for every $n \in \Nats$, we have that $n\in f_1$ with probability $p_0^{(1)}$ and $n\in f_2$ with probability $p_0^{(2)}$.
At the next level, we allocate the elements in $f_1$ to two different nodes $f_{11}$ and $f_{12}$ at level two, independently with probabilities $p_{f_{1}}^{(1)}$ and $p_{f_1}^{(2)}$, respectively, where $p_{f_1}^{(1)} \equaldist p_0^{(1)}$ and $p_{f_1}^{(2)} \equaldist p_0^{(2)}$.
The elements in $f_2$ are likewise allocated to two nodes $f_{21}$ and $f_{22}$ at level two independently with probabilities $p_{f_2}^{(1)}$ and $p_{f_2}^{(2)}$, with $p_{f_2}^{(1)} \equaldist p_0^{(1)}$ and $p_{f_2}^{(2)} \equaldist p_0^{(2)}$.
Because each node is a subset of a node in the level above, we can treat the nodes as the internal vertices in a binary branching tree structure.
We show a diagram of this scheme with two levels along with the associated tree structure in \cref{fig:finite_tree_small}.
Note that the subscript of each node's label, which is a sequence of the digits $1$ and $2$ with length equal to the level associated with the node, uniquely characterizes the path through the ancestors of the node in the tree.
Each node in the tree is a subset of all of its ancestors in the tree.  Therefore, if any node is equal to the empty set $\emptyset$, then all of its descendants in the tree are equal to the empty set.
For example, a node labelled $f_{21}$ (at level two) has parent node $f_2$ (at level one), and we have that $f_{21} \subseteq f_{2} \subseteq \Nats$.
\begin{figure}
\centering
\subfigure[Nested feature allocations]{
                \includegraphics[scale=.35]{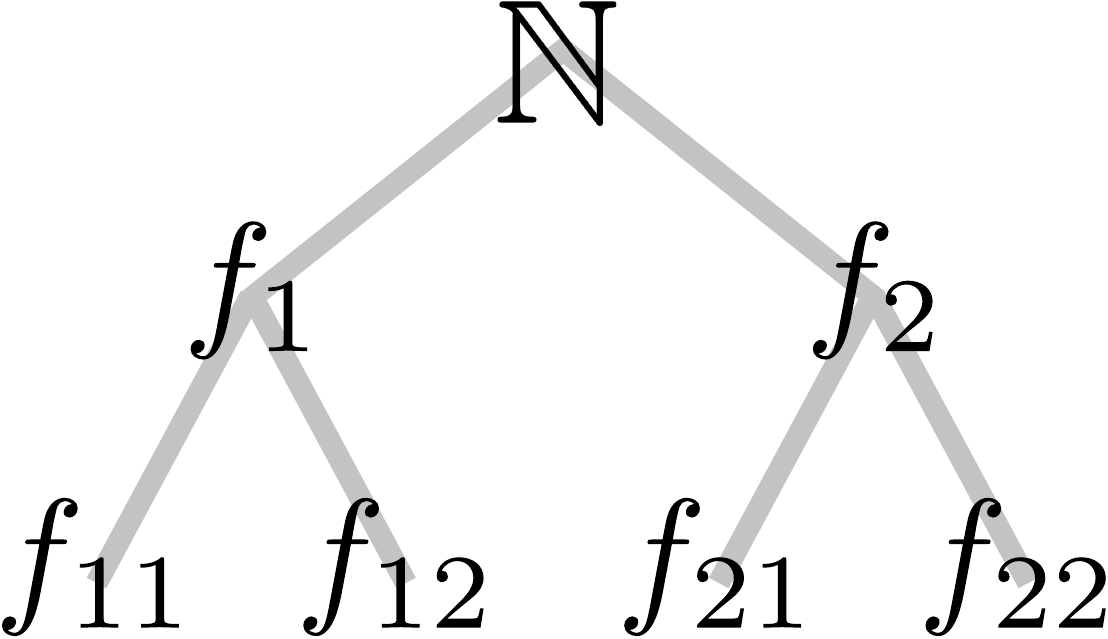}
                \label{fig:finite_tree_small}
                }
        \qquad \qquad
\subfigure[The tree structure associated with the non-empty features]{
                \includegraphics[scale=.3]{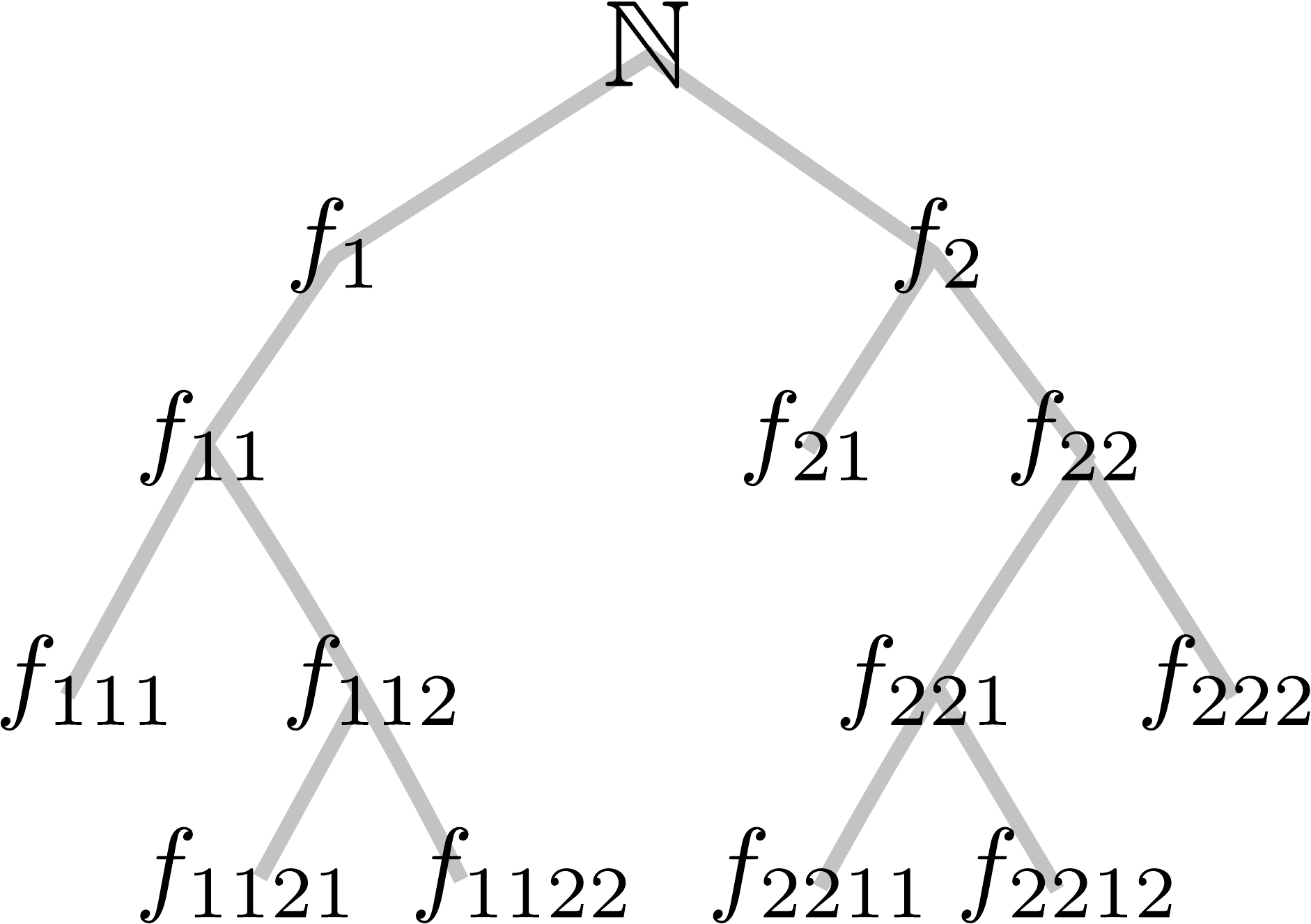}
                \label{fig:finite_tree_large}
        }
       	\caption{Depictions of a nested feature allocation scheme.  In (a) we show the nested scheme for $L=2$ levels, where the elements in each node at level one are allocated to two nodes at level two.  In (b) we show the tree structure corresponding to an $L=4$ level deep scheme, where the vertices in the tree are non-empty nodes.}
	\label{fig:finite_tree}
\end{figure}

We may continue this scheme recursively for $L$ levels, where we allocate the elements in every (non-empty) node $f_{\bullet}$ at level $\ell-1$ to two nodes $f_{\bullet 1}$ and $f_{\bullet 2}$ in level $\ell$, independently with probabilities $p_{f_{\bullet 1}}^{(1)}$ and $p_{f_{\bullet 2}}^{(2)}$, with $p_{f_{\bullet 1}}^{(1)} \equaldist p_0^{(1)}$ and $p_{f_{\bullet 2}}^{(2)} \equaldist p_0^{(2)}$.
Here the prefix $\bullet$ is some length--($\ell-1$) sequence of the digits $1$ and $2$ identifying the ancestral path of node $f_{\bullet}$ through the tree.
Associate each node at level $\ell\le L$ of the scheme with the time index
\[
t_\ell = (\ell-1)/L
	.
	\label{eq:pseudotime}
\]
Let $\Tcal_{\Nats,L}$ denote the tree structure and associated node times defined on the non-empty nodes, i.e., each node that is not equal to the empty set is a vertex in the tree, and the branches are the segments between nodes.
For example, we show the tree structure $\Tcal_{\Nats,L}$ for a $L=4$ level scheme in \cref{fig:finite_tree_large}.
Taking the limit $L\rightarrow \infty$ of this model, we obtain the tree structure of a beta diffusion tree:
\begin{thm}[infinite limit] \label{thm:continuum_limit}
In the limit $L \rightarrow \infty$, the sequence of tree structures $\Tcal_{\Nats,L}$ recovers the beta diffusion tree with set of objects $\Nats$.
\end{thm}

This result is analogous to previous results that derive the Dirichlet and Pitman--Yor diffusion trees as the infinite limits of nested partitioning models \citep{bertoin2006random, teh2011modelling, KnowlesThesis2012, knowlesTPAMIPYDT}.
To gain some intuition, consider the tree structure $\Tcal_{1,L}$ allocating only the element $\{1\}$ in the nested scheme.
At any level, we may integrate out the beta random variables $p_{f_\bullet}^{(1)}$ and $p_{f_\bullet}^{(2)}$ associated with node $f_\bullet$ from the model, so that $\{1\}$ is allocated to the first node at the next level $f_{\bullet 1}$ with probability $1 - \sfn / L \rightarrow 1$, as $L\rightarrow \infty$, and the second node $f_{\bullet 2}$ with probability $\bfn/L \rightarrow 0$, as $L\rightarrow \infty$.
Note that the probability of these two events co-occurring is $(1-\sfn/L)(\bfn/L) = \bfn L^{-1} - \bfn\sfn L^{-2} \rightarrow 0$, as $L\rightarrow \infty$.
Therefore, in the limit, the number of levels until $\{1\}$ is allocated to either zero or two nodes will be exponentially-distributed with rates $\sfn$ and $\bfn$, respectively.
We can see that this model is equivalent to the beta diffusion tree, where the event that an element is allocated to zero or two nodes at some level $\ell$ in the nested scheme is associated with the event that a particle stops or replicates at time $t_\ell$ in the beta diffusion tree.
With this interpretation, an element being allocated to exactly one node at any level in the nested scheme (which in the limit will be node $f_{\bullet 1}$ almost surely) is equivalent to a particle diffusing along an existing path in the beta diffusion tree.

\vspace{1em}
\begin{proof}(of \cref{thm:continuum_limit}).
Let $f_\bullet = \{n_1, n_2, \dotsc, n_m\}$ be a non-empty node (with $m$ elements) at some level in the nested scheme and consider sequentially allocating its elements to the nodes $f_{\bullet 1}$ and $f_{\bullet 2}$ at the subsequent level with probabilities $p_{f_\bullet}^{(1)}$ and $p_{f_\bullet}^{(2)}$, respectively, where $p_{f_\bullet}^{(1)} \equaldist p_0^{(1)}$ and $p_{f_\bullet}^{(2)} \equaldist p_0^{(2)}$, with $p_0^{(1)}$ and $p_0^{(2)}$ specified as in \cref{eq:level_probs}.
We will call $f_{\bullet 1}$ the ``first node'' and $f_{\bullet 2}$ the ``second node''.

\vspace{.5em}
\noindent {\bf (i) Probability of stopping:}
The conditional distribution of $p_{f_\bullet}^{(1)}$, given the allocations of $m$ elements is
\[
p_{f_\bullet}^{(1)} \given n_1
	\dist
	\betadist \Bigl (  \soffset \Bigl ( 1 - \frac{\sfn}{L} \Bigr ) + n_1, \soffset \frac{\sfn}{L} + m - n_1 \Bigr )
	,
\]
where $n_1 \defas \cardsharp f_{\bullet 1}$ is the number of elements allocated to $f_{\bullet 1}$.
Marginalizing over $p_{f_\bullet}^{(1)}$, it follows that element $n_{m+1}$ is allocated to $f_{\bullet 1}$ with probability
\[
\Pr \{ n_{m+1} \in f_{\bullet 1} \given n_1 \}
	= \frac{ \soffset (1-\sfn/L) + n_1}{ \soffset+m }
	.
	\label{eq:feat1_predictive}
\]
First consider when $n_1=m$, i.e., all previous elements were allocated to $f_{\bullet 1}$.
Then this probability becomes $1-\frac{ \soffset \sfn/L } {\soffset+m}$, which approaches one as $L\rightarrow \infty$.
Then the element $n_{m+1}$ is \emph{not} allocated to $f_{\bullet 1}$ with probability $\frac{\soffset \sfn/L}{\soffset+m}$.
In the language of the beta diffusion tree, an object ``stops'' at time $t_\ell$ if it is not allocated to the ``first node'' at level $\ell$.

More precisely, begin with a node at level $\ell$ that contains the element $n_{m+1}$ and consider following the allocation of this object to the ``first node'' in each subsequent level.
The number of levels $k$ until $n_{m+1}$ is not allocated to the first node is distributed as
\[ 
\frac{ \soffset \sfn/L }{ \soffset+m } \
		\Bigl (
			1 - \frac{ \soffset \sfn/L }{ \soffset+m }
		\Bigr )^{k-1}
	.
	\label{eq:stop_prob_L}
\]
This is a geometric distribution with parameter $\frac{ \soffset \sfn/L }{ \soffset+m }$.
Then $t_s \defas (b+k-1)/L - t_b = k/L$ is the time until $n_{m+1}$ is not allocated to the first node (i.e., until stopping), and it is a standard result that, in the limit $L\rightarrow \infty$, the distribution of $t_s = k/L$ converges to an exponential distribution with rate
\[
\frac{ \soffset }{ \soffset+m } \sfn
	.
\]
Now return to \cref{eq:feat1_predictive} and consider the case when $n_1 < m$, i.e., one or more elements (before the element $n_{m+1}$) were not allocated to $f_{\bullet 1}$.
Then the conditional probability that element $n_{m+1}$ is not allocated to $f_{\bullet 1}$, given the allocations of the previous $m$ elements, is given by
\[
1 - \frac{\soffset+n_1}{\soffset+m} 
	= \frac{m-n_1}{\soffset+m}
	.
\]
Noting that $m-n_1$ is the number of previous elements not allocated to $f_{\bullet 1}$, this is the probability of stopping at a previous stop point in the beta diffusion tree (c.f. \cref{eq:stopprob}).

\vspace{.5em}
\noindent {\bf (ii) Probability of replicating:}
Now consider allocating element $n_{m+1}$ to $f_{\bullet 2}$.  Given the allocations of $m$ elements, the distribution of $p_{f_\bullet}^{(2)}$ is
\[
p_\bullet^{(2)} \given n_2
	\dist \betadist \Bigl ( \boffset \frac{\bfn}{L} + n_2, \boffset \Bigl ( 1-\frac{\bfn}{L} \Bigr ) + m - n_2 \Bigr )
	,
\]
where $n_2 \defas \cardsharp f_{\bullet 2}$ is the number of elements allocated to $f_{\bullet 2}$.
Then, given the allocations of the previous $m$ elements, element $n_{m+1}$ is allocated to $f_{\bullet 2}$ at level $\ell$ with probability
\[
\Pr \{ n_{m+1} \in f_{\bullet 2} \given n_2 \}
	= \frac{\boffset \bfn/L + n_2}{\boffset+m}
	.
	\label{eq:feat2_predictive}
\]
Consider when $n_2=0$, i.e., no elements were previously allocated to $f_{\bullet 2}$.
Then this probability becomes $\frac{\boffset \bfn/L}{\boffset+m}$, which approaches zero in the limit $L\rightarrow \infty$.
Because an element is allocated to the first node with probability one (in the limit $L\rightarrow\infty$), then we can say that the corresponding particle replicates at time $t_\ell$ if the element is allocated to the second node at level $\ell$.
We can follow the derivation above to show that the time until the element $n_{m+1}$ is allocated to a ``second node'' at some level is exponentially-distributed with rate
\[
\frac{ \boffset }{ \boffset+m } \bfn
	.
\]
Returning to \cref{eq:feat2_predictive}, if $n_2 > 0$, i.e., one or more previous elements were allocated to $f_{\bullet 2}$, then in the limit $L\rightarrow \infty$, \cref{eq:feat2_predictive} becomes $n_2/(\boffset+m)$.  Indeed, this is the probability of replicating at a previous replicate point in the beta diffusion tree (c.f. \cref{eq:branchprob}).

\vspace{.5em}
\noindent {\bf (i) Probability of co-occurrence:}
Finally, consider the probability that $f_\bullet$ allocates $n_{m+1}$ to $f_{\bullet 2}$ but not to $f_{\bullet 1}$.
In the limit $L\rightarrow \infty$ this corresponds to the event that $n_{m+1}$ creates a replicate and stop node simultaneously.
By \cref{eq:feat1_predictive} and \cref{eq:feat2_predictive}, this occurs with probability
\[
\Bigl (
\frac{ \soffset \sfn }{ \soffset + m }
	\Bigr )
	\Bigl (
	\frac{ \boffset \bfn }{ \boffset + m }
	\Bigr )
	L^{-2}
	\rightarrow 0
	\quad \text{as} \quad
	L\rightarrow \infty
	,
\]
as desired.
\end{proof}

We note that the result above can be generalized to arbitrary replicate and stop functions, $\bfn(t)$ and $\sfn(t)$, by following the approach taken by \citet{KnowlesThesis2012} and \citet{knowlesTPAMIPYDT}.

\subsection{Properties}
\label{sec:properties}

The nested feature allocation scheme perspective in the previous section enables us to give one further characterization of the beta diffusion tree, derived in \cref{sec:galtonwatson}, as a continuous-time, \emph{multitype Markov branching process} \citep{mode1971, athreya1972branching, harris2002theory}.
In this characterization, branches in the tree structure of the beta diffusion tree are interpreted as the lifetimes of individuals in a population, where each individual is labeled as one of $N$ \emph{types} determined by the number of particles along the branch.
Replicate nodes are interpreted as points when the individual gives birth to another, and stop nodes are interpreted as points when the individual either changes type or dies.

Several established results on these stochastic processes reveal characteristics of the beta diffusion tree that make it a useful model in applications.
For example, we show that the branching process defined by the beta diffusion tree satisfies the so-called \emph{non-explosion hypothesis}, which states that the population will never become infinite in a finite amount of time.
This translates into the following result for the beta diffusion tree:
%
%
\begin{thm} \label{result:finiteleaves}
If $\Tcal_N$ is the tree structure of a beta diffusion tree with a finite set of $N$ objects and finite parameters $\sfn, \bfn, \soffset, \boffset$, then the number of leaves in $\Tcal_N$ is almost surely finite.
\end{thm}
%
%
The practical interpretation of this result being that the number of latent components in a model will be (almost surely) finite for a finite amount of data points.

Furthermore, we may characterize the expected number of leaves in the tree structure.
In \cref{sec:galtonwatson} we show that the \emph{infinitesimal generating matrix} of this stochastic process, where the state of the process is given by the number of individuals (of each type) in the population, is characterized by the $N\times N$ lower triangular matrix $\bm G = (g_{i,j})_{i,j\le N}$ with 
\[
\label{eq:generatormatrixbrief}
g_{i,j}
	=
	\begin{cases}
	\binom{i}{j} \Bigl [
		\soffset \sfn B(\soffset+j, i-j)
		+ \boffset \bfn B(\boffset + i-j, j )
		\Bigr ]
		,
		\qquad
		& 1\le j < i
		\\
	\boffset \bfn B(\boffset, i) - \soffset \sfn H_i^{\soffset}
	,
	\qquad
	& j=i
	,
	\\
	0 , \qquad & j > i
	,
	\end{cases}
\]
where, as in the proof of \cref{thm:exchangeability}, $H_n^{\alpha} \defas \psi(\alpha+n) - \psi(\alpha)$ and $B(\alpha,\beta) \defas \Gamma(\alpha) \Gamma(\beta) / \Gamma(\alpha+\beta)$.
In the theory of branching processes, it is well known that the matrix exponential of $\bm G t$ characterizes the expected population at time $t$.
Because we have associated the leaves in a beta diffusion tree with time $t=1$, we have the following:
%
%
\begin{thm} \label{result:expectedleaves}
Let $N, \sfn, \bfn, \soffset, \boffset$, and $\Tcal_N$ be as in \cref{result:finiteleaves}.
For every $j\le N$, let $K_{N,j}$ denote the number of leaves in $\Tcal_N$ with exactly $j$ objects, and let $\numfeat{N} = \sum_{j=1}^N \numfeat{N,j}$ denote the total number of leaves in $\Tcal_N$.
Let $\bm M \defas \exp( \bm G )$ denote the matrix exponential of $\bm G$ with elements $\bm M = (m_{i,j})_{i,j\le N}$.
Then
\[
\EE [\numfeat{N,j}]
	= m_{N,j}
	,
	\qquad
	j\le N
	,
	\label{eq:expectedleafdensity}
\]
and so
\[
\EE [ \numfeat{N} ]
	=
	\sum_{j=1}^N m_{N,j}
	.
	\label{eq:expectedleaves}
\]
\end{thm}
%
Analytically computing these expectations from the matrix exponential of a triangular matrix is systematic, though in practice we will usually resort to numerical evaluations.
As one would expect, the replicate rate parameter $\bfn$ has the greatest influence on the expected number of features in \cref{eq:expectedleaves}, as depicted in \cref{fig:expectedleaves}.
\begin{figure}[!t]
\centering
\includegraphics[scale=.5]{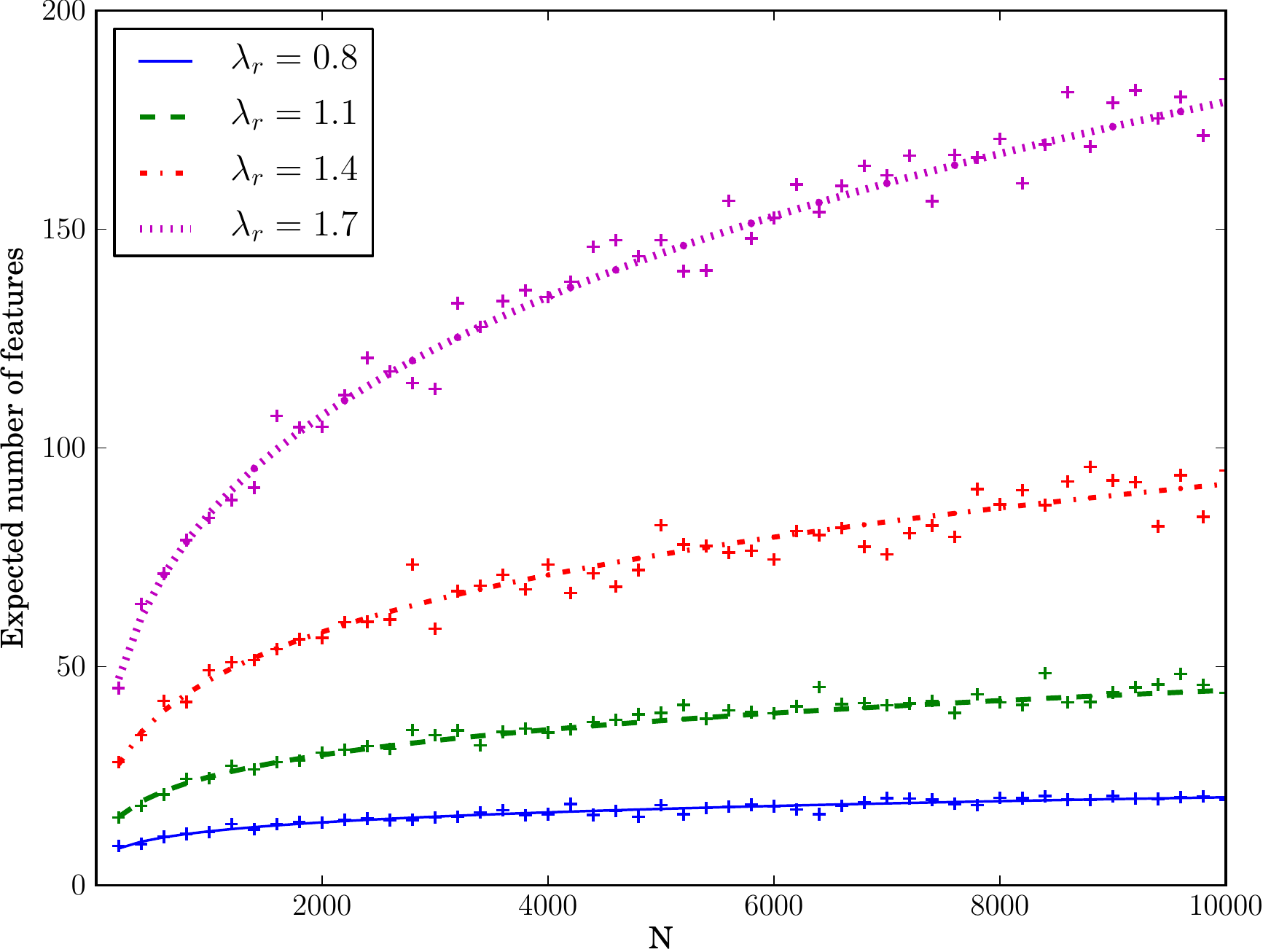}
\caption{Expected number of features (plotted as lines) in the beta diffusion tree for different values of $\bfn$, along with the number of leaves in trees simulated from the model (plotted as points).}
\label{fig:expectedleaves}
\end{figure}
%
%

\section{Application: A hierarchical factor analysis model}
\label{sec:LG}

We have seen that the leaves of a beta diffusion tree determine a latent feature allocation, and that the tree structure defines a hierarchical clustering of the features.
We show another simulated beta diffusion tree with $N=150$ objects in \cref{fig:large_demo_tree}.
\begin{figure}[!t]
\centering
\subfigure[Tree structure]{
                \includegraphics[scale=.4]{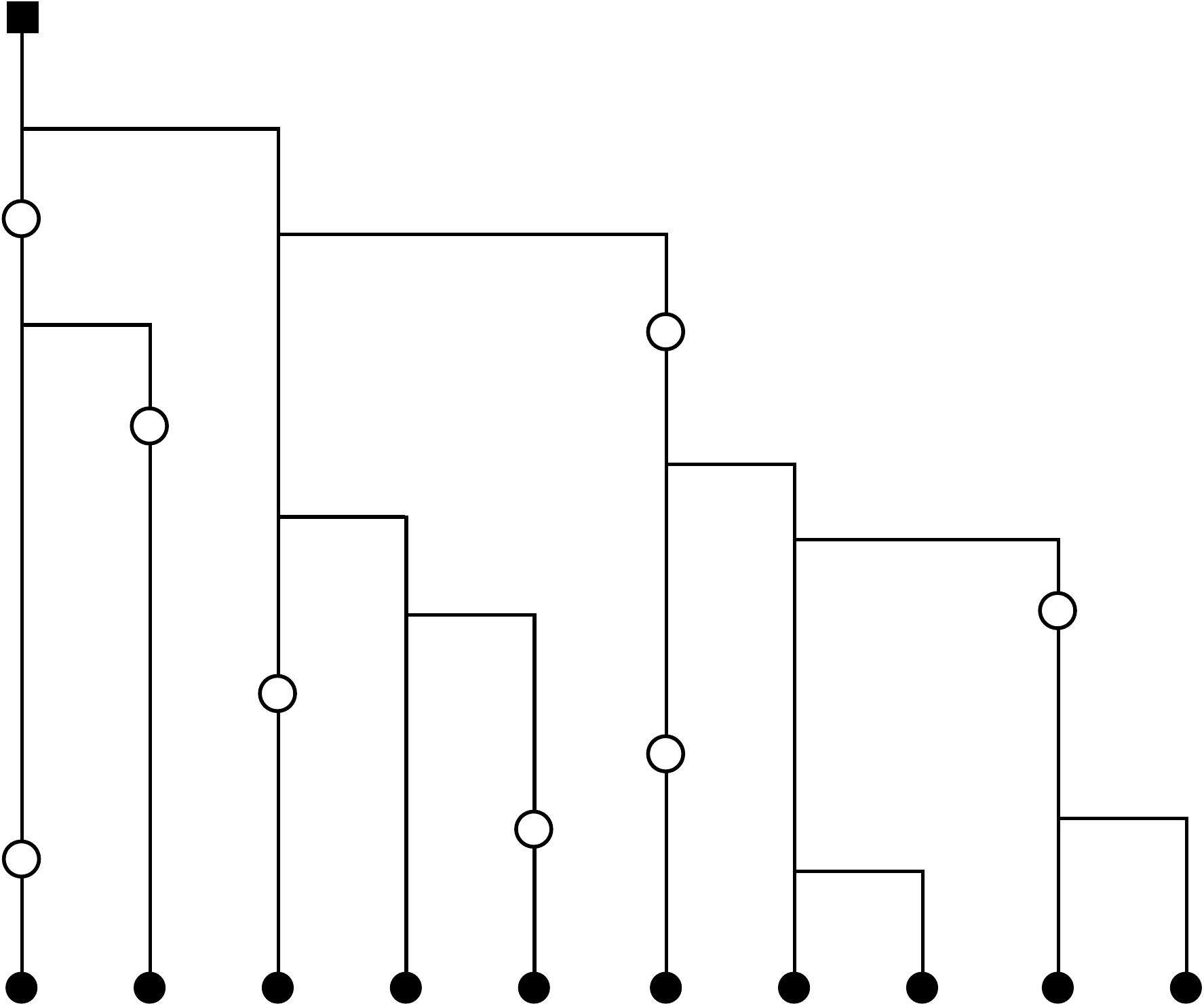}
                \label{fig:large_demo_tree}
                }
        \qquad
\subfigure[Latent feature matrix]{
                \includegraphics[scale=.3]{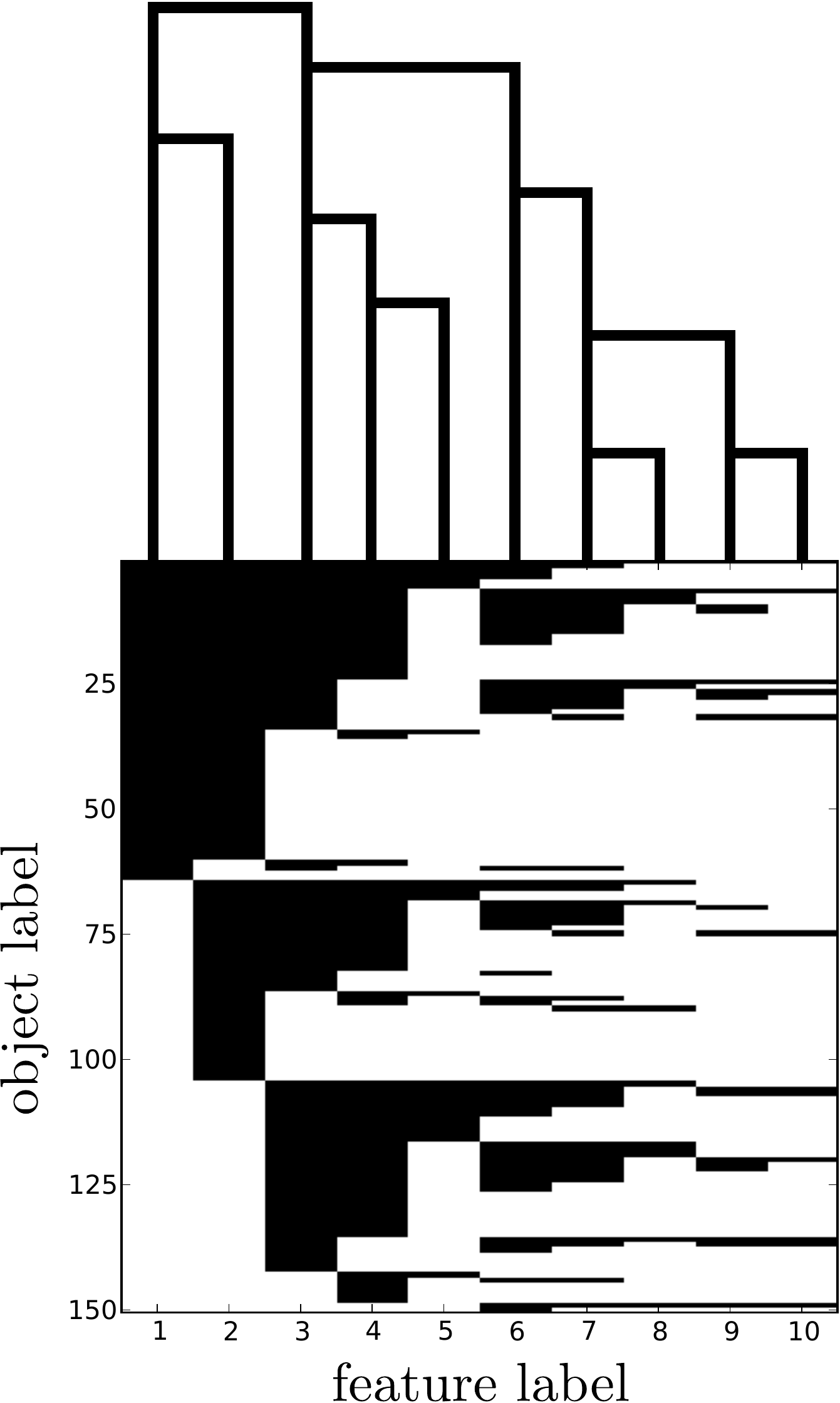}
                \label{fig:large_demo_Zmat}
        }
       	\caption{(a) A simulated beta diffusion tree with $N=150$ objects, where an open circle represents a stop node.  (b) The corresponding representation as a binary feature matrix, where the columns are interpreted as the features.  The hierarchical clustering defined by the tree structure is shown over the columns.  For visualization, the rows have been sorted according to their binary representations.}
	\label{fig:large_demo}
\end{figure}
Let $\numfeat{N}$ denote the number of leaf nodes (features), which we have seen is unbounded yet almost surely finite (e.g., $\numfeat{150} = 10$ in \cref{fig:large_demo_tree}).
In this example, it is convenient to represent the feature allocation as a binary matrix, which we will denote as $\bm Z$, where the $n$-th row $\bm z_n \in \{0,1\}^{\numfeat{N}}$ indicates the features to which object $n$ is allocated, i.e., $z_{n,k} = 1$ indicates object $n$ is allocated to feature $k$ and $z_{n,k}=0$ otherwise.
Then each column of $\bm Z$ represents a feature, and the tree structure defines a hierarchical clustering of the columns, depicted in \cref{fig:large_demo_Zmat}.

In applications, we typically associate the objects allocated to a feature with a set of feature-specific latent parameters.
The objects can be observed data that depend on the latent parameters, or the objects can themselves be additional unobserved variables in the model.
A convenient choice for a set of continuous-valued latent parameters associated with each feature (i.e., each leaf node in the beta diffusion tree) are the locations of the leaf nodes in $\mathcal X$.
Consider the following example:
Let $\bm Z$ be the binary matrix representation of the feature allocation corresponding to a beta diffusion tree with $\numfeat{N}$ leaf nodes, and let $\mathcal X = \Reals^D$.
Recall that the $k$-th column of $\bm Z$ corresponds to a leaf node in the tree with diffusion location $\bm x_k$ in $\Reals^D$ at time $t=1$ (c.f. \cref{fig:demo_tree}).
We model a collection of $N$ data points $\bm y_1, \dotsc, \bm y_N$ in $\Reals^D$ by
\[
\bm y_n = \bm z_n^T \bm X + \bm \varepsilon_n
	,
	\qquad n\le N
	,
	\label{eq:LG_model}
\]
where $\bm X$ is a $\numfeat{N} \times D$ factor loading matrix whose $k$-th row is given by $\bm x_k$, and $\bm \varepsilon_1, \dotsc, \bm \varepsilon_n$ are \iid\ Gaussian noise vectors with zero mean and covariance $\sigma_Y^2 \bm I_D$.
Here $\sigma_Y$ is a noise parameter.
Let $\bm Y$ be the $N\times D$ matrix with its $n$-th row given by $\bm y_n$.
Then $\bm Y$ is matrix Gaussian and we may write $\EE [ \bm Y \given \mathcal T_N, \bm X ] = \bm Z \bm X$.
This is a factor analysis model that generalizes the models studied by \citet{GG2011} and \citet{DG2009}. 
In the former, the latent features (columns of $\bm Z$) are independent and in the latter, the features are correlated via a flat clustering.  In both models, the factor loadings $\bm x_1, \dotsc, \bm x_{\numfeat{N}}$ are mutually independent.
With the beta diffusion tree, however, \emph{both} the latent features and factor loadings are hierarchically-related through the tree structure.
A graphical model comparison of these three models is shown in \cref{fig:LG_GM}.
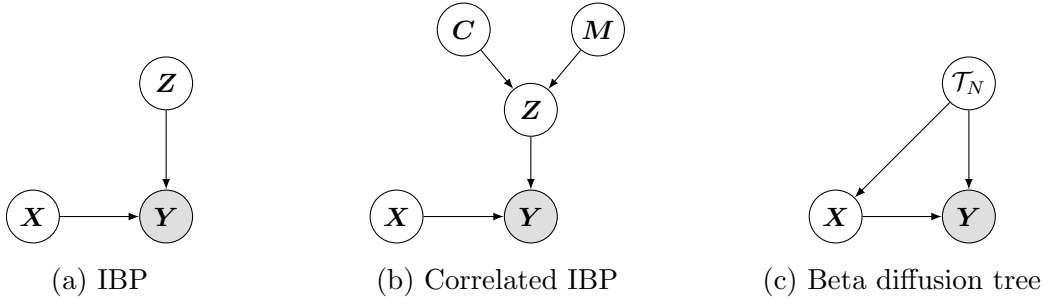
\begin{figure}[ht]
\begin{minipage}[b]{1\linewidth}
\centering
\begin{tabular}{ccc}
\begin{tikzpicture}[scale=1, transform shape]  
\node [matrix,matrix anchor=mid, column sep=30pt, row sep=30pt] {
&\node (Z) [black, latent] {$\bm Z$}; \\
\node (X) [black, latent] {$\bm X$}; & \node (Y) [black, obs] {$\bm Y$};  \\
};
    
\draw [->]  (X) to (Y) ;
\draw [->]  (Z) to (Y) ;

\end{tikzpicture}
\quad  \quad \qquad & \quad \quad
\begin{tikzpicture}[scale=1, transform shape]  
\node [matrix,matrix anchor=mid, column sep=5pt, row sep=10pt] {
&\node (C) [black, latent] {$\bm C$}; & & \node (M) [black, latent] {$\bm M$}; \\
& &\node (Z) [black, latent] {$\bm Z$}; \\
& [2pt] & \\
\node (X) [black, latent] {$\bm X$}; & & \node (Y) [black, obs] {$\bm Y$};  \\
};
    
\draw [->]  (X) to (Y) ;
\draw [->]  (Z) to (Y) ;
\draw [->]  (C) to (Z) ;
\draw [->]  (M) to (Z) ;

\end{tikzpicture}
\quad \quad & \quad \quad
\begin{tikzpicture}[scale=1, transform shape]  
\node [matrix,matrix anchor=mid, column sep=30pt, row sep=30pt] {
&\node (T) [black, latent] {$\mathcal T_N$}; \\
\node (X) [black, latent] {$\bm X$}; & \node (Y) [black, obs] {$\bm Y$};  \\
};
    
\draw [->]  (X) to (Y) ;
\draw [->]  (T) to (Y) ;
\draw [->]  (T) to (X) ;

\end{tikzpicture}
\\
(a) IBP
\quad \quad \qquad & \quad \quad
(b) Correlated IBP
\quad \quad & \quad \quad
(c) Beta diffusion tree
\end{tabular}
\end{minipage}
\caption{A graphical model comparison of a latent feature linear Gaussian model, where the feature allocation matrix $\bm Z$ is modeled by (a) the IBP, (b) a correlated IBP, and (c) the beta diffusion tree.
In (a), the columns of the matrix $\bm Z$ are independent.
In (b), two clustering matrices $\bm C$ and $\bm M$ are combined to induce dependencies between the columns of $\bm Z$ (see \cite{DG2009} for examples).  In (c), the matrix $\bm Z$ is a deterministic function of the tree structure $\mathcal T_N$, and the factor loadings $\bm X$ are also hierarchically related through $\Tcal_N$.  For simplicity, the hyperparameters in each model are not displayed.}
\label{fig:LG_GM}
\end{figure}

For the remainder, we let the diffusion paths in the beta diffusion tree be Brownian motions with variance $\sigma_X^2$ (c.f. \cref{eq:brownian_motion}).
In this case, we may analytically integrate out the specific paths that were taken by the particles, along with the locations of the internal (non-leaf) nodes in the tree structure.
It is easy to show that, conditioned on $\mathcal T_N$, the distribution of $\bm X$ is matrix Gaussian with conditional density
\[ 
p( \bm X \given \mathcal T_N )
	= \frac 1 { (2\pi)^{D\numfeat{N}/2} \sigma_X^{\numfeat{N} D} \vert \bm V \vert^{D/2} }
		\exp \left \{
		- \frac 1 {2 \sigma_X^2} \text{tr} [
			\bm X^T \bm V^{-1} \bm X
		]
	\right \}
	,
	\label{eq:GP_likel}
\]
where $\bm V$ is a $\numfeat{N} \times \numfeat{N}$ matrix with entries given by
\[
\nu_{\ell,k} = 
\begin{cases}
	t_{a(\ell,k)} , \qquad &\text{ if } \ell \ne k ,
	\\
	1 , \qquad &\text{ if } \ell=k ,
\end{cases}
\]
and $t_{a(\ell,k)}$ is the node time of the most recent common ancestor to leaf nodes $\ell$ and $k$ in the tree.
Because $\bm Y$ and $\bm X$ are both Gaussian, we may then
analytically integrate out the factor loadings $\bm X$ from the model, giving the resulting conditional density of the data
\[
\label{eq:likel}
\begin{split}
p(\bm Y \given \mathcal T_N )
		&= \frac 1 {(2\pi)^{ND/2} \sigma_Y^{(N-\numfeat{N})D} \sigma_X^{D\numfeat{N}} \bigl \vert \bm V \bm Z^T \bm Z + \frac{\sigma_Y^2}{\sigma_X^2} \bm I_{\numfeat{N}} \bigr \vert^{D/2} }
		\\
		&\qquad \times
		\exp \Bigl \{
				- \frac 1 {2 \sigma_Y^2} 
					\tr \Bigl [
					 \bm Y^T \Bigl ( 
					 \bm I_N
					 - \bm Z \Bigl ( \bm V \bm Z^T \bm Z + \frac{\sigma_Y^2}{\sigma_X^2} \bm I_{\numfeat{N}} \Bigr )^{-1} 
					 \bm V \bm Z^T 
					 \Bigr ) \bm Y
					 \Bigr ]
			\Bigr \}
		.
\end{split}
\]
Finally, by \cref{eq:likel,eq:GP_likel}, the conditional distribution of $\bm X$, given $\Tcal_N$ and $\bm Y$, is matrix Gaussian and has the conditional density
\[ \label{eq:posteriorloadings}
\begin{split}
p( \bm X \given \Tcal_N, \bm Y )
	&=
	\frac 1 { (2\pi)^{D\numfeat{N}/2} \sigma_Y^{D\numfeat{N}} \det{ \bm Q }^{D/2} } 
	\\
	&\qquad \qquad \times
	\exp \Bigl \{
		-\frac 1 {2 \sigma_Y^2} \text{tr}
		\Bigl [
			(\bm X - \bm Q \bm Z^T \bm Y)^T \bm Q^{-1} (\bm X - \bm Q \bm Z^T \bm Y)
		\Bigr ]
	\Bigr \}
	,
\end{split}
\]
where
$\bm Q \defas 
\bigl ( \bm V \bm Z^T \bm Z + \frac{\sigma_Y^2}{\sigma_X^2} \bm I_{\numfeat{N}} \bigr )^{-1} \bm V$.

It is straightforward to modify this model for different applications.
For example, if we require the factor loadings to be non-negative, we may replace the model in \cref{eq:LG_model} with
\[
\bm y_n &= \bm z_n^T \bm H + \bm \varepsilon_n
	,
	\qquad n \le N ,
	\\
h_{k,d} &= \exp \{ x_{k,d} \}
	,
	\qquad d\le D, \: \forall k
	,
	\label{eq:link_fn}
\]
where $h_{k,d}$ are the entries of the factor loading matrix $\bm H$.
In this case, because the factor loadings are no longer Gaussian, we cannot analytically integrate them out of the model as in \cref{sec:LG}, and we must numerically integrate over their values during inference.
Because $\bm H$ is a deterministic function of the diffusion locations $\bm X$, this task is equivalent to integrating over $\bm X$, which requires sampling from its conditional distribution given by \cref{eq:posteriorloadings}.

Another simple extension of our model could have diffusion in $D$ dimensions and data in $p$ variates, obtained with a model such as
\[
\bm y_n = \bm z_n^T \bm X \bm W + \bm \varepsilon_n
	,
	\qquad
	n\le N
	,
\]
where $\bm y_n$ is a $p$-vector representing the $n$-th measurement, as before $\bm X$ is the $\numfeat{N}\times D$ matrix of diffusion locations representing the factors, $\bm W$ is a $D\times p$ matrix of variate-specific factor weights, and $\bm \varepsilon_n$ is Gaussian noise.
The interpretation here is that the latent factors are hierarchically-related and each have different effects across the observed covariates.
The covariates themselves could be assumed {\it a priori} independent with independent Gaussian priors on $\bm W$, for example, or dependencies can be induced with an appropriate prior \citep[e.g., see the model by][]{palla2012nonparametric}.

\section{Inference}
\label{sec:inference}

We construct a Markov chain Monte Carlo procedure to jointly integrate over the tree structures $\Tcal_N$.
The sampler iterates between a series of moves: resampling subtrees, adding and removing replicate and stop nodes, and resampling the configurations at the replicate and stop nodes.
Each move proposes a sample of $\mathcal T_N$ (the tree structure and node times), which is accepted or rejected with a Metropolis--Hastings step.
We verify our procedure correctly targets the posterior distribution over tree structures with several joint distribution tests \citep{Geweke2004}, the results for which we supply in \cref{sec:Geweke}.

\subsection{Resampling subtrees}
\label{sec:resample_subtree}

%
The following move resamples \emph{subtrees} rooted at some internal node in the tree structure:
First select a subtree with probability proportional to the number of particles that have traversed the subtree.
Uniformly at random, select one of the particles that traversed the subtree and remove the particle \emph{and all replicates} from the subtree.
Resample the paths the particle takes down the remaining subtree according to the prior, conditioned on the paths of all other particles.
Once the particle and any created replicates have stopped or reached time $t=1$, then we obtain a new tree structure $\Tcal_N^*$.
A depiction of this procedure making one subtree proposal is shown in \cref{fig:inference}.
We can see that the larger the selected subtree, the larger the proposed move in the latent parameter space.

We accept or reject the proposal $\mathcal T_N^*$ with a Metropolis--Hastings step.
Because we sampled the proposed subtree from the prior, we do not need to include terms for the probability of the tree structures in the Metropolis--Hastings ratio.
Recall that $\Vcal(\Tcal_N)$ denotes the set of non-root nodes in the tree structure $\Tcal_N$, and that $m(u)$ denotes the number of particles through node $u$.
Let $v^* \in \Vcal(\Tcal_N)$ be the first node following the root node of the selected subtree.
Note then that this subtree was selected with probability proportional to $m(v^*)$, and the acceptance probability in the case of the factor analysis model from \cref{sec:LG} reduces to
\[ 
A_{\mathcal T_N \rightarrow \mathcal T_N^*}
	&\defas
	\max \left \{
		1 , 
		\frac { p ( \bm Y \given \Tcal_N^* ) \: m(v^*) / \sum_{u \in \Vcal(\Tcal_N^*)} m(u) }
				{ p ( \bm Y \given \Tcal_N ) \: m(v^*) / \sum_{u \in \Vcal(\Tcal_N)} m(u) }
		\right \}
		\\
	&\: =
	\max \left \{
		1 , 
		\frac { p ( \bm Y \given \Tcal_N^* ) \: \sum_{u \in \Vcal (\Tcal_N)} m(u) }
				{ p ( \bm Y \given \Tcal_N ) \: \sum_{u \in \Vcal (\Tcal_N^*)} m(u) }
		\right \}
		,
		\label{eq:MH_accept}
\]
where $p ( \bm Y \given \Tcal_N )$ is the likelihood function given by \cref{eq:likel}.
\begin{figure*}[t!]
\begin{minipage}[b]{1\linewidth}
\begin{center}
\begin{tabular}{cccc}
\includegraphics[scale=.35]{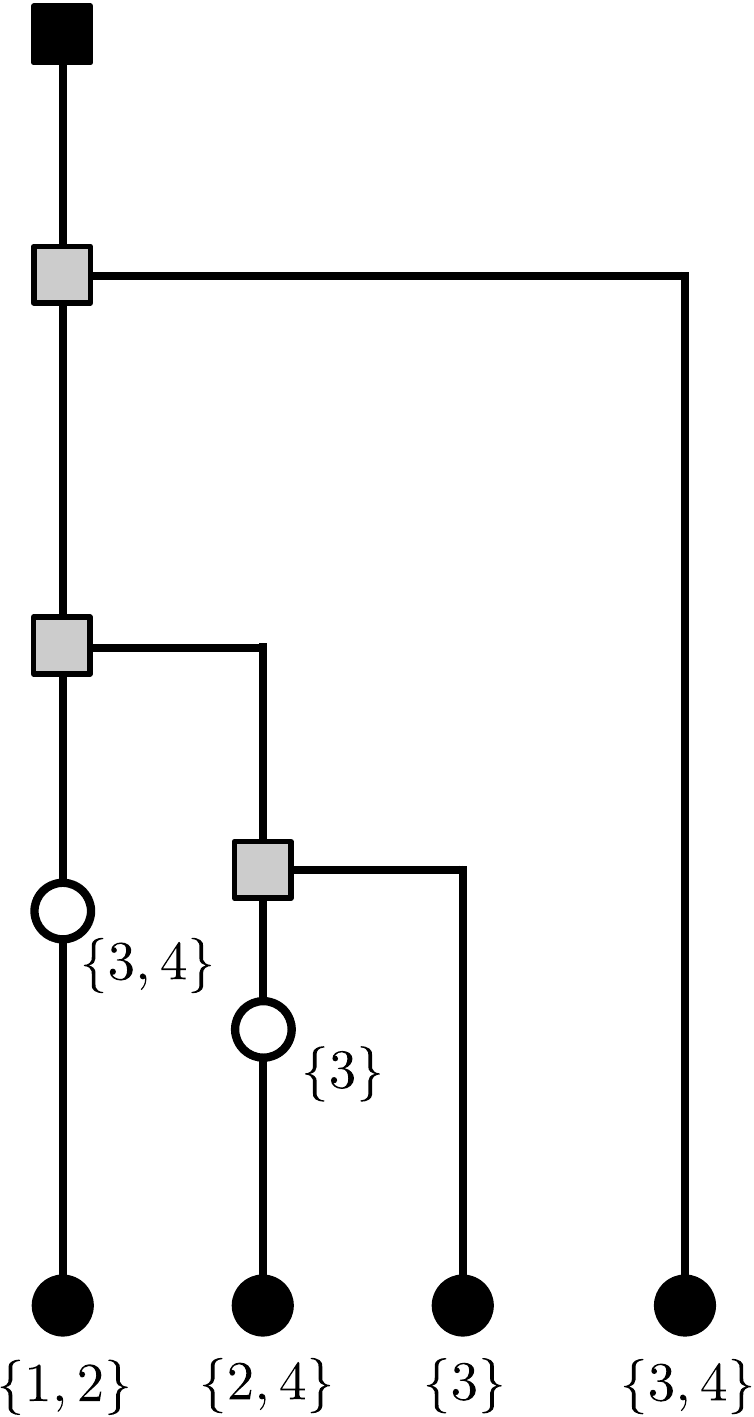}
&
\quad
\includegraphics[scale=.35]{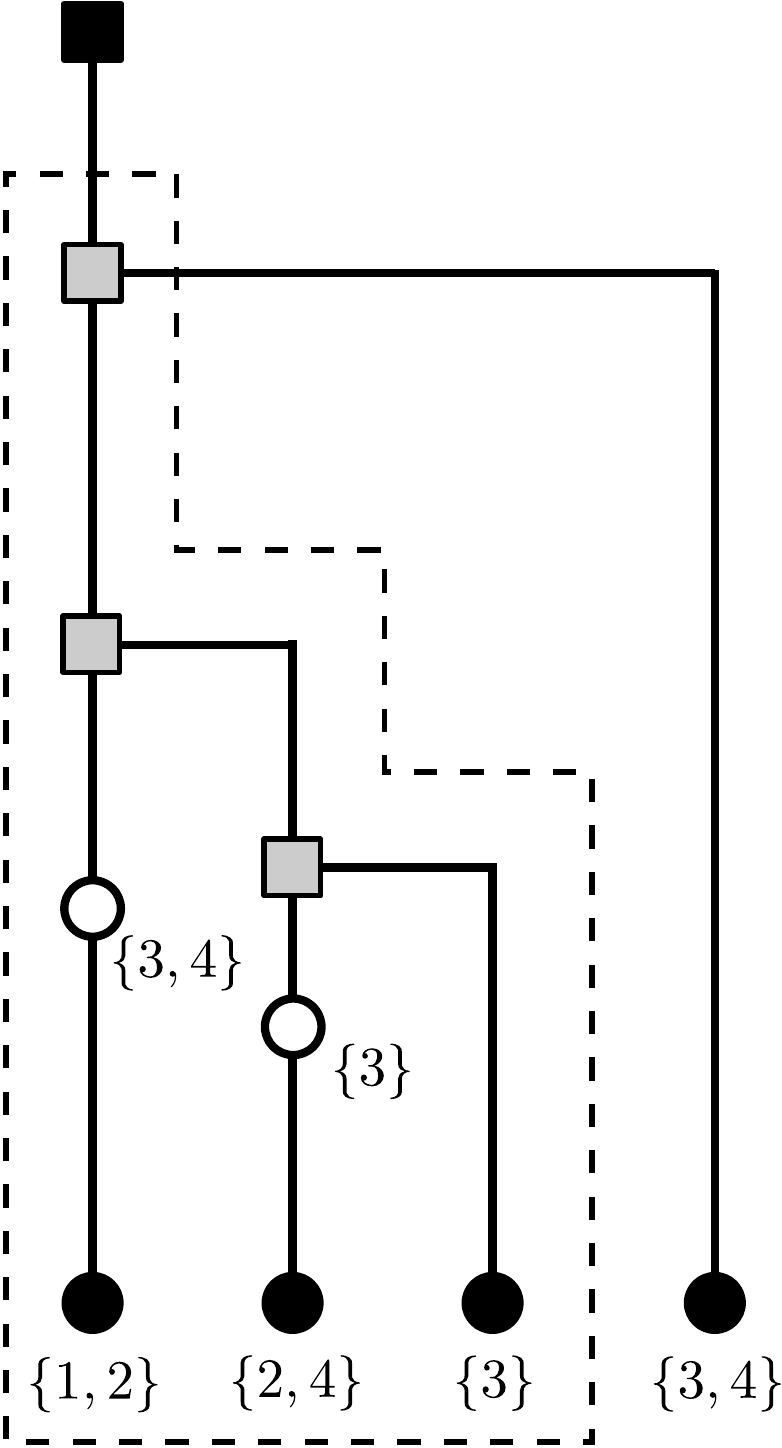}
&
\quad
\includegraphics[scale=.35]{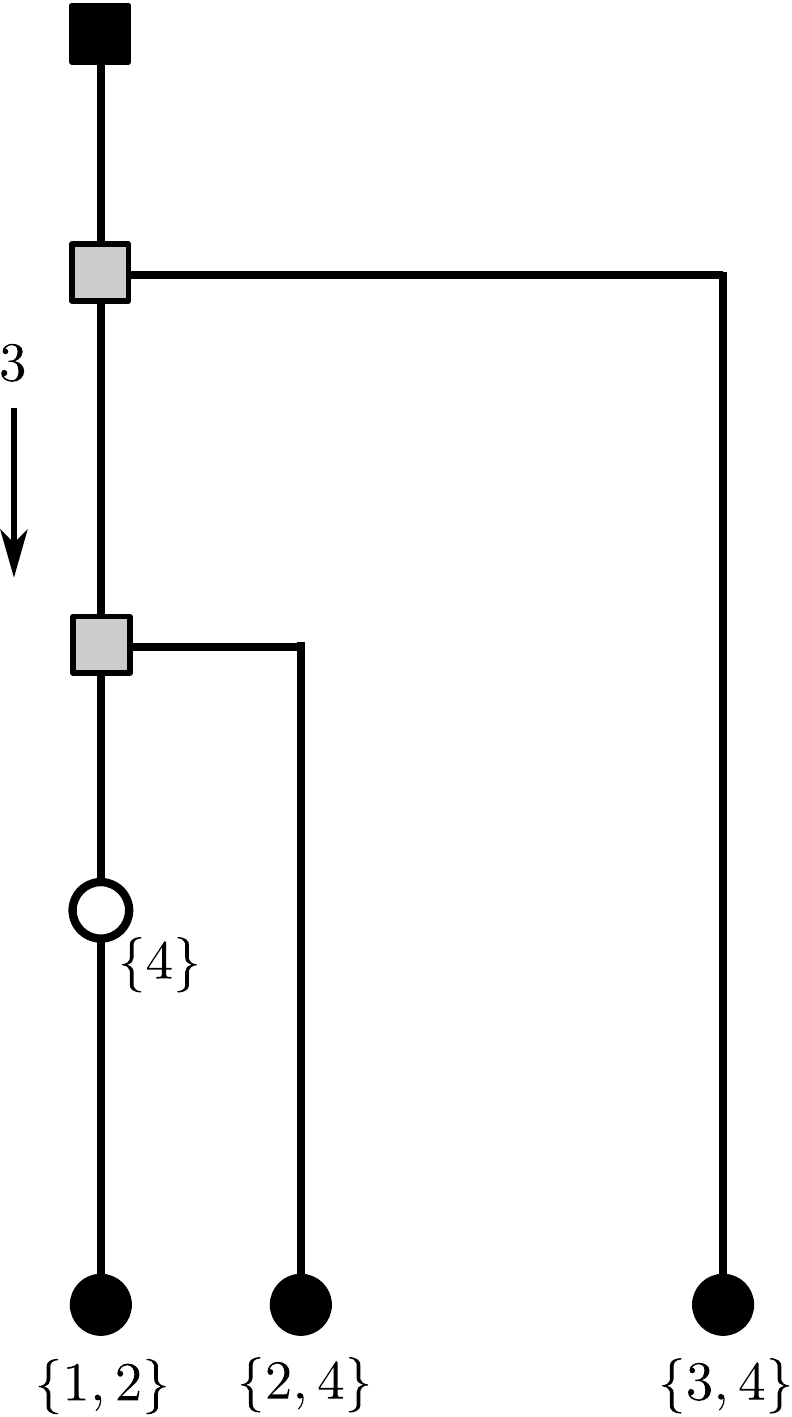}
&
\quad
\includegraphics[scale=.35]{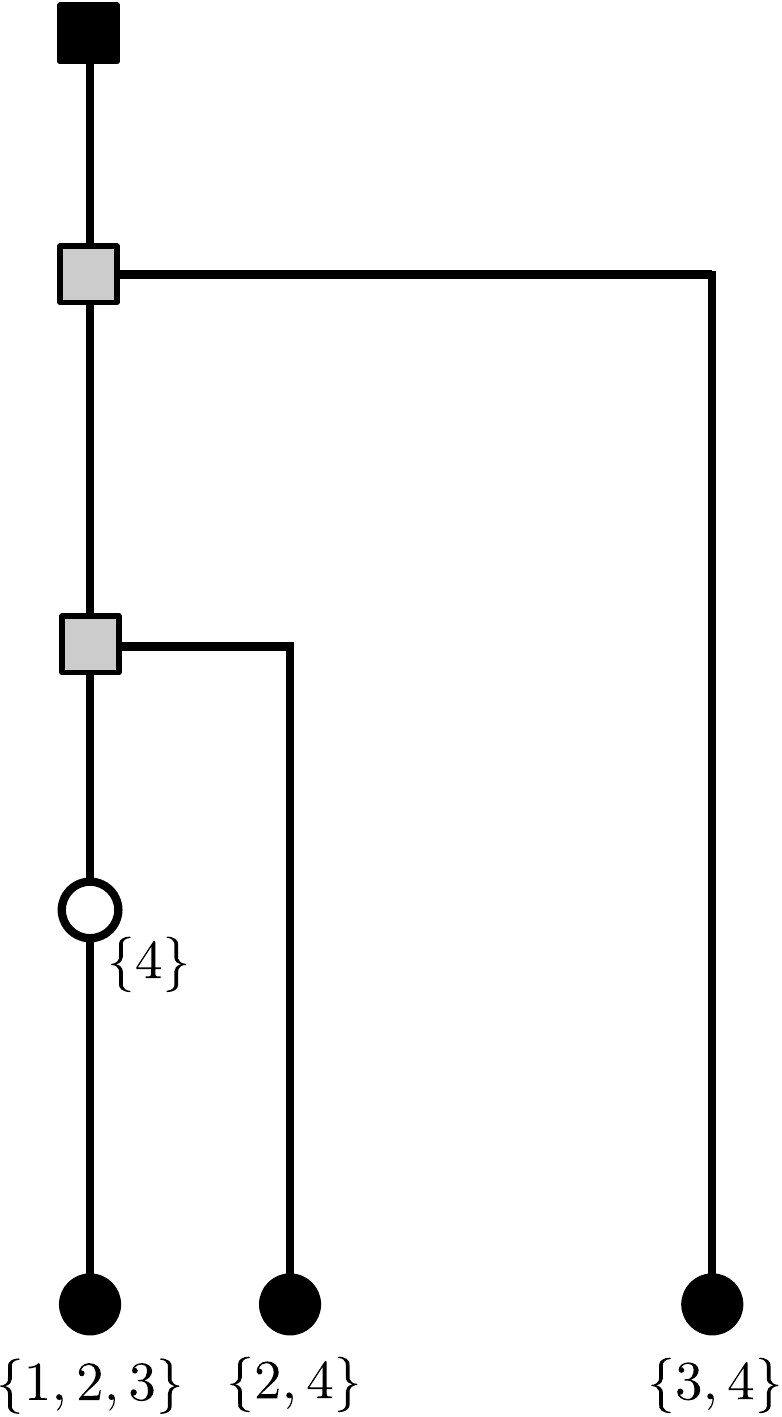}
\\
\quad
\\
\includegraphics[scale=.2]{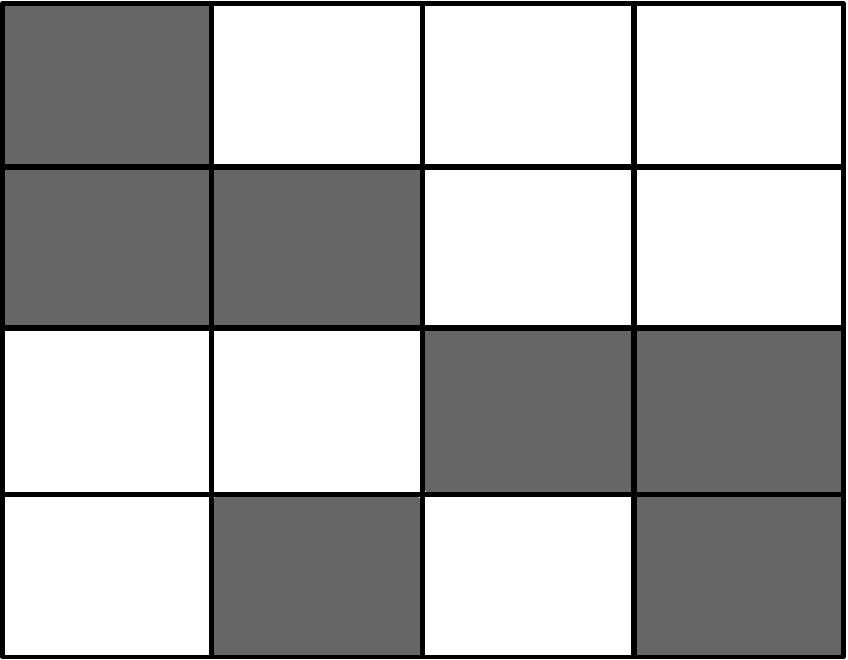}
&
\qquad
\includegraphics[scale=.2]{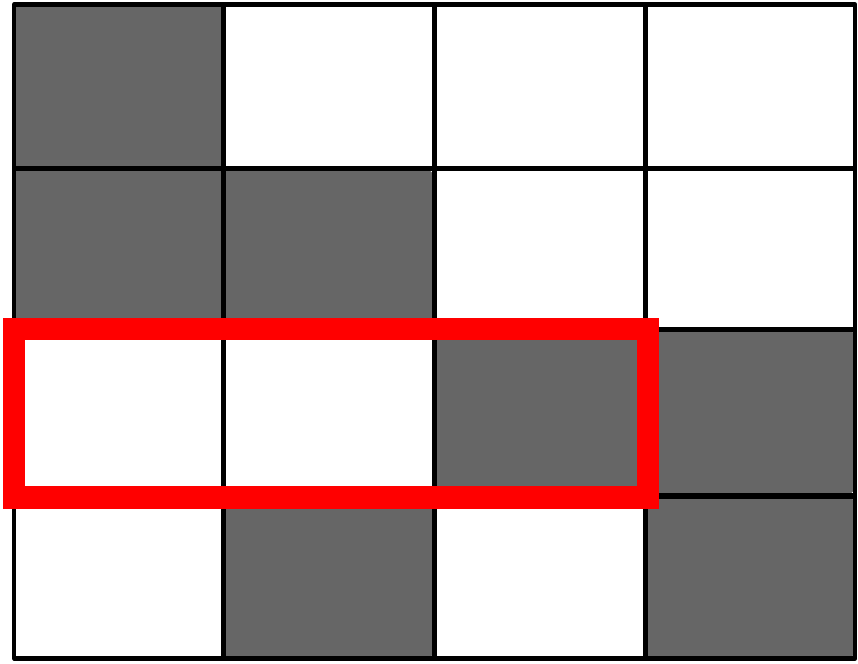}
&
\qquad
\includegraphics[scale=.2]{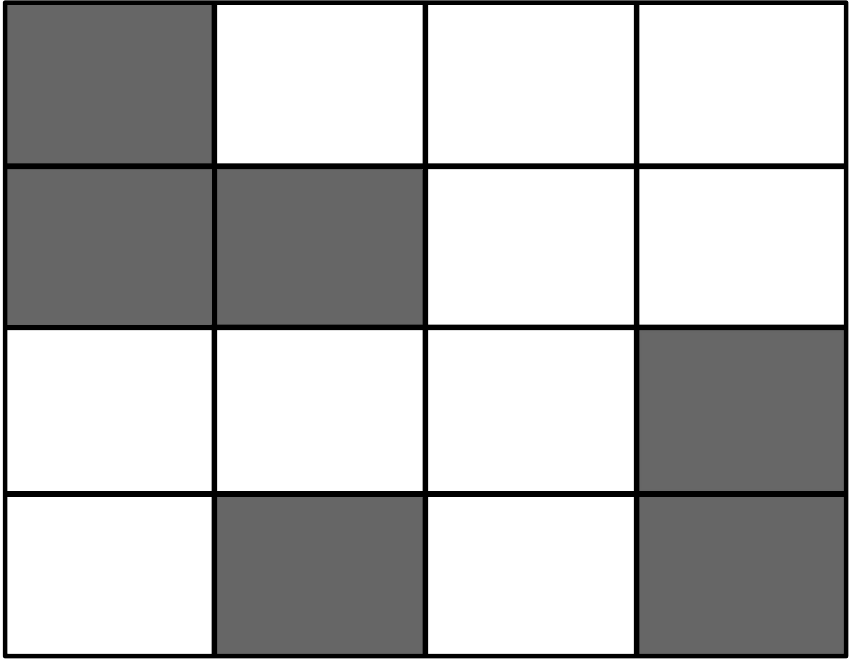}
&
\qquad
\includegraphics[scale=.2]{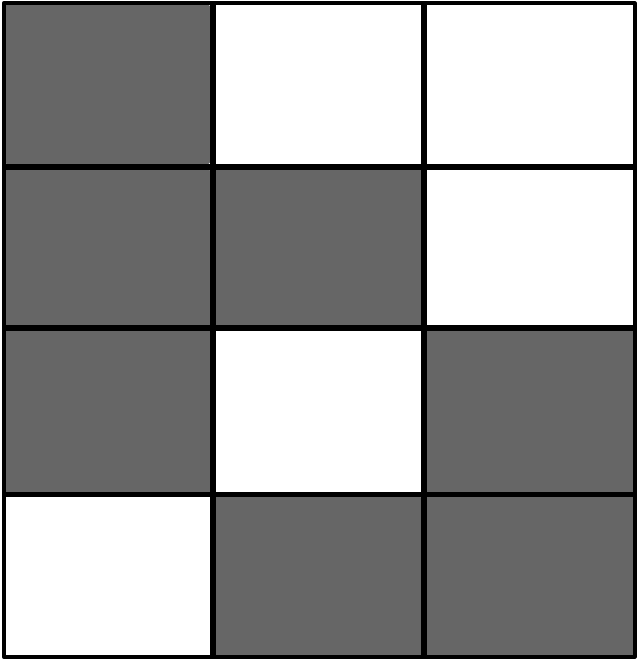}
\\
(a) & \qquad \quad (b) & \qquad \quad (c) & \qquad \quad (d)
\end{tabular}
\end{center}
\end{minipage}
\caption{An example of a proposal that resamples a subtree.  The original tree structure and corresponding feature matrix are shown in (a).  In (b), we select a subtree (outlined) with probability proportional to the number of particles traversing the subtree (four particles, in this example).  Uniformly at random, we select a particle down the branch (a 3-particle, in this example) that will be removed from the subtree.  The corresponding entries in the feature matrix that may be affected by this proposal are boxed in red.  In (c), we remove the 3-particle and resample its paths down the tree according to the prior, conditioned on the paths of all other particles.  In (d), the proposed tree and corresponding feature matrix (where an empty column has been removed) are shown.  The required terms for the acceptance ratio in \cref{eq:MH_accept} are $\sum_{u\in\Vcal(\Tcal_N)} m(u) = 25$ and $\sum_{u\in \Vcal(\Tcal_N^*)} = 19$.}
\label{fig:inference}
\end{figure*}

\subsection{Additional proposals}
\label{sec:advancedinference}

We make several additional proposals to improve mixing, which we summarize here.  Details are left to \cref{sec:furtherinference}.
\begin{itemize}
\item {\bf Multiple subtrees:}
We consider simultaneously resampling the paths of multiple particles down subtrees.
This is done by selecting the subtree as in the previous section, followed by selecting a subset of the particles down that subtree uniformly at random and resampling each of their paths down the subtree according to the prior.
The acceptance probability is again given by \cref{eq:MH_accept}.
As the number of resampled particles increases, so do the sizes of the proposed moves in the state space, and we therefore expect the acceptance probability to be low when many particles are resampled.
In our experiments, we limit the number of resampled particles to $\ceil{N/10}$.

\item {\bf Adding and removing replicate and stop nodes:}
A sampler based on only resampling subtrees will be inefficient at proposing the removal of replicate nodes, because such a proposal is only possible by first thinning the number of particles that take the divergent branch, and then proposing to resample a new subtree without a replicate point (via the moves in \cref{sec:resample_subtree}).
We therefore make proposals to remove replicate nodes from the tree structure, as well as proposals to add replicate nodes in order to leave the steady state distribution invariant. 
Similarly, we propose adding and removing stop nodes to branches in the tree.

\item {\bf Resample replicate and stop node configurations:}
We propose changing the decisions that particles take at replicate and stop nodes.
In particular, for a particle that replicates and takes the divergent branch at a replicate node, we propose removing the particle from the divergent subtree.
Alternatively, particles that do not replicate at branch points may be allowed to replicate and diffuse along the divergent branch (according to the prior).
Similarly, proposals to reverse the decisions at stop nodes are made.

\item {\bf Heuristics to prune and thicken branches:}
In order to specifically target unpopular replicate or stop nodes, we propose removing a replicate or stop node that is selected with probability inversely proportional to the number of particles that decided to replicate or stop at the node, respectively.
However, it is not clear how to propose adding unpopular nodes to the tree structure in an analogous manner, and because such reverse proposals are not made, these moves do not leave the stationary distribution of the Markov chain invariant.
Following an appropriate burn-in period, a practitioner may therefore remove these steps from the sampler in order to produce samples from the correct stationary distribution.
These moves are not included in the joint distribution tests in \cref{sec:Geweke}.

\end{itemize}

\subsection{Scheduling and computational complexity}
\label{sec:complexity}

In our experiments, we iterate between resampling the tree structure $\Tcal_N$ and the other latent variables and parameters in the model.
During each iteration, we make $2N$ proposals of resampling single subtrees, $N$ proposals resampling multiple subtrees simultaneously, $N$ proposals that change the decisions particles take at replicate and stop nodes, and $\mathcal I/4$ proposals to add and remove replicate and stop nodes, each, where $\mathcal I$ is the number of internal nodes in the tree structure at the beginning of the iteration.
Every five iterations, we replace the last proposal with those that specifically target unpopular replicate and stop nodes.
The mixing times did not appear to be sensitive to different settings of this schedule.

Every proposal requires one evaluation of the likelihood function in \cref{eq:likel}, the matrix inversion for which costs $O(\numfeat{N}^3)$, where $\numfeat{N}$ is the number of latent features.
In factor analysis applications, we typically expect $\numfeat{N} \ll N$.
We note that this is significantly cheaper than similar computations in density modeling applications where the leaves of a Dirichlet diffusion tree, for example, are associated with each data point, thereby costing $O(N^3)$ for the matrix inversion in a Gaussian density.
In order to overcome such difficulties, implementations of these models resort to message passing techniques that bring this computation down to $O(N)$ (e.g., see \citet{knowles2011message} and \citet{teh2007treecoalescent}).
These same techniques may be applied during inference with the beta diffusion tree.
Further suggestions to improve the efficiency of inference are discussed in \cref{sec:conclusion}.

\subsection{Sample hyperparameters}
\label{sec:hyperparameters}

Without good prior knowledge of what the tree structure hyperparameters $\soffset$, $\boffset$, $\sfn$, and $\bfn$ should be, we give them each a broad $\gammadist(\alpha,\beta)$ prior distribution and resample their values during inference.
We resample the concentration parameters $\soffset$ and $\boffset$ with slice sampling \citep{Neal2003}. 
Sampling the stop and replicate rate parameters $\sfn$ and $\bfn$ from their posterior distributions is straightforward by noting that
\[
\sfn \given \Tcal_N
	&\dist 
	\gammadist \Bigl (
		\alpha + \vert \mathcal S (\Tcal_N) \vert , \: \:
		\beta + \sum_{[ef] \in \mathcal B(\Tcal_N)} (t_f - t_e) H_{m(f)}^{\soffset} 
	\Bigr )
	,
\]
where as before $S (\Tcal_N)$ denotes the set of stop nodes and $H_n^\alpha = \psi(\alpha+n) - \psi(\alpha)$. 
A similar expression can be found for the posterior distribution of $\bfn$.
In all of our experiments, we set $\alpha = \beta = 1$.
For the factor analysis model, the noise hyperparameters $\sigma_X$ and $\sigma_Y$ are also resampled with slice sampling, for which we use the broad prior distributions $1/\sigma_Y^2 \dist \gammadist(1, 1)$ and $1/\sigma_X^2 \dist \gammadist(1,1)$.
%

\section{Numerical comparisons on test data}
\label{sec:experiments}

%
We implement the inference procedure presented in \cref{sec:inference} on the factor analysis model described in \cref{sec:LG}, which jointly models the binary matrix of factors $\bm Z$ and the matrix of factor loadings $\bm X$ with the beta diffusion tree.
We measure the quality of the model by evaluating the log-likelihood of the inferred model given a test set of held-out data (details below). %
We compare this performance against three baseline methods, each of which models the binary factors $\bm Z$ differently, however, in each case the number of latent features (and thus columns and rows of $\bm Z$ and $\bm X$, respectively) are unbounded and automatically inferred during inference.
The ``IBP'' method models $\bm Z$ with the two-parameter Indian buffet process \citep{GGS2007}.
The ``DP-IBP'' and ``IBP-IBP'' methods model $\bm Z$ with the two correlated latent feature models introduced by \citet{DG2009}, the first of which correlates features by partitioning them into blocks according to the Dirichlet process, and the second of which correlates features by allocating them to yet another underlying feature allocation.\footnote{For both correlated IBP methods, we use the two-parameter IBP.}
All three of these baseline methods model the factor loadings (independently from the factors) as mutually independent Gaussian vectors $\bm x_k \dist \Ncal ( \bm 0, \sigma_X^2 \bm I_D )$ for $k \le \numfeat{N}$, where $\numfeat{N}$ is the number of non-empty features and $N$ is the number of data points.
The factor loadings for all three models were analytically integrated out. 
The inference procedure described by \citet{meeds2007} was used to sample the IBP matrices, and the Metropolis--Hastings moves and partial Gibbs sampling steps for the Dirichlet process described by \citet{neal2000markov} were implemented for the DP-IBP method.
All hyperparameters were given broad prior distributions and sampled with slice sampling. 
\begin{figure}[t!]
\centering
\subfigure[{\it E. Coli} dataset]{
                \includegraphics[scale=.35]{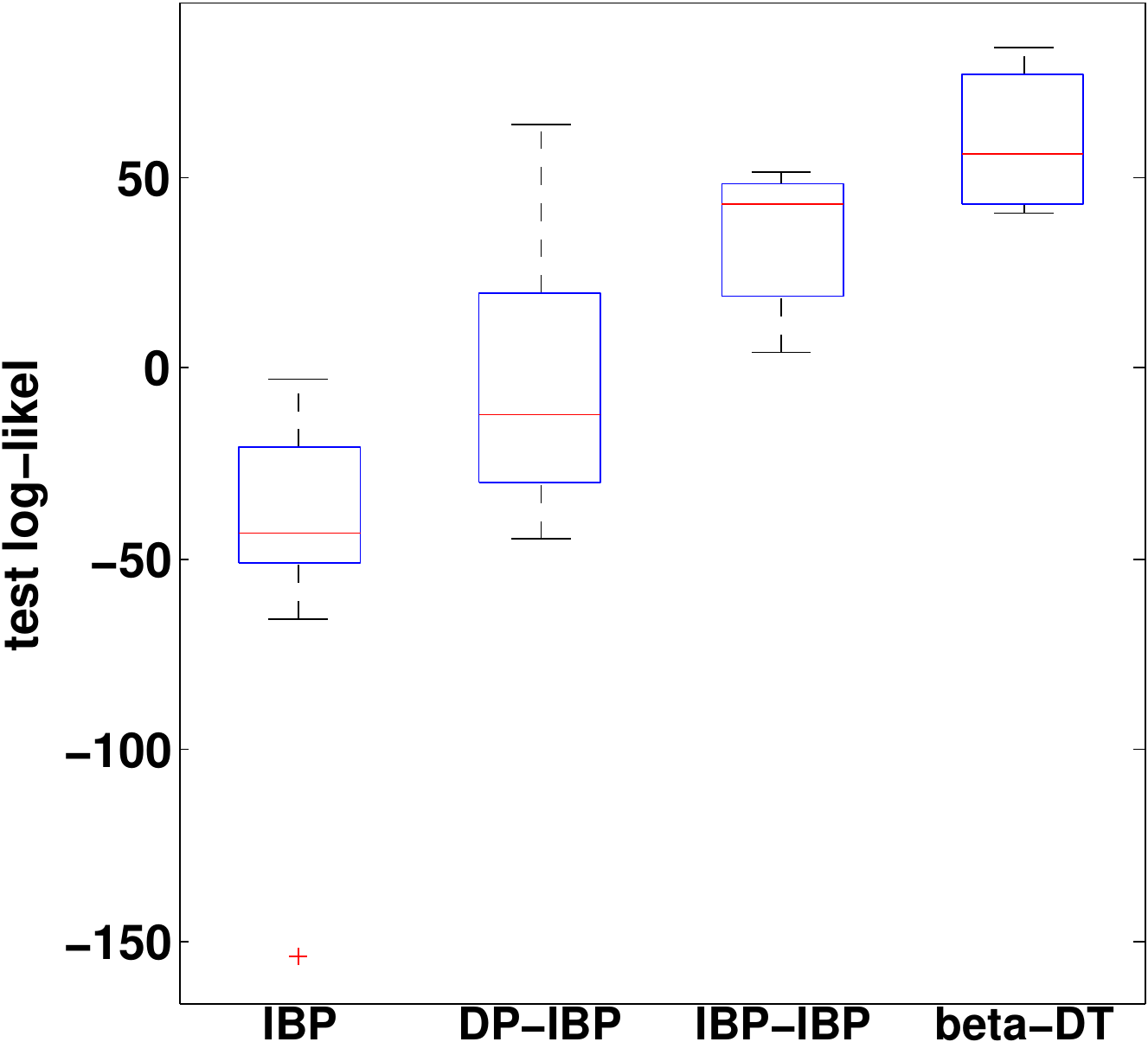}
                \label{fig:ecoli_LL}
                }
\subfigure[UN dataset]{
                \includegraphics[scale=.35]{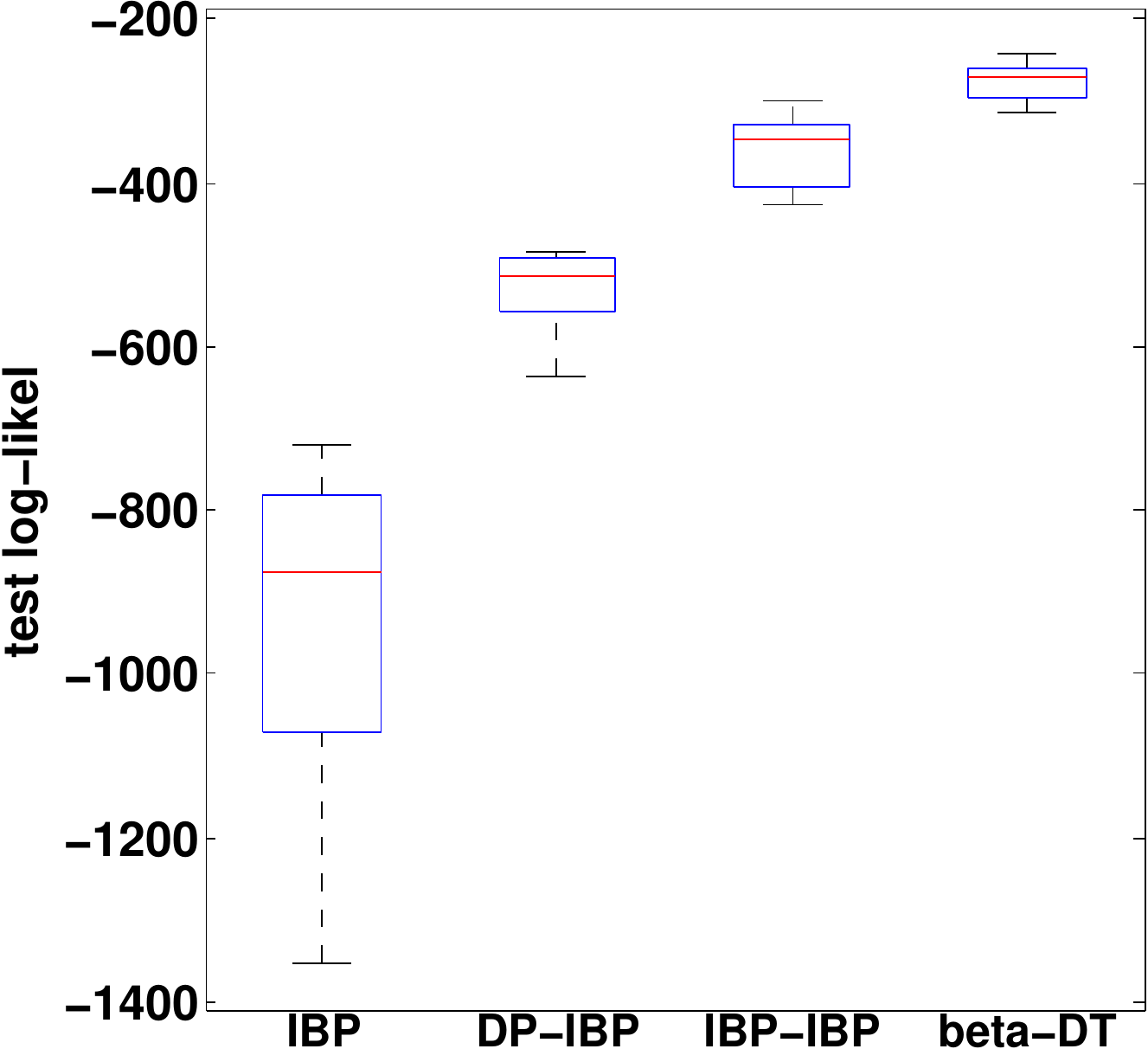}
                \label{fig:UN_LL}
        }
\subfigure[India dataset]{
                \includegraphics[scale=.35]{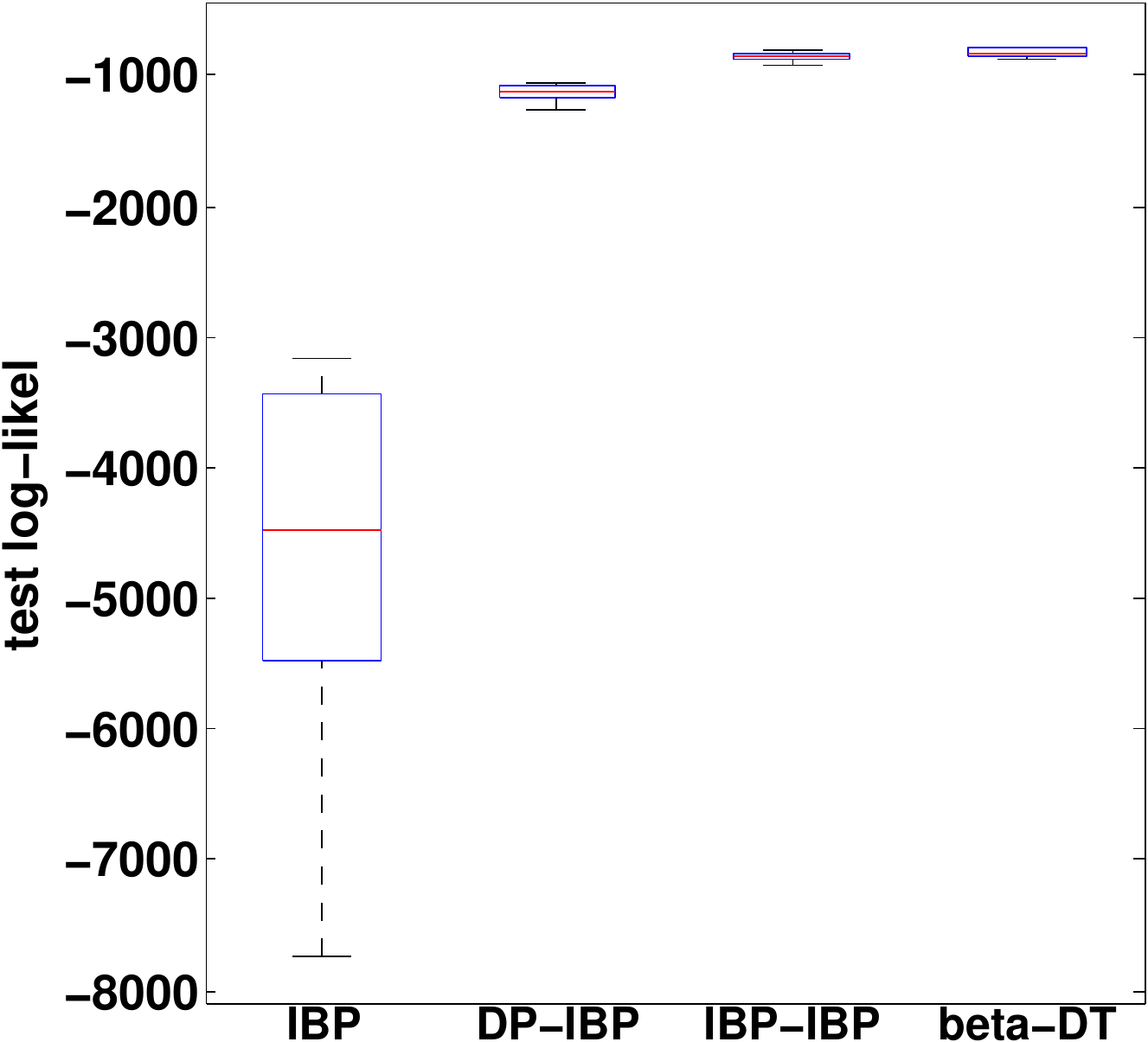}
                \label{fig:india_LL}
        }
       	\caption{Box plots of the test log-likelihoods for the four different models on three different data sets.  See the text for descriptions of the data sets and methods.  The beta diffusion tree achieves the best performance in each case.}
	\label{fig:exp_LL}
\end{figure}
\begin{figure}[!]
\centering
\subfigure[{\it E.~Coli} dataset]{
                \includegraphics[scale=.35]{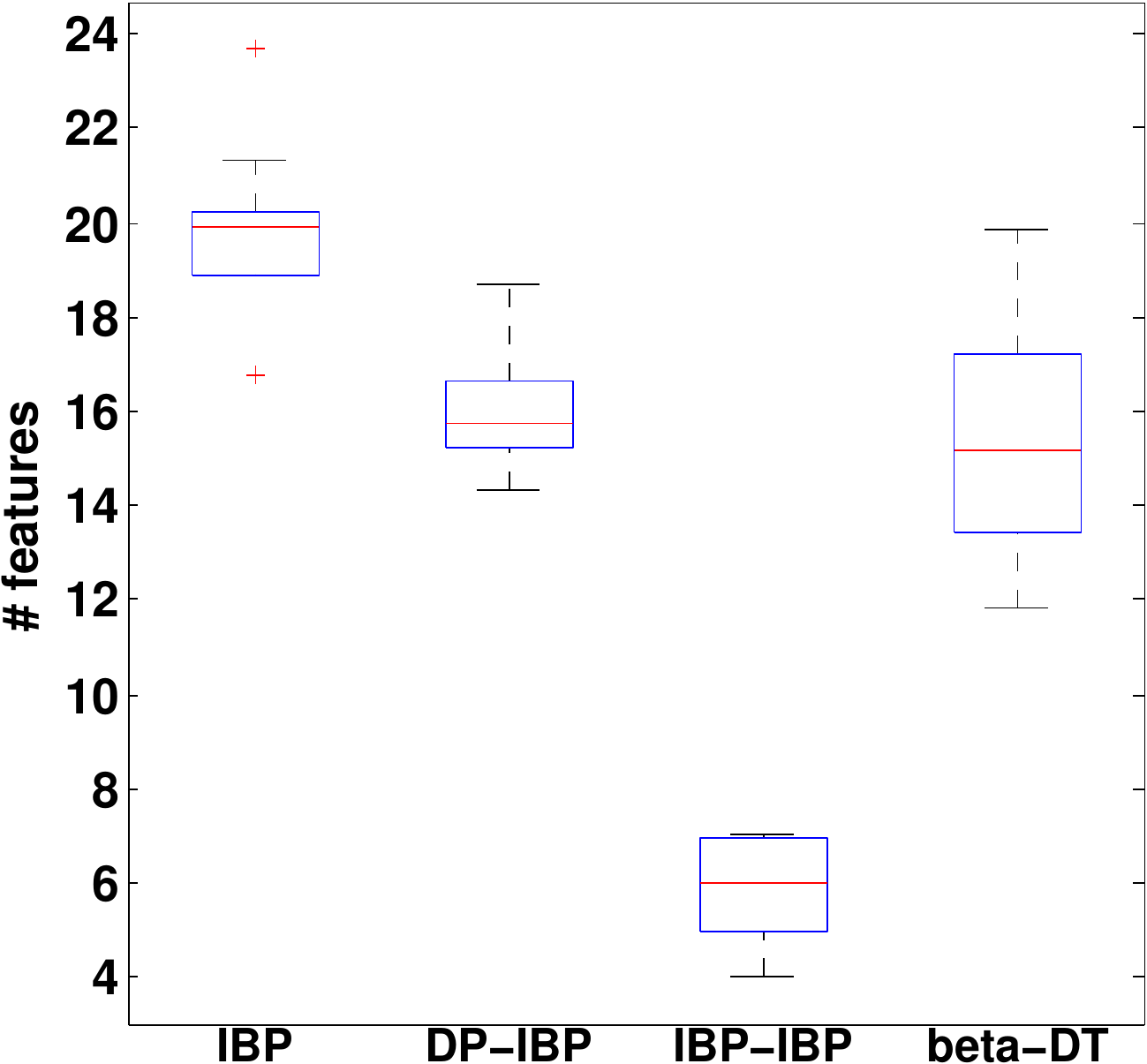}
                \label{fig:ecoli_K}
                }
\subfigure[UN dataset]{
                \includegraphics[scale=.35]{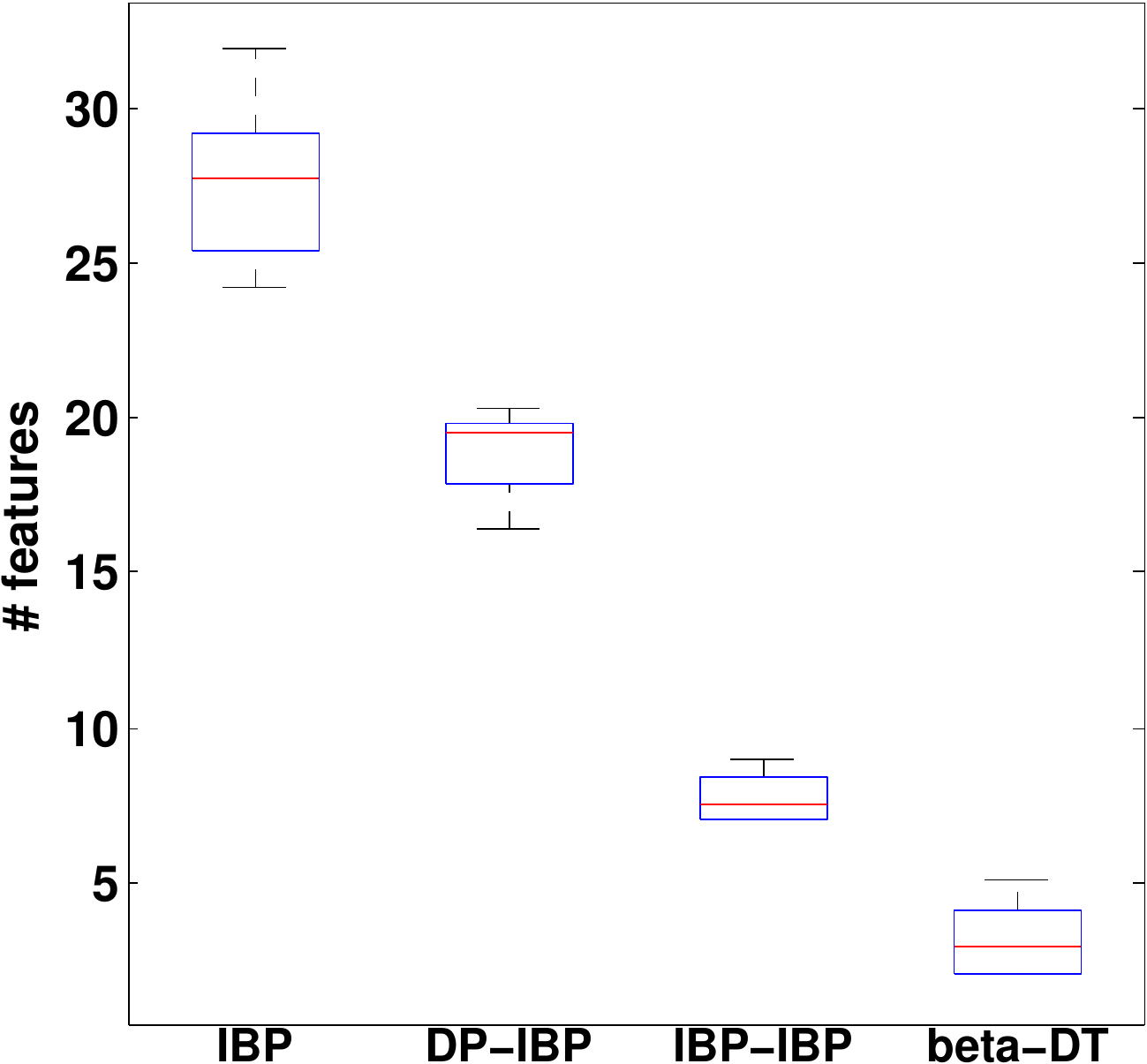}
                \label{fig:UN_K}
        }
\subfigure[India dataset]{
                \includegraphics[scale=.35]{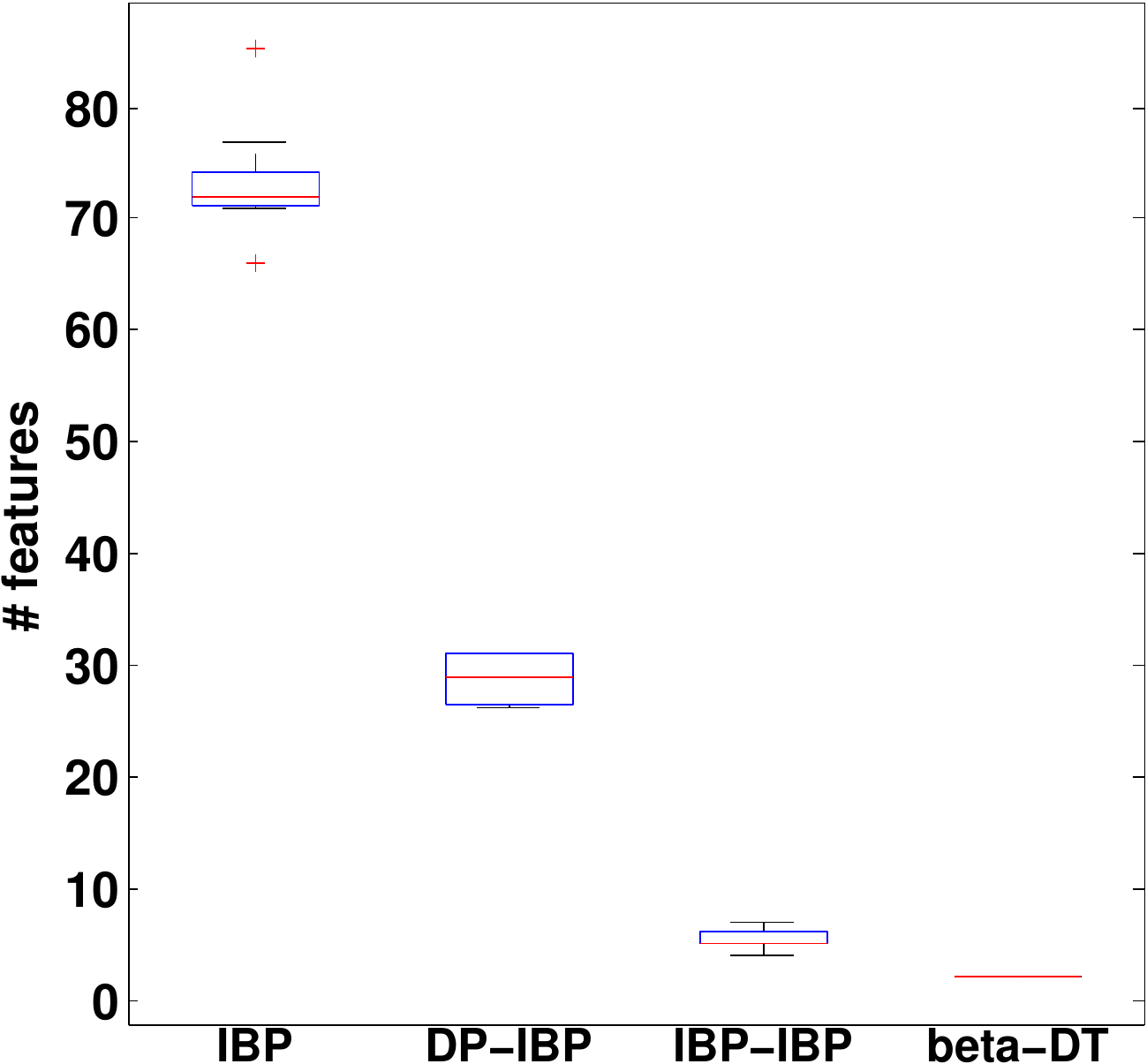}
                \label{fig:india_K}
        }
       	\caption{Box plots of the average number of inferred features in the numerical experiments.}
	\label{fig:exp_K}
\end{figure}

We studied three different data sets.
For the sake of comparison, we created two data sets, labeled ``UN'' and ``India'' below, to be comparable to the identically named data sets studied by \citet{DG2009}.  These were selected to be of approximately the same sizes as each of their counterparts and from similar databases.

\vspace{.5em}
\noindent
{\bf \emph{E.~Coli}:} Expression levels of $N=100$ genes measured at $D=24$ time points in an \emph{E.~Coli} culture obtained from the study in \citet{kao2004transcriptome}.

\vspace{.5em}
\noindent
{\bf UN:} Statistics from the United Nations for $N=161$ countries on $D=15$ human development measurements \citep{UNdevel2013}.
The variables correspond to measurements such as GDP, education levels, and life expectancies.
The 161 countries selected correspond to countries with no missing data among the 15 collected variables.
Because of this sampling bias, the selected countries are largely those with (relatively) high \emph{human development indices}, as measured by the UN.

\vspace{.5em}
\noindent
{\bf India:} Statistics for $N=400$ Indian households on $D=15$ socioeconomic measurements \citep{indiaICPSR}.
The data were collected across all Indian states and both urban and rural areas.
The variables correspond to measurements such as value of household assets, number of household members, consumption on food, etc.
There are no missing data entries among the selected households.

\vspace{.5em}
\noindent
For each data set, we created 10 different test sets, each one holding out a different 10\% of the data.
In \cref{fig:exp_LL}, we display the box-plots of the test log-likelihood scores over the 10 test sets, where the score for a single set is averaged over 3,000 samples (of the latent variables and parameters of the model) collected following the burn-in period of each method.
The beta diffusion tree achieved the highest median score in every experiment, with the IBP-IBP achieving the second best performance in each instance.
The difference between these two sets of scores is statistically significant in each case, based on a t-test at a 0.05 significance level.  The p-values for the null hypothesis that the means of the two sets are the same were $5.30\times 10^{-03}$, $1.50\times 10^{-04}$, and $7.85\times 10^{-03}$ for the {\it E. Coli}, UN, and India data sets, respectively.
In \cref{fig:exp_K}, we display box plots of the number of features inferred for each test set (averaged over the 3,000 samples following the burn-in).
On the UN and India data sets, the number of features decreases as the chosen model becomes more expressive.
Most notably, the beta diffusion tree model robustly infers only two features on every test set of the India data set.
This behavior is expected if hierarchical structure is truly present in the data sets, because the baseline methods, without the ability to model such hierarchies, will need to instead suggest additional features in order to account for the complexity of the data.
The superior performance on the test log-likelihood metric therefore suggests that such hierarchical latent structure is an appropriate model for these data sets.
On the {\it E.~Coli} data set, the IBP-IBP method learns a much lower number of features than the beta diffusion tree, however, the higher test log-likelihood scores for the beta diffusion tree suggest that a hierarchical feature allocation still provides a better model for the data.

It is also interesting to study the inferred latent variables directly.
For example, we can extend the analysis by \citet{DG2009} on the UN development statistics.  As opposed to the methods that model a flat clustering of the features, such as the DP-IBP and IBP-IBP, the beta diffusion tree clearly captures a hierarchical structure underlying the data, as visible in \cref{fig:UN_demo}.
Here we display the maximum \emph{a posteriori} probability sample (among 2,000 samples collected after a burn-in period on the data set with no missing entries) of the feature matrix.
The inferred tree structure for this sample is also displayed over the features.
For visualization, the rows (corresponding to different countries) are sorted from highest human development index (HDI -- a score computed by the UN based on many human development statistics) to lowest.
For reference, we also display the HDI scores for five ranges of equal sizes, along with the names of the top and bottom 10 countries in each range.  
We can see that a hierarchical structure is present; on the one hand, many highly developed countries are assigned to the third feature, with a more refined set belonging to the fourth feature.
An even finer subset belongs to the fifth.
\clearpage
\begin{sidewaysfigure}
\centering
      \includegraphics[scale=.85]{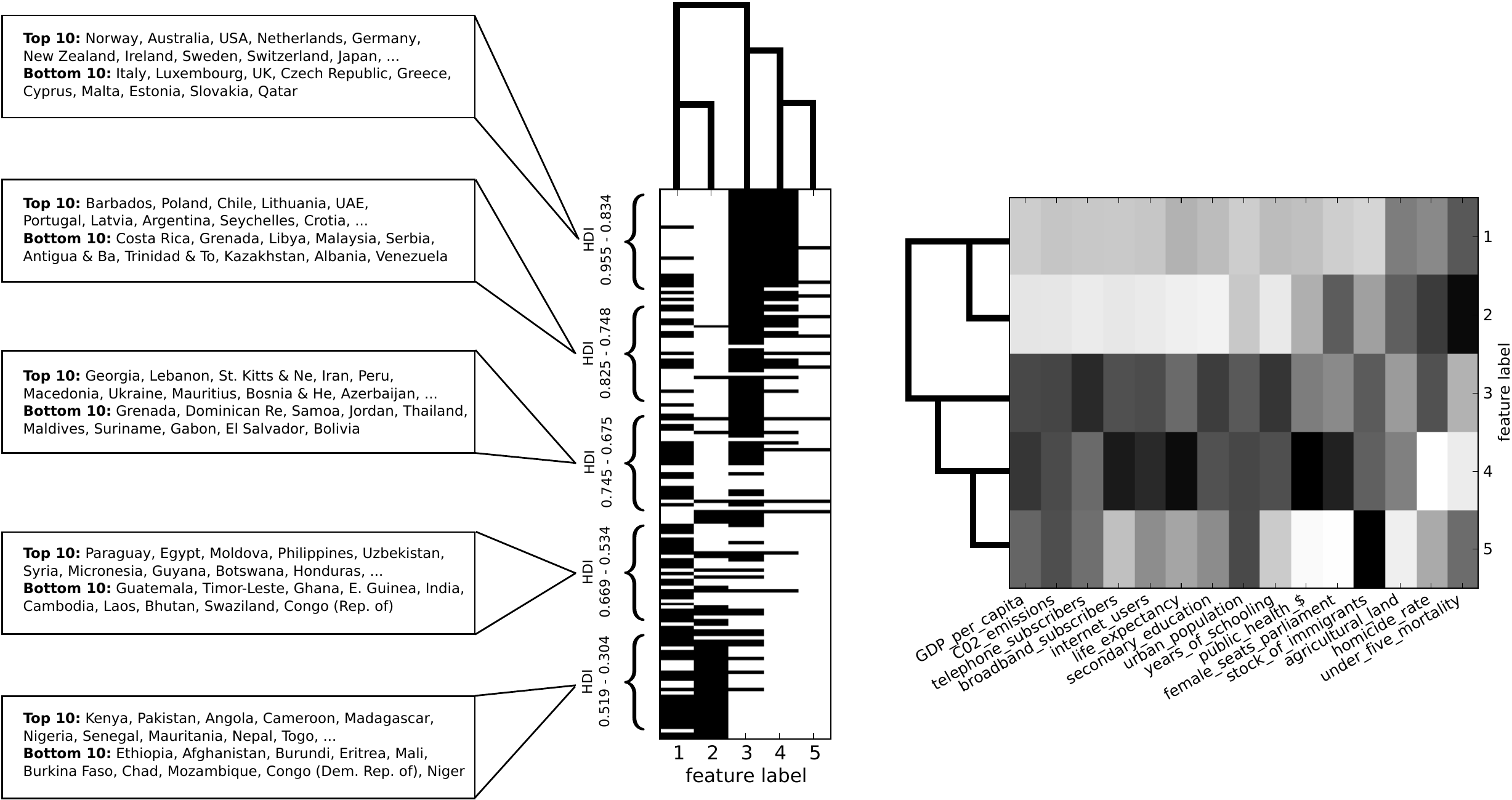}
	\caption{Left: Inferred factor matrix and corresponding hierarchy over the features for the UN data set.  Black entries correspond to a one, and white entries correspond to a zero.  The rows correspond to countries, which are ranked by their Human Development Indices (HDI).  The names of the top and bottom 10 countries in five different ranges of the HDIs are displayed.  Right: The posterior mean of the factor loading matrix, along with the hierarchy over the features.  The rows correspond to the latent features and the columns correspond to the different observed $D=15$ variables.  Darker colors correspond to larger values.}
	\label{fig:UN_demo}
\end{sidewaysfigure}
\clearpage
\noindent
On the other hand, the less developed countries have high prevalence in the second feature, with a broader set belonging to the first.
This subset is not strict; interestingly, many countries belonging to the second feature do not belong to the first.
We have also displayed the posterior mean of the factor loading matrix (the expression for which can be obtained from \cref{eq:posteriorloadings}).
The third feature places higher weight on the variables we expect to be positively correlated with the highly developed countries, for example, GDP per capita, the number of broadband subscribers, and life expectancy.
These features place lower weight on the variables we expect to be negatively correlated with the highly developed countries, notably, the rates for homicide and infant mortality.
The first and second features are the reverse: they place higher weight on variables such as homicide and infant mortality rates, and lower weight on the variables such as GDP and life expectancy.
Each of these pairs of factor loadings reflect the same hierarchical relationship seen in their corresponding factors, and this interpretation of both the factors and factor loadings are consistent with the inferred tree structure (displayed) over the feature labels.
Similarly, in \cref{fig:ecoli_demo} we display the maximum {\it a posteriori} probability feature matrix and corresponding hierarchy over the features for the {\it E.~Coli} data set when no data is held out.  
We note that, in this figure, the features are not necessarily ordered with the divergent branches to the right like in the previous figures in the document.
In this case, the individual genes are not as interpretable as the countries in the UN data set, however, the hierarchical structure is reflected in the feature allocation matrix.
\begin{figure}[!]
	\centering
      \includegraphics[scale=.4]{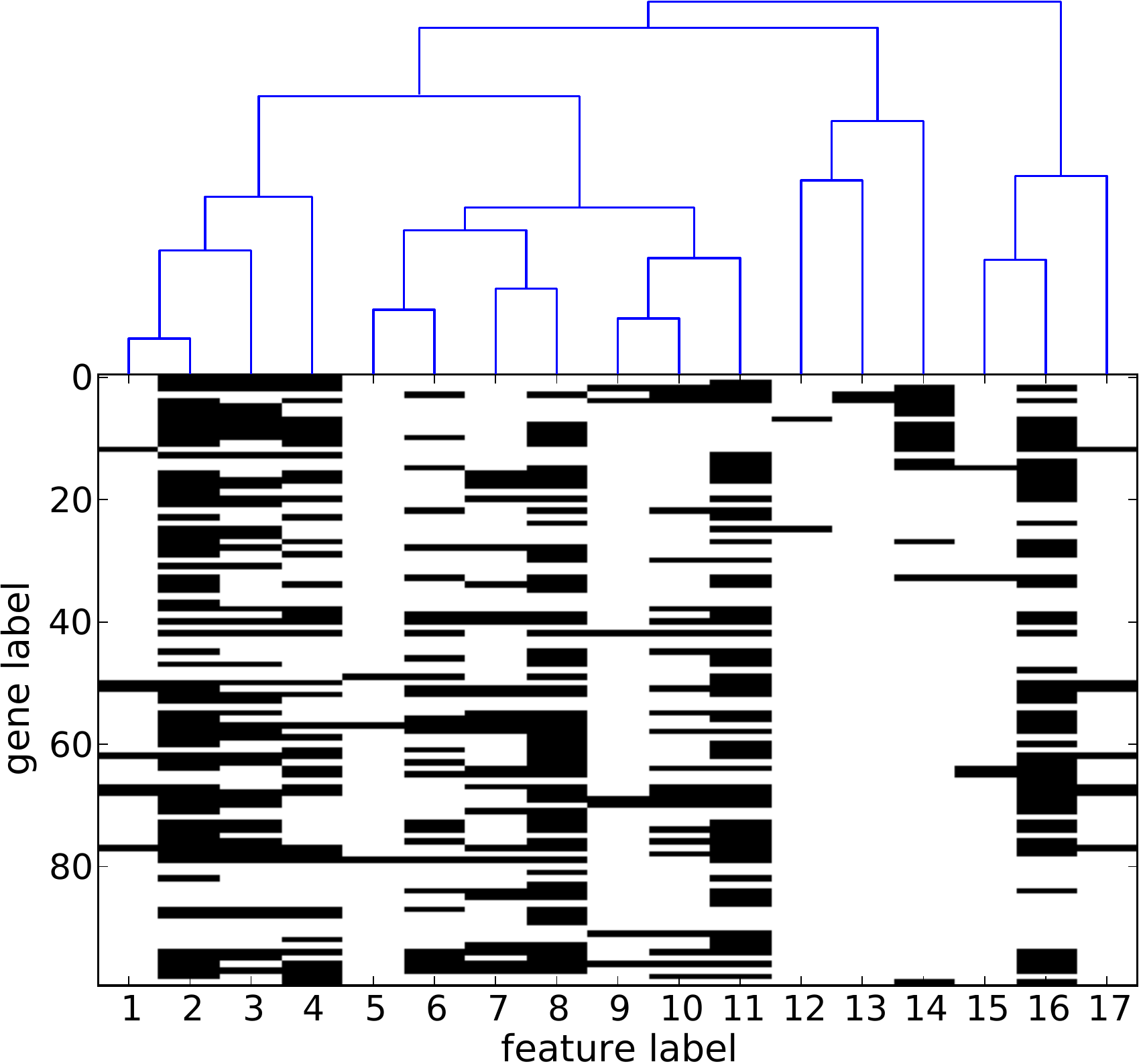}
	\caption{The inferred factor matrix and hierarchy over the features for the \emph{E.~Coli} data set.  The features are not necessarily ordered with the divergent branches to the right, as in previous figures in the document.}
	\label{fig:ecoli_demo}
\end{figure}
%
%

\section{Conclusion and future work}
\label{sec:conclusion}

The beta diffusion tree is an expressive new model class of tree-structured feature allocations of $\Nats$, where the number of features are unbounded.
We showed that the beta diffusion tree is exchangeable with respect to the ordering of $\Nats$, and by characterizing it as the infinite limit of a nested feature allocation scheme, we revealed useful insights into the properties of the model.
The superior performance of this model class in our numerical experiments, compared to independent or flatly-clustered features, provides evidence that hierarchically-structured feature allocations are appropriate for a wide range of statistical applications and that the beta diffusion tree can successfully capture this latent structure.

There are many future directions to be explored, a few of which we now briefly discuss.
As noted in \cref{sec:related_work}, the Pitman--Yor diffusion tree \citep{KnowlesGhahramani2011,knowlesTPAMIPYDT,KnowlesThesis2012} extended the Dirichlet diffusion tree beyond binary branching to an arbitrary branching degree.
In particular, when a particle in the diffusion process reaches a previous divergence point (where at least one previous particle diverged from the current path), the particle may decide to go down an existing branch with probability depending on the previous number of particles down the branch, or it may choose to diverge to form a new path.
Connections with this process were made to the \emph{Pitman--Yor} (aka \emph{two-parameter Poisson--Dirichlet}) process \citep{teh2006hierarchical,perman1992size}.
Similarly, the \emph{$\Lambda$-coalescent} \citep{pitman1999coalescents,hu2013binary} generalizes the tree structure of Kingman's coalescent \citep{kingman1982coalescent} to arbitrary branching degrees (as discussed in \cref{sec:related_work}).
Trees with flexible branching degrees may provide more parsimonious explanations for data, and such a modeling assumption is appropriate in many applications.
It would therefore be interesting to see an extension of the beta diffusion tree to arbitrary branching degrees.

As seen in \cref{sec:continuum_limit}, the features in the beta diffusion tree are not exchangeable (at a replicate point, particles are guaranteed to follow an original branch but not necessarily a divergent branch).
The non-exchangeability of the branches is reflected by the fact that the allocation probabilities $p_0^{(1)}$ and $p_0^{(2)}$ (given by \cref{eq:level_probs}) in the nested feature allocation scheme are not exchangeable.
In contrast, the exchangeability of the countable collection of feature allocation probabilities obtained from the beta process \citep{hjort1990nonparametric} implies the exchangeability of the features in the allocations obtained with the Indian buffet process \citep{ThibauxJordan2007}.
One could investigate if there is a variant or generalization of the beta diffusion tree in which the features are exchangeable.
This could be a desirable modeling assumption in some applications and may enable the development of new inference procedures.

In some applications, it may be appropriate to assume that if an object belongs to one feature (or some subset of features), then it \emph{must} belong to another ``super feature''.
For example, consider a ``market basket'' application where objects correspond to sets of purchases by customers, and we want to model a set of features corresponding to unobserved classifications of the purchased goods, such as ``household goods'', ``toiletries'', and ``cleaning supplies''.
If a basket contains items belonging to the category ``toiletries'', then it contains the category ``household goods'', and both of these latent categories could have different effects on the observed data.  A different basket could contain items belonging to the category ``cleaning supplies'' and so contains the category ``household goods'', sharing the effects of this latter category.
This ``strict subsets'' structure can be enforced with the beta diffusion tree by setting the stop rate parameter to be zero $\sfn=0$, and it would be interesting to see the beta diffusion tree applied to such problems.

As discussed in \cref{sec:related_work}, the beta diffusion tree (as well as the Dirichlet and Pitman--Yor diffusion trees) may be viewed as fragmentation processes, which are closely related to a class of stochastic processes called coagulation processes \citep{bertoin2006random}, in the sense that some fragmentation processes are equivalent to a coagulation process upon time reversal, and vice versa.
This so called \emph{duality} characterizes the field of research on fragmentation-coagulation processes in the physics, probability, and statistics communities, and their interpretation as random tree structures are well-studied \citep{pitman2006combinatorial}.
One such process with this property is the Dirichlet diffusion tree, which upon time reversal obtains Kingman's coalescent \citep{teh2011modelling,bertoin2006random}.
The stochastic process introduced by \citet{teh2011modelling} can be viewed as combining these two dual processes in order to model an evolving partition of objects, in which each block of the partition may fragment into two blocks or coagulate with another block over time.
A discrete-time analogue of the process was studied by \citet{elliott2012scalable}.
It is not yet clear that any such coagulation process would result from a time reversal of the beta diffusion tree.
However, such a dual process may exist for some variant or generalization of the beta diffusion tree, and if so, one could obtain a model for time evolving feature allocations from a combination of the two processes.

Finally, new inference procedures should be explored.  Integrating over both the tree structure and the paths of particles down the tree (resulting in the feature allocations) is a challenging inference problem, as both objects live in large, combinatorial spaces.
Sampling moves that avoid a Metropolis--Hastings step and therefore change the states of the latent variables in the model with each proposal would be of immediate benefit to the mixing times of the inference procedure.
For example, a variant of the slice sampling techniques employed by \citet{NealDDT2003} for the Dirichlet diffusion tree and by \citet{knowlesTPAMIPYDT} for the Pitman--Yor diffusion trees may be developed.
These procedures could be combined with belief propagation and sequential Monte Carlo techniques to further improve inference \citep{knowlesTPAMIPYDT, hu2013binary, knowles2011message, teh2007treecoalescent}.

%
%
\acks{We thank anonymous reviewers for feedback.
We also thank Hong Ge, Daniel M. Roy, Christian Steinrueken, and anonymous reviewers for feedback on drafts of a shorter conference version of this paper.
This work was carried out while C. Heaukulani was funded by the Stephen Thomas studentship at Queens' College, Cambridge, with co-funding from the Cambridge Trusts.  Z. Ghahramani is funded by EPSRC grant EP/I036575/1.}
%
%

\appendix

\section{Simulating a beta diffusion tree in practice}
\label{sec:prior_sampler}

Consider simulating a beta diffusion tree with $N$ objects according to the generative process in \cref{sec:BDTmain}.
We first simulate the tree structure $\Tcal_N$, followed by the locations of the nodes in $\mathcal X$, given $\Tcal_N$.
The first object starts at time $t=0$ with one particle at the origin.  We simulate a replicate and a stop time, $t_r$ and $t_s$, respectively, from exponential distributions with rates $\bfn$ and $\sfn$, respectively.
If both times occur after $t=1$, then we create a leaf node with the particle at $t=1$.  If at least one time occurs before $t=1$, then create an internal node of the tree as follows:
\begin{itemize}
\item[(a)] If $t_s$ occurs before $t_r$, then ignore $t_r$ and the particle stops at $t=t_s$ and we create a stop node at $t=t_s$.

\item[(b)] If $t_r$ occurs before $t_s$, then ignore $t_s$ and the particle replicates at $t=t_r$ and we create a replicate node at $t=t_r$.
\end{itemize}
We repeat this procedure for each replicate that is made.  By the memoryless property of the exponential distribution, the waiting time for a particle starting at time $t' > 0$ until either replicating or stopping is still exponentially distributed with rates $\bfn$ and $\sfn$, respectively.

For each subsequent particle $n \ge 2$, we again start with one particle at the origin, which follows the path/branch of the previous particles.
Simulate a replicate and stop time, $t_r$ and $t_s$, respectively, from exponential distributions with rates $\boffset/(\boffset + n -1)$ and $\soffset/(\soffset + n - 1)$, respectively.
If either time occurs before the next node on this branch, then create an internal node on the branch according to the rules (a) and (b) above.
If both occur after the next node on the branch, then the particle has reached that node.
If the node is a replicate node, then it decides to replicate and send a copy of itself down the divergent path with probability given by \cref{eq:branchprob}.  If the node is a stop node, then it stops at the node with probability given by \cref{eq:stopprob}.
Finally, we repeat this procedure for every replicate travelling along a branch, where its replicate and stop times are exponentially distributed with rates $\boffset/(\boffset + m)$ and $\soffset/(\soffset+m)$, respectively, and $m$ is the number of particles that have previously traversed the branch.

The tree structure $\Tcal_N$ has now been simulated.
The root node is deterministically set at the origin, and we can simulate the locations in $\mathcal X$ conditioned on the tree structure.
For example, if we let $\mathcal X = \Reals^D$ and let the diffusion paths be Brownian motions with variance $\sigma_X^2$, then we make a sweep down the tree structure simulating
\[
\bm x_f \given \bm x_e 
	\dist \mathcal{N} ( \bm x_e, \sigma_X^2 (t_f - t_e) \bm I_D )
	,
\]
for each branch $[ef]$ in the set of branches $\mathcal B(\Tcal_N)$.

This procedure can be generalized to the case when the replicate and stop rates vary with time by following the methods described by \citet{KnowlesThesis2012} and \citet{knowlesTPAMIPYDT}.
We note that in this case the times until replicating and stopping on a branch are the waiting times until the first jumps in two (dependent) inhomogeneous Poisson processes with rate functions
\[
\frac{\boffset}{\boffset + m} \bfn(t)
	\quad \text{ and } \quad
	\frac{\soffset}{\soffset+m} \sfn(t)
	,
\]
respectively, where $\bfn(t)$ and $\sfn(t)$ are positive-valued functions, and $m$ is the number of particles that have previously traversed the branch.

\section{Multitype continuous-time Markov branching processes}
\label{sec:galtonwatson}

We show that the beta diffusion tree is characterized by a well-studied type of stochastic process called a \emph{branching process} in order to derive the results in \cref{sec:properties}.
The theory of branching processes has a long history in probability and statistics, and the reader may refer to \citet{harris2002theory} and \citet{athreya1972branching} for background.
For a thorough treatment of the multitype case, the reader may refer to \citet{mode1971}.

We return to the nested feature allocation scheme with $N$ objects and $L$ levels from \cref{sec:continuum_limit}.
An example of the tree structure corresponding to a nested scheme is depicted in \cref{fig:finite_tree_nodes}, where each vertex represents a non-empty node.
Nodes with two children are enlarged, which we call replicate nodes, and the nodes shaded in grey are those that do not assign at least one of their elements to the ``first node'' at the subsequent level, which we call stop nodes.
\begin{figure}
\centering
\includegraphics[scale=.5]{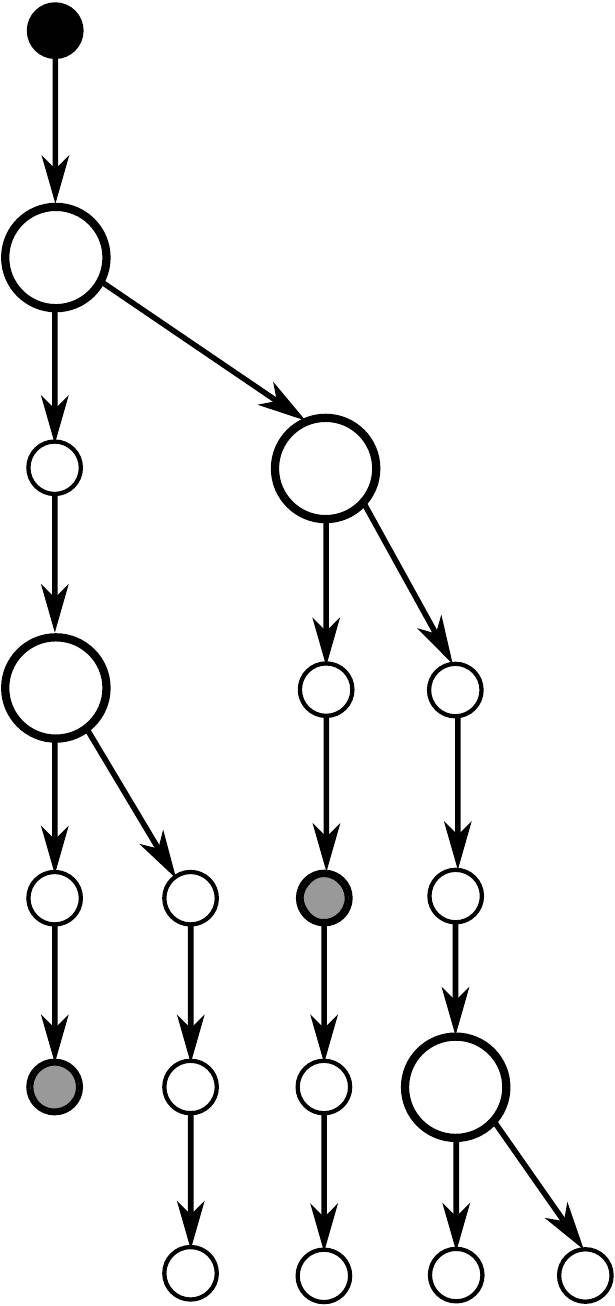}
\caption{Tree structure corresponding to a nested feature allocation scheme, where each vertex corresponds to a non-empty node.  Nodes with two children are enlarged, and nodes that do not allocate all of their elements to the ``first node'' at the subsequent level are shaded grey.}
\label{fig:finite_tree_nodes}
\end{figure}
As we have seen in \cref{sec:continuum_limit}, the probabilities of an element being assigned to the first and second nodes approach one and zero, respectively, as $L\rightarrow \infty$, and producing either a stop or replicate node becomes a rare event.
By \cref{thm:continuum_limit}, this continuous-time process is equivalent to the tree structure of a beta diffusion tree, where the branches correspond to chains of identical (non-stop, non-replicate) nodes, the replicate nodes create new chains, and the stop nodes change the number of particles on a branch, in some cases terminating the branch.  
It is therefore intuitive to think of the chains of identical nodes as individuals in a population.  Each individual is classified as one of $N$ \emph{types} determined by the number of elements allocated to the node represented by that vertex.  The individual gives birth to new individuals at replicate nodes and either changes type or dies at stop nodes.

\subsection{A multitype birth-death process}

We derive the waiting times until a replicate or stop node is produced on a chain of identical nodes (i.e., until the individual gives birth, changes type or dies) in the limit $L\rightarrow\infty$.
Let $(a)_n \defas a (a+1) \cdots (a+n-1) = \Gamma(a+n)/\Gamma(a)$ denote the Pochhammer symbol.
The probability that a node $f_{\bullet}$ with $n$ elements does not allocate every element to the first node (i.e., produces a stop node) at the subsequent level is given by
\[
&1 - \int (p_{f_\bullet}^{(1)})^n \, \betadist \Bigl ( p_{f_\bullet}^{(1)} ;
				\soffset \Bigl ( 1 - \frac{\sfn}{L} \Bigr ) , \soffset \frac{\sfn}{L}
				\Bigr )
			\dee p_{f_\bullet}^{(1)}
			\\
		&\qquad \qquad =
		1 - \frac{ (\soffset (1-\sfn/L))_n }{ (\soffset)_n }
			\\
		&\qquad \qquad =
		\sfn \soffset H_n^{\soffset} / L + O(L^{-2})
 		,
\]
where the final equation is obtained by noting that, for any $a, b>0$,
\[
(b-a^{-1})_n
	= (b)_n ( 1 - a^{-1} H_n^b ) + O(a^{-2})
	,
\]
and we recall from the derivations in \cref{sec:exchangeability} that $H_n^\alpha \defas \sum_{i=0}^{n-1} (\alpha+i)^{-1} = \psi(\alpha+n)-\psi(\alpha)$.
Starting at any level on a chain of identical (non-stop, non-replicate) nodes, it follows that the number of levels $m$ until producing a stop node is distributed as
\[
\frac{\sfn \soffset H_n^{\soffset}}{L}
	 \Bigl (
			1 - \frac{\sfn \soffset H_n^{\soffset}}{L}
		\Bigr )^{m-1}
		,
\]
where we have ignored the terms in $O(L^{-2})$.
This is a geometric distribution in $m$.
Following the proof of \cref{thm:continuum_limit}, the waiting time until producing a stop node is $m/L$, the distribution of which converges in the limit $L\rightarrow \infty$ to an exponential distribution with rate
\[
\mu_n \defas \sfn \soffset H_n^{\soffset}
	= \sfn \soffset [ \psi(\soffset + n) - \psi(\soffset) ]
	.
	\label{eq:stoprate}
\]
This is the waiting time until a type $n$ individual has a child of a different type (or in some cases dies).
Similarly, the probability of \emph{at least one} element being assigned to the second node at the subsequent level is
\[
&1 - \int (1-p_{f_\bullet}^{(2)})^n
	\betadist \Bigl (\, \boffset \frac{\bfn}{L} ,\, \boffset \Bigl ( 1 - \frac{\bfn}{L} \Bigr )\, \Big ) \dee p_{f_\bullet}^{(2)}
	\\
	&\qquad \qquad =
	1 - \frac{ ( \boffset ( 1 - \bfn/L) )_n }{ (\boffset)_n }
	\\
	&\qquad \qquad =
		\bfn \boffset H_n^{\boffset} / L
		+ O( L^{-2} )
	,
	\label{eq:P2_first}
\]
so that in the limit $L\rightarrow\infty$, the waiting time until producing a replicate node (i.e., until a type $n$ individual has two children) is exponentially distributed with rate
\[
\beta_n \defas \bfn \boffset H_n^{\boffset}
	= \bfn \boffset [ \psi(\boffset + n) - \psi(\boffset) ]
	.
	\label{eq:replicaterate}
\]

Note that in the limit $L\rightarrow\infty$, reaching a stop node and a replicate node will never co-occur with probability one.
In particular, the probability that $f_\bullet$ both (1) does not assign all of its elements to $f_{\bullet 1}$ and (2) assigns at least one element to $f_{\bullet 2}$ is given by
\[
&
\Bigl [
\sfn \soffset H_n^{\soffset} / L
		+ O( L^{-2} )
		\Bigr ]
		\Bigl [
\bfn \boffset H_n^{\boffset} / L
		+ O( L^{-2} )
	\Bigr ]
	\\
	&\hspace{1.25in}
	= 
	( \sfn \soffset H_n^{\soffset} )( \bfn \boffset H_n^{\boffset} )
		L^{-2} + O(L^{-3})
		,
\]
which approaches zero as $L\rightarrow \infty$.

\subsection{Multitype continuous-time branching processes}
\label{sec:branchingprocesses}

As in \cref{sec:continuum_limit}, associate the time $t_\ell = (\ell-1)/L \in [0,1]$ with level $\ell$ in the nested scheme.
For every $n\le N$, let $\xi_n(t)$ denote the number of type $n$ individuals in the population at time $t$.
In the limit, we study the continuous-time stochastic process $\bm \xi \defas (\bm \xi (t))_{t\ge 0}$ on $\NNInts^N$, where $\bm \xi(t) \defas (\xi_1(t), \dotsc, \xi_N(t))$, and $\NNInts \defas \{0, 1, 2, \dotsc \}$ denotes the set of non-negative integers.
The interpretation here is that the vector $\bm \xi(t)$ represents the total population of individuals at time $t$. 
From \cref{eq:stoprate,eq:replicaterate}, we see that each type $n$ individual exists for an exponentially distributed length of time with rate $\mu_n + \beta_n$, after which it is replaced by a collection of individuals denoted by $\bm j \in \NNInts^N$ (i.e., by $j_k$ individuals of type $k$, for $k\le N$) with probability $p_n(\bm j)$.
We will soon study the functions $p_1, \dotsc, p_N$ in depth.
Individuals produced from a parent engender independent lines of descent, and by the memoryless property of the exponential distribution, their lifetime and child distributions depend only on their type.
It follows that $\bm \xi$ is a \emph{continuous-time multitype Markov branching process}, a special type of branching process studied thoroughly by \citet{mode1971}.

The functions $p_n$ are called the \emph{offspring probability functions}, and clearly they must satisfy $\sum_{\bm j \in \NNInts^N} p_n(\bm j) = 1$, for every $n\le N$.
In order to characterize these functions, there are three cases to consider: At the end of the individual's lifetime, it is replaced by either (i) one identical individual and one additional individual of a potentially different type (an event we can interpret as the continuation of the original individual and the birth of a new individual), (ii) one individual of a different type (an event we call a type change), or (iii) zero individuals (an event we can interpret as the death of the individual).
Case (i) occurs in the event a replicate node is created, and cases (ii) and (iii) occur in the event that a stop node is created.
Moreover, note that each feature in the nested feature allocation scheme can only have as many elements as its parent node in the tree structure.
It follows that, in case (i), a type $n$ individual can only give birth to an individual of type $k\le n$, and in case~(ii), the individual can only change to an individual of type $k < n$.

\vspace{.5em}
\noindent {\bf Case (i)}:
Let $\phi_{n,k}$ be the conditional probability that, given a replicate node is created, the child node produced is of type $k$.
If $n=1$, then clearly $\phi_{1,1} = 1$ and $\phi_{1,k} = 0$, for every $2 \le k \le N$.
We therefore focus on the case when $2 \le n \le N$, for which we return briefly to the discrete-time setup.
Consider assigning $n$ elements from a node to the ``second node'' at the subsequent level in the nested feature allocation scheme.
Conditioned on $p_{f_\bullet}^{(2)}$, each element is assigned to the second node independently with probability $p_{f_\bullet}^{(2)}$.
Unconditionally, the number of elements assigned to the second node is beta-binomially distributed.
Finally, conditioning on the event a replicate node is created (i.e, at least one element is assigned to the second node) is obtained with a zero-truncated beta binomial distribution.
In particular, for every $1\le k\le n$ we have
\[
\phi_{n,k} &=
	\lim_{L\rightarrow \infty}
		\frac{ \betabinom( k ; n, \boffset \bfn/L, \boffset (1 - \bfn/L) ) }
		{ 1 - \betabinom( 0 ; n, \boffset \bfn/L, \boffset (1 - \bfn/L) ) }
	\\
	&=
	\lim_{L\rightarrow \infty}
	\binom{n}{k} \frac{ ( \boffset \bfn/L)_k ( \boffset (1-\bfn/L) )_{n-k} }{ (\boffset)_n - (\boffset(1-\bfn/L))_n }
	\\
	&=
	\binom{n}{k}  
	\frac{ B(\boffset+n-k, k) }{H_n^{\boffset}}
	,
\]
and $\phi_{n,k}=0$ if $n<k\le N$.

\vspace{.5em}
\noindent {\bf Cases (ii) and (iii)}:
Similarly, let $\eta_{n,k}$ be the conditional probability that, given a stop node is created, a type $n$ node changes to type $k$, where we define $\eta_{n,0}$ as the (conditional) probability that the stop node terminates the chain (i.e., the individual dies).
Clearly if $n=1$, then $\eta_{1,k} = 0$ for $1\le k\le N$ and $\eta_{n,0} = 1$.
We focus on the case $2 \le n \le N$ and return to the nested scheme.
Given that a node with $n$ elements does not assign at least one element to the ``first node'' at the subsequent level (i.e., a stop node is created), the probability that exactly $k$ objects are assigned to the first node is again given by the truncated beta-binomial distribution, and for every $0\le k \le n-1$ we have
\[
\eta_{n,k}
	&=
	\lim_{L\rightarrow\infty}
	\binom{n}{k}
	\frac{ (\soffset \sfn/L)_{n-k} (\soffset(1-\sfn/L))_k }{ (\soffset)_n - (\soffset(1-\sfn/L))_n }
	\\
	&=
	\binom{n}{k}
	\frac{ B(\soffset+k, n-k) }{ H_n^{\soffset} }
	,
\]
and $\eta_{n,k}=0$ for $n\le k\le N$.

An important tool in the study of branching processes are the \emph{offspring generating functions}, defined for each type $n\le N$ by
\[
P_n (\bm s) \defas \sum_{\bm j \in \NNInts^N} p_n (\bm j) s_1^{j_1} \dots s_N^{j_N}
	,
	\qquad
	\bm s \in \Reals^N
	.
	\label{eq:generatingfns}
\]
From our derivations above, it is easy to see that these generating functions may be written, for every type $n \le N$, as
\[
P_n (\bm s) 
	&= 
	\frac{\mu_n}{ \mu_n + \beta_n }
	\eta_{n,0}
	+
	\frac{\mu_n}{ \mu_n + \beta_n }
		\sum_{k=1}^{n-1} \eta_{n,k} s_k
		+ \frac{\beta_n}{\mu_n + \beta_n}
			\sum_{k=1}^n \phi_{n,k} s_n s_k
		,
		\quad
		\bm s \in \Reals^N
		.
\]
The first term on the right hand side corresponds to the event the individual dies, the second term to the event of a type change, and the third term to the event of a birth.
From the offspring generating functions, define the \emph{offspring mean functions} for every $n\le N$ by
\[
f_{n,k}
	&\defas 
	\frac{\partial}{\partial s_k} P_n (\bm s) \Big \vert_{\bm s = \bm 1} 
	\\
	&\: =
	\begin{cases}
	(\mu_n \eta_{n,k} + \beta_n \phi_{n,k})/(\mu_n + \beta_n)
	,
	&\qquad 1 \le k \le n-1
	,
	\\
	\beta_n ( \phi_{n,n} + 1) / (\mu_n + \beta_n)
	,
	&\qquad k = n
	,
	\\
	0
	,
	&\qquad n < k \le N
	,
	\end{cases}
	\label{eq:offspringmeans}
\]
which are interpreted as the mean number of offspring of type $k$ produced by an individual of type $n$.
An important basic assumption in the study of branching processes is that the population never \emph{explodes}, i.e., that there will never be an infinite number of individuals in the population (in finite time).
This condition can be guaranteed if all the offspring means are finite, i.e., if $f_{n,k} < \infty$, for all $n,\, k \le N$.
A quick inspection of \cref{eq:offspringmeans} and its components reveals that this condition is ensured so long as the parameters $\sfn$, $\bfn$, $\soffset$, $\boffset$, and the number of types $N$ are all finite.
This assumption is often called the \emph{non-explosion hypothesis} \citep[][Ch.~7]{athreya1972branching}, an immediate result of which ensures an almost surely finite number of leaves in the beta diffusion tree, stated formally by \cref{result:finiteleaves} in \cref{sec:properties}.
This is a reassuring property for any stochastic process employed as a non-parametric latent variable model, the translation in this case being that the number of latent features will be (almost surely) finite for any finite data set.

Carrying further with our analysis, we next define the $N\times N$ \emph{infinitesimal generating matrix} $\bm G \defas (g_{n,k})_{n,k\le N}$ with entries
\[
\label{eq:generatormatrix}
g_{n,k} \defas ( \mu_n + \beta_n ) ( f_{n,k} - \delta_{\{n=k\}} )
	,
	\qquad n,\, k \le N
	,
\]
where $\delta_{\{n=k\}}$ is equal to one when $n=k$ and zero otherwise.
One may verify that \cref{eq:generatormatrix} agrees with \cref{eq:generatormatrixbrief} in \cref{sec:properties}.
\cref{result:finiteleaves} ensures that this matrix exists,
and it is intimately linked with the population means.
Recall that $\xi_k(t)$ represents the number of type $k$ individuals in the population at time $t$.
Define an $N \times N$ matrix $\bm M (t)$, called the \emph{population mean matrix}, at time $t > 0$ with entries
\[
m_{n,k} (t) \defas \EE [ \xi_k ( t ) \given \bm \xi (0) = \bm e_n ]
	,
	\qquad
	k \le N
	,
	\label{eq:meandefns}
\]
where $\bm e_j$ is the vector with its $j$-th entry equal to one and all other entries equal to zero.
The conditional expectation in \cref{eq:meandefns} treats the initial configuration of the population as random; it may be interpreted as the expected number of type $k$ individuals in the population at time $t$, given that the population started at time $t=0$ with one individual of type $n$.
In our case, however, recall that we deterministically started the process with a single individual of type $N$.
We will therefore only be interested in studying the $N$-th row of the matrix $\bm M(t)$.
It is well known \citep[][Ch.~7]{athreya1972branching} that
\[
\bm M(t) = \exp ( \bm G t )
	,
	\qquad
	t > 0
	,
\]
where $\exp( \bm B )$ for a square matrix $\bm B$ denotes the matrix exponential function.
By recalling that the number of leaves in the beta diffusion tree corresponds to the size of the population $\bm \xi (t)$ at time $t=1$, we see that this result immediately characterizes the expected number of leaves of each cardinality as the entries in the $N$-th row of $\bm M(1)$, as stated by \cref{result:expectedleaves} in \cref{sec:properties}.
%

\subsection{Further considerations}
\label{sec:furtherbranching}

We have only taken advantage of a few elementary results from the literature on branching processes, however, there are a few obstacles that a further analysis in this direction must overcome.
For example, most analyses of the mean matrix $\bm M(t)$ typically involve the assumption that the branching process is \emph{irreducible}, a property that the beta diffusion tree does not satisfy.
In particular, in a multitype branching process, two types $i$ and $j$ are said to \emph{communicate} if it is possible for an individual of type $i$ to be an ancestor to another individual of type $j$ (at any later time), and vice versa.
In the case of an irreducible branching process, the mean matrix $\bm M(t)$ is positive and the \emph{Perron--Frobenius} theorem of positive matrices ensures that there exists a positive real eigenvalue $\lambda_0$ of $\bm M(t)$ that is strictly larger in magnitude than the real part of all other eigenvalues (in particular, the multiplicity of $\lambda_0$ is one). 
The value of the dominant eigenvalue $\lambda_0$ determines whether the process is \emph{sub-critical, critical, or super-critical}, an important classification of the branching process from which almost any analysis must follow.
In our case, it is clear that an individual of type $m$ can only be an ancestor to individuals of type $k$ if $k\le n$.
Our branching process is therefore said to be \emph{decomposable} or \emph{reducible} and there exists an ordering of the $N$ types: an individual of type $N$ can be an ancestor to individuals of all types; an individual of type $n$ for $1 < n < N$ can be an ancestor of individuals of type $k$, for $k\le n$; and an individual of type $1$ can only be an ancestor to individuals of its own type.
Each type only communicates with itself.
Clearly, the corresponding discrete-time process described in \cref{sec:galtonwatson} is also decomposable, a case that has been studied, for example, by \citet{KestenStigum1967} and \citet{foster1978}.  The continuous-time case has been studied by \citet{savin1962}.  A study of both cases was done by \citet{ogura1975}.
Research in this direction could uncover further properties of the beta diffusion tree.
For example, the second moments of $\bm \xi$ are typically analyzed by studying their Chapman--Kolmogorov equations, from which it may be possible to obtain a characterization of the variance on the number of leaves in the beta diffusion tree.

\section{Details on inference}
\label{sec:furtherinference}

In addition to resampling subtrees as described in \cref{sec:inference}, we also make several proposals that resample specific parts of the tree structure in order to improve mixing.
These moves are summarized in \cref{sec:inference} and we now describe them in detail.
Each move is a Metropolis--Hastings proposal, which we verify are targeting the correct posterior distributions with joint distribution tests, as detailed in \cref{sec:Geweke}.

\subsection{Adding and removing replicate and stop nodes}
\label{sec:addremovenodes}

Recall that $\Rcal(\Tcal_N)$ denotes the set of replicate nodes in $\Tcal_N$ and that $\Bcal(\Tcal_N)$ denotes the set of branches.
We propose adding and removing replicate nodes from the tree structure as follows:
\begin{enumerate}
\item {\bf Remove replicate node:} Select a replicate node $v^* \in \Rcal(\Tcal_N)$ with probability \emph{inversely} proportional to the number of particles that traversed the branch ending at $v^*$.
Propose a new tree structure $\Tcal_N^*$ by removing the subtree emanating from the divergent branch at $v^*$, which is accepted with probability
\[ 
A_{\mathcal T_N \rightarrow \mathcal T_N^*}
	= 
	\max \left \{
		1 , 
		\frac { p ( \bm Y \given \Tcal_N^* ) }
				{ p ( \bm Y \given \Tcal_N ) }
		\frac { m(v^*) / \sum_{u \in \Vcal(\Tcal_N^*)} m(u) }
				{ \sum_{u \in \Rcal(\Tcal_N) \setminus \{v^*\} } m(u) / \sum_{u \in \Rcal(\Tcal_N)} m(u) }
		\right \}
		.
\]

\item {\bf Add replicate node:} Select a branch $[ef] \in \Bcal(\Tcal_N)$ in the tree with probability proportional to the number of particles down the branch, and propose adding a replicate node on this branch to form $\Tcal_N^*$ according to the following procedure:

\begin{enumerate}
\item Sample a new replicate time $t^*$ on this branch according to the prior for the very first particle down this branch.  If the sampled replicate time occurs after the end of the branch, resample a new time until this is not the case.

\item Propose adding a replicate node $v^*$ on this branch at time $t^*$.  Note that $m(v^*) = m(f)$.  Uniformly at random, select a particle down the branch to be the one that created the replicate node.  Sample the selected particle's paths down a new subtree (diverging at the replicate node) according to the prior until $t=1$.

\item Sequentially for each remaining particle down the branch, which we index by $j = 1,\dotsc,m(v^*)-1$, decide whether or not the particle sends a replicate down this new subtree with probability $n_r(v^*) / (\boffset + j)$, where $n_r(v^*) \ge 1$ is the number of particles that previously chose to send a replicate down the new subtree.
If the particle replicates, then we resample its paths down the new subtree according to the prior until $t=1$, given the paths of all previous particles that previously travelled down this new subtree.
\end{enumerate}
The new tree structure $\Tcal_N^*$ is accepted with probability
\[ 
A_{\mathcal T_N \rightarrow \mathcal T_N^*}
	= 
	\max \left \{
		1 , 
		\frac { p ( \bm Y \given \Tcal_N^* ) }
				{ p ( \bm Y \given \Tcal_N ) }
		\frac { \sum_{u \in \Rcal(\Tcal_N^*) \setminus \{v^*\} } m(u) / \sum_{u \in \Rcal(\Tcal_N^*)} m(u) }
				{ m(f) / \sum_{u \in \Vcal(\Tcal_N)} m(u) }
		\right \}
		.
\]
\end{enumerate}

Recall that $\Scal(\Tcal_N)$ denotes the set of stop nodes in the tree.  We propose removing and adding stop nodes from the tree structure as follows: 
\begin{enumerate}
\item {\bf Remove stop node:} Select a stop node $v^*$ in the tree with probability inversely proportional to the number of particles down the branch ending at the stop node.  Propose $\Tcal_N^*$ by removing the stop node from the tree as follows:
Sequentially for each particle that stopped at $v^*$, sample the paths of the particle down the subtree rooted at $v^*$.
Then remove $v^*$ from the tree and accept the new tree $\Tcal_N^*$ with probability
\[ 
A_{\mathcal T_N \rightarrow \mathcal T_N^*}
	= 
	\max \left \{
		1 , 
		\frac { p ( \bm Y \given \Tcal_N^* ) }
				{ p ( \bm Y \given \Tcal_N ) }
		\frac { m(v^*) / \sum_{u \in \Vcal(\Tcal_N^*)} m(u) }
				{ \sum_{u \in \Scal(\Tcal_N) \setminus \{v^*\} } m(u) / \sum_{u \in \Scal(\Tcal_N)} m(u) }
		\right \}
		.
\]

\item {\bf Add stop node:} Select a branch $[ef] \in \Bcal(\Tcal_N)$ in the tree with probability proportional to the number of particles down the branch.  Propose adding a stop node on this branch to form $\Tcal_N^*$ as follows:
\begin{enumerate}
\item Sample a new stop time $t^*$ on this branch according to the prior, as if no particles had previously traversed the branch.
If the sampled stop time occurs after the end of the branch, resample a new time until this is not the case.

\item Propose adding a stop node $v^*$ on this branch at $t^*$.  Uniformly at random, select a particle down the branch to be the one that created the stop node.  Remove the particle from the subtree following $v^*$.

\item Sequentially for each remaining particle down this branch, which we index by $j=1,\dotsc,m(v^*)-1$, decide whether or not the particle stops at $v^*$ with probability $n_s / (\soffset + j)$, where $n_s\ge 1$ is the number of particles that previously chose to stop at $v^*$.  If the particle stops, then remove it from the subtree following $v^*$.

\end{enumerate}
Accept the new tree $\Tcal_N^*$ with probability
\[ 
A_{\mathcal T_N \rightarrow \mathcal T_N^*}
	= 
	\max \left \{
		1 , 
		\frac { p ( \bm Y \given \Tcal_N^* ) }
				{ p ( \bm Y \given \Tcal_N ) }
		\frac { \sum_{u \in \Scal(\Tcal_N^*) \setminus \{v^*\} } m(u) / \sum_{u \in \Scal(\Tcal_N^*)} m(u) }
				{ m(f) / \sum_{u \in \Vcal(\Tcal_N)} m(u) }
		\right \}
		.
\]
\end{enumerate}
%

\subsection{Resample replicate and stop node configurations}

We propose changing the decisions that particles take at replicate and stop nodes with the following proposals:
Select an internal node $v^* \in \Rcal(\Tcal_N) \cup \Scal(\Tcal_N)$ in the tree with probability proportional to the number of particles down the branch ending at $v^*$.
Uniformly at random, select one of the particles down the branch ending at $v^*$.
Propose $\Tcal_N^*$ as follows:
\begin{enumerate}

\item {\bf If $v^*$ is a replicate node:}  If the selected particle chose to replicate at $v^*$ and take the divergent path, propose removing the particle from the subtree emanating from the divergent branch at $v^*$.
If the particle did not replicate at $v^*$, propose replicating the particle by sampling the replicate's paths down the subtree emanating from the divergent branch until $t=1$ according to the prior.

\item {\bf If $v^*$ is a stop node:}  If the selected particle chose to stop at $v^*$, propose adding the particle to the subtree rooted at $v^*$ by sampling its paths down the subtree until $t=1$ according to the prior.
If the particle did not stop at $v^*$, propose stopping the particle by removing it from the subtree rooted at $v^*$.

\end{enumerate}
In both cases, accept $\Tcal_N^*$ with probability
\[
A_{\mathcal T_N \rightarrow \mathcal T_N^*}
	=
	\max \left \{
		1 , 
		\frac { p ( \bm Y \given \Tcal_N^* ) \: \sum_{u \in \Rcal (\Tcal_N) \cup \Scal (\Tcal_N)} m(u) }
				{ p ( \bm Y \given \Tcal_N ) \: \sum_{u \in \Rcal (\Tcal_N^*) \cup \Scal (\Tcal_N^*)} m(u) }
		\right \}
		.
\]

\subsection{Heuristics to prune and thicken branches}
\label{sec:heuristicmoves}

In our experiments, we also found the following heuristics for removing poorly populated replicate and stop nodes helpful.

\begin{enumerate}
\item {\bf Prune thin branches}: Select a divergent branch emanating from some replicate node $v^*$ with probability inversely proportional to the \emph{proportion} of particles that decided to replicate and take the divergent branch.
Propose removing the subtree emanating from this branch to form $\Tcal_N^*$.

\item {\bf Thickening branches}: Select a stop node $v^*$ in the tree with probability inversely proportional to the number of particles that stopped at the node.
Propose removing the stop node on this branch as in \cref{sec:addremovenodes} to propose $\Tcal_N^*$.
\end{enumerate}
For both moves we accept $\Tcal_N^*$ with probability
\[
A_{\mathcal T_N \rightarrow \mathcal T_N^*}
	=
	\max \left \{
		1 , 
		\frac { p ( \bm Y \given \Tcal_N^* )\, p( \Tcal_N^*) }
				{ p ( \bm Y \given \Tcal_N )\, p(\Tcal_N) }
		\right \}
		,
\]
where $p(\Tcal_N)$ is given by the product of \cref{eq:prob_nodes,eq:prob_branches}.

These moves specifically target the removal of unpopular branch and stop nodes, which we found worthwhile to exploit.
It is not clear how to appropriately propose adding and subtracting unpopular nodes to the tree structure in an analogous manner.
Because such reverse proposals are not made, these moves do not leave the stationary distribution of the Markov chain invariant, however, in our experiments they shortened the times until burn-in.
Following an appropriate burn-in period, a practitioner may therefore remove these steps from the sampler in order to produce samples from the correct stationary distribution.
These moves are not included in the joint distribution tests in \cref{sec:Geweke}.

\section{Joint distribution tests}
\label{sec:Geweke}

We can verify that our MCMC procedure in \cref{sec:inference} and \cref{sec:furtherinference} (not including the heuristic move described therein) are sampling from the correct posterior distributions (after convergence of the Markov chain) with the joint distribution tests discussed by \citet{Geweke2004}.
In general, there are two methods to sample from the joint distribution $p( Y, \Theta )$ over latent parameters $\Theta$ and data $Y$. 
The first is referred to as a ``marginal-conditional'' sampler, which proceeds as follows:
\begin{flalign*}
&\text{\bf for } m = 1,\dotsc, M \: \text{\bf do:} & \\
&\qquad \Theta^{(m)} \dist p(\Theta) & \\
&\qquad Y^{(m)} \given \Theta^{(m)} \dist p( \, Y \given \Theta^{(m)} ) & \\
&\text{\bf end}
\end{flalign*}
which produces $M$ independent samples of $(Y, \Theta)$. 
The second method is referred to as a ``successive-conditional'' sampler, which proceeds as follows:
\begin{flalign*}
&\Theta^{(1)} \dist p( \Theta ) & \\
&Y^{(1)} \given \Theta^{(1)} \dist p ( \, Y \given \Theta^{(1)} ) & \\
&\text{\bf for } m = 2,\dotsc, M \: \text{\bf do:} & \\
& \qquad \Theta^{(m)} \given \Theta^{(m-1)}, Y^{(m-1)} \dist Q( \Theta \given \Theta^{(m-1)}, Y^{(m-1)} ) & \\
&\qquad Y^{(m)} \given \Theta^{(m)} \dist p( \, Y \given \Theta^{(m)} ) & \\
&\text{\bf end}
\end{flalign*}
where $Q$ represents a single (or multiple) iterations of our MCMC sampler.
We may then compare parameter values and test statistics computed on the two sets of samples and check if they appear to come from the same distribution.
We run these joint distribution tests on our MCMC sampler with $N=5$, $D=2$, and the prior distributions specified in \cref{sec:hyperparameters}, producing 2,000 samples (thinned from 200,000 iterations in order to reduce autocorrelation).
We can visually inspect the histograms of some parameter values from the two different samplers, for example, in \cref{fig:geweke_plots} we display histograms of the value of $\sigma_Y$ and the number of branch nodes in the tree from the two samplers.
While the two distributions for each statistic do indeed look similar, we can check this formally with a two-sample Kolmogorov--Smirnov (KS) test, where the null hypothesis is that the two sets of samples come from the same distribution.
A KS-test comparing the two sets of samples for $\sigma_Y$ produces a p-value of 0.6406, and for the number of branch nodes produces a p-value of 0.1036, so at a 0.05 significance level, we fail to reject the null hypothesis that the two samples come from the same distribution.
In \cref{tab:geweke}, we display the p-values for the model hyperparameters and the following test statistics for the tree structure: The numbers of latent features $K$, branch nodes, stop nodes, and non-zero entries in $\bm Z$.  We also test the density of the matrix $\bm Z$ and the time of the first node in the tree structure following the root.
All tests pass at a 0.05 significance level.
\begin{figure}[h!]
\centering
\subfigure[\# branch nodes -- marginal-conditional]{
                \includegraphics[scale=.3]{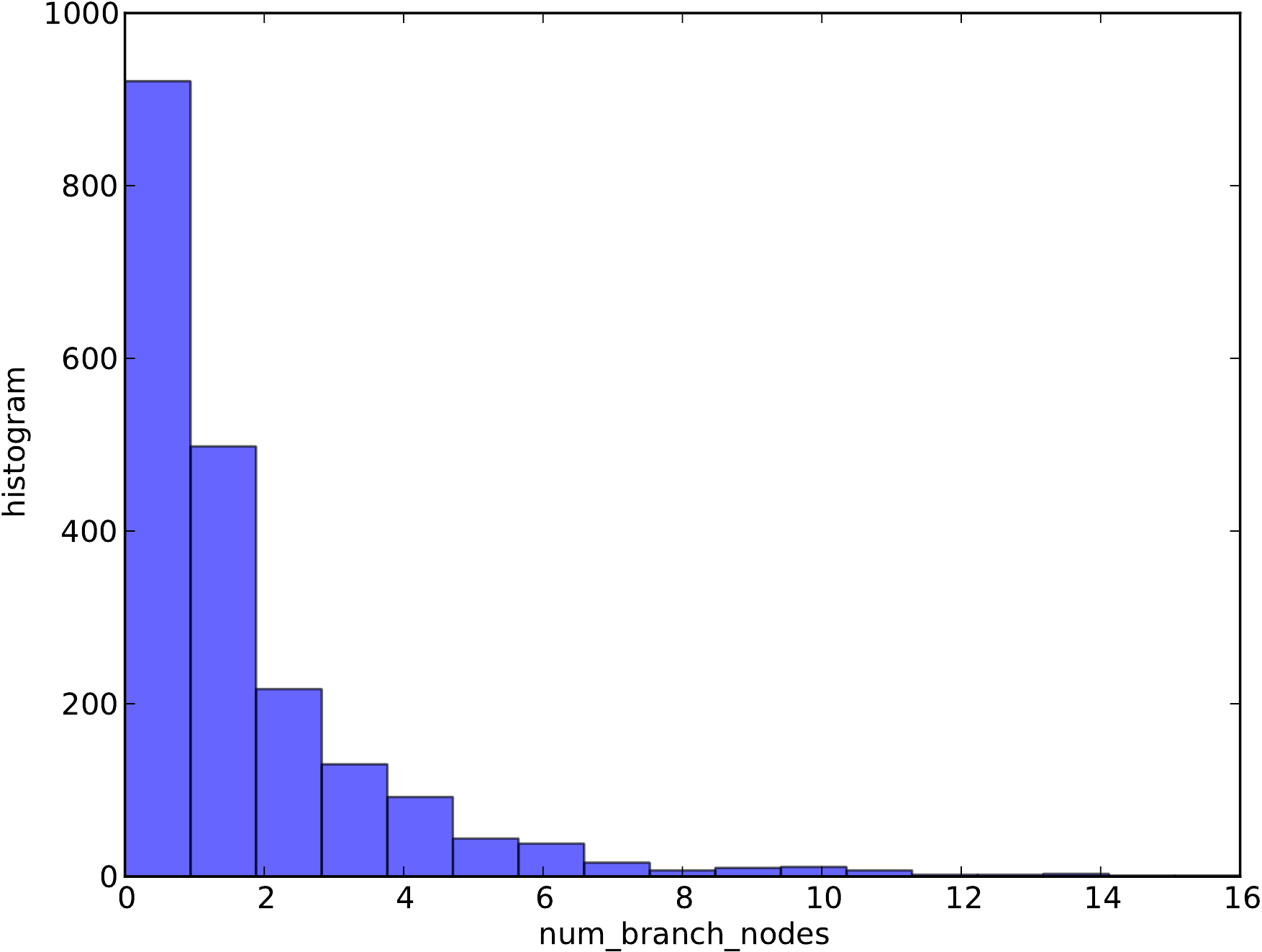}
                }
        \qquad
\subfigure[\# branch nodes -- successive-conditional]{
                \includegraphics[scale=.3]{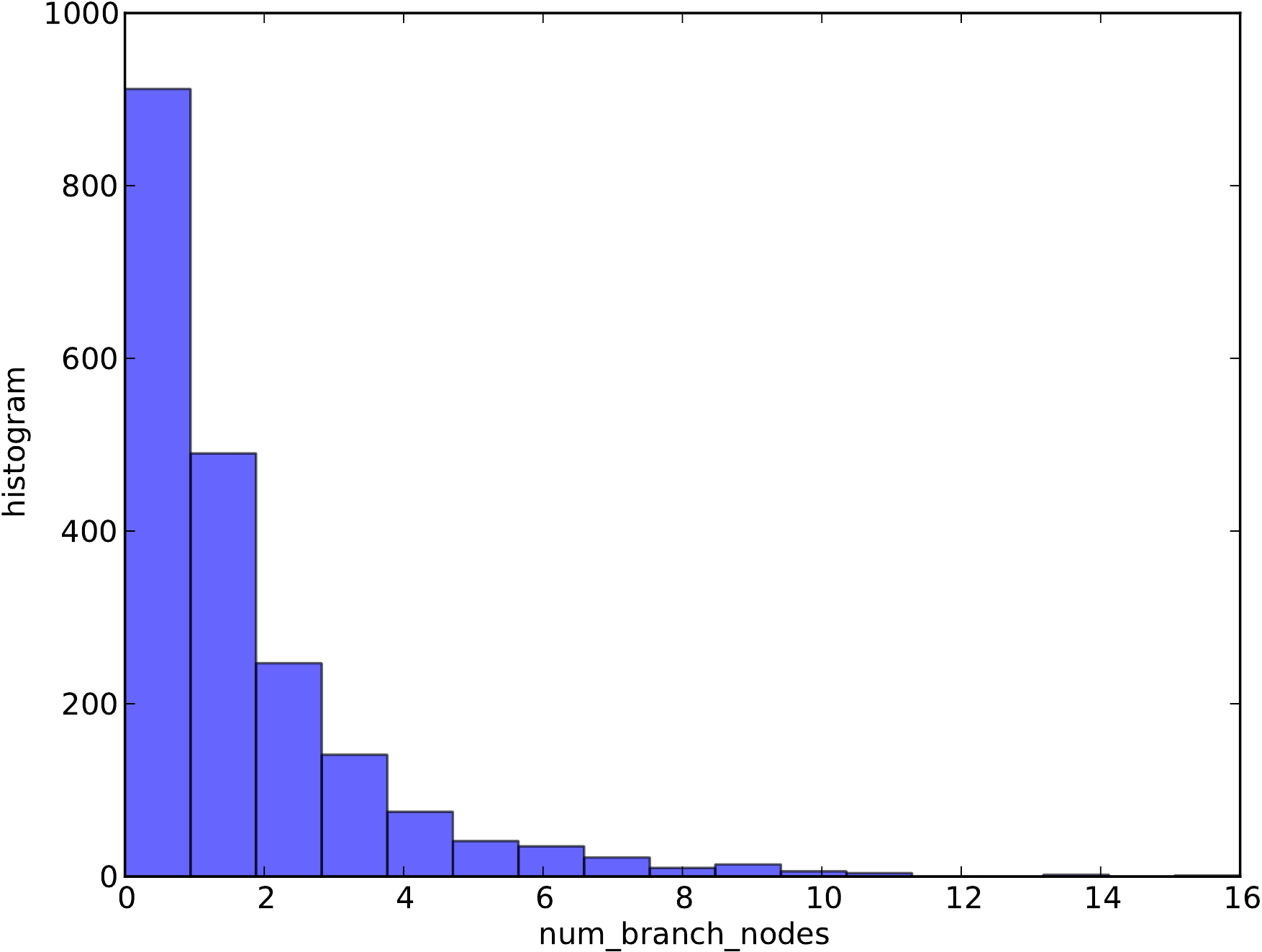}
        }
\subfigure[$\sigma_Y$ -- marginal-conditional]{
                \includegraphics[scale=.3]{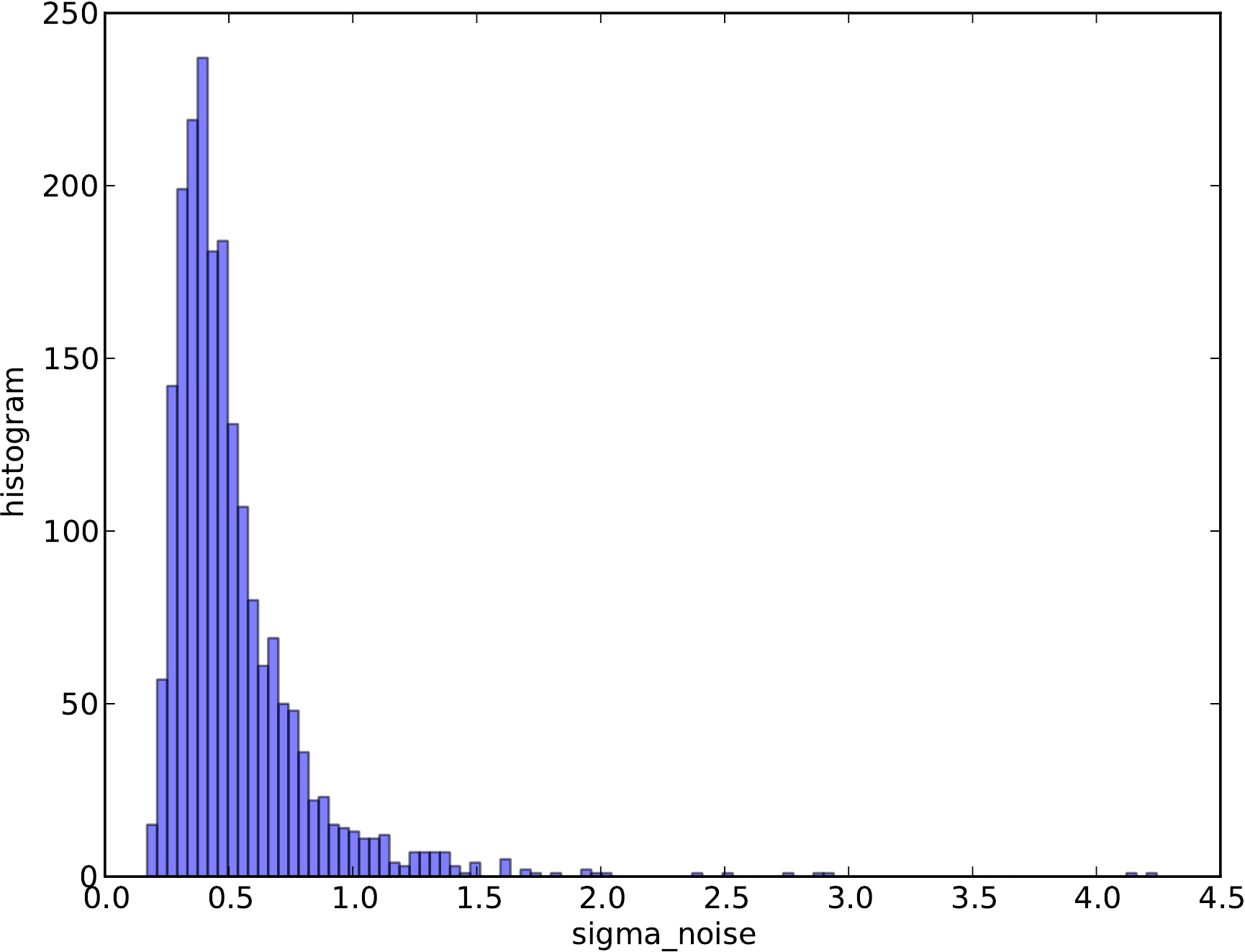}
                }
                \qquad
\subfigure[$\sigma_Y$ -- successive-conditional]{
                \includegraphics[scale=.3]{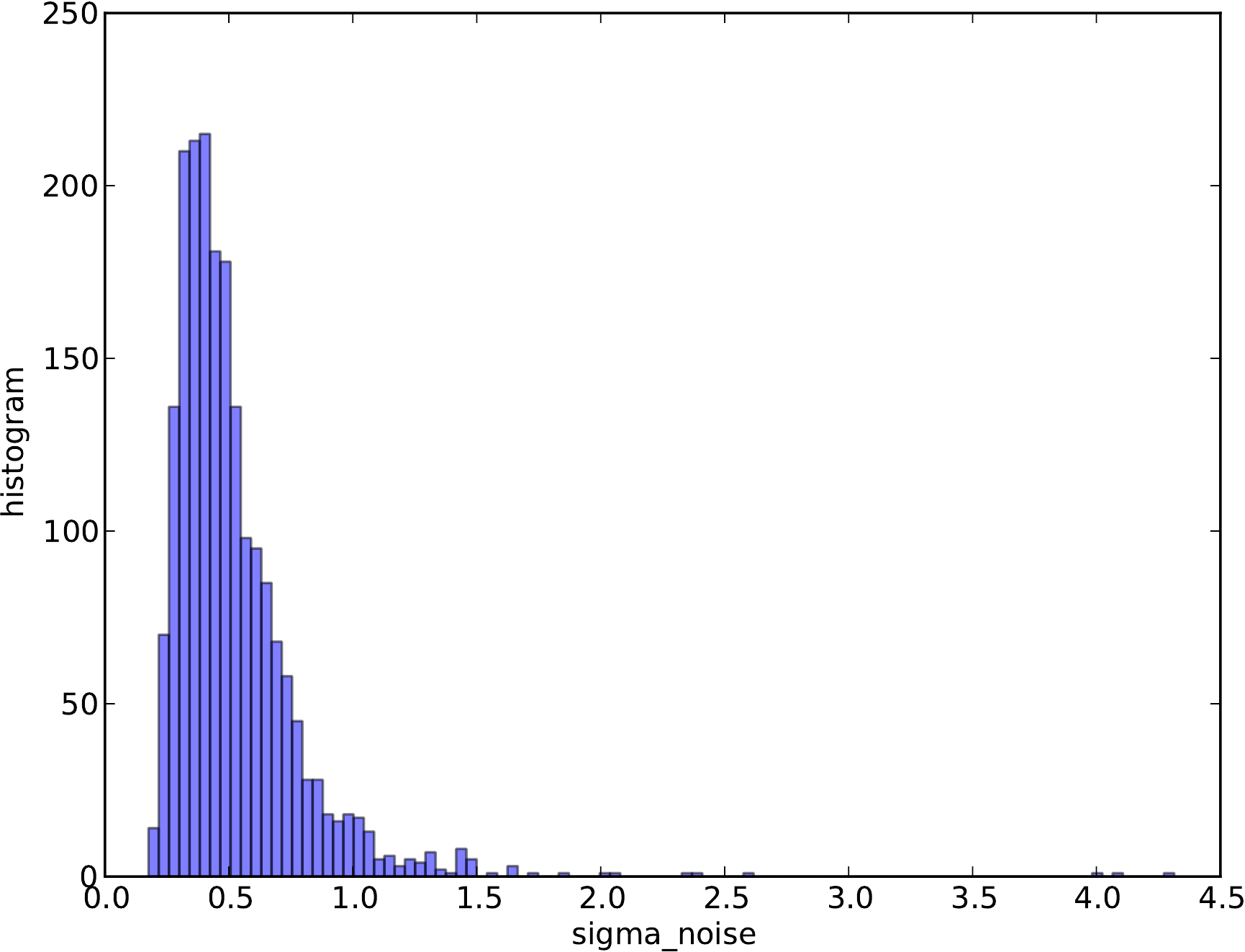}
        }
      	\caption{Histograms of 2,000 independent samples (thinned from 200,000 iterations) from the marginal-conditional and successive-conditional samplers, both of which should sample from the joint distribution over latent parameters and observed data in the linear Gaussian example.  (a) and (b) show the histograms of the number of branch nodes in the tree structure.  (c) and (d) show the histograms for values of the parameter $\sigma_Y$.}
	\label{fig:geweke_plots}
\end{figure}
\begin{table}[h!]
\caption{p-values for two-sample Kolmogorov--Smirnov tests of the null hypothesis that the samples (of various test statistics and latent parameter values) from the marginal-conditional and successive-conditional samplers come from the same distribution.  All tests pass at a 0.05 significance level.}
\label{tab:geweke}
\vskip 0.15in
\begin{center}
\begin{small}
\begin{sc}
\begin{tabular}{lcccccc}
\multicolumn{7}{c}{Tree structure statistics} \\
\hline
 & $K$ & \# branch & \# stop & nnz $\bm Z$ & density $\bm Z$ & first time \\
p-val. & 0.0535 & 0.1036 & 0.9967 & 0.5288 & 0.3698 & 0.0846 \\
\hline
\quad \\
\multicolumn{7}{c}{Hyperparameters} \\
\hline
 & $\soffset$ & $\boffset$ & $\sfn$ & $\bfn$ & $\sigma_X$ & $\sigma_Y$ \\
p-val. & 0.4617 & 0.1298 & 0.5287 & 0.4101 & 0.9967 & 0.6406 \\
\hline
\end{tabular}
\end{sc}
\end{small}
\end{center}
\vskip -0.1in
\end{table}

\bibliography{BDT_refs}

\end{document}